\documentclass[letterpaper]{article} % DO NOT CHANGE THIS
\usepackage[submission]{aaai23}  % DO NOT CHANGE THIS
\usepackage{times}  % DO NOT CHANGE THIS
\usepackage{helvet}  % DO NOT CHANGE THIS
\usepackage{courier}  % DO NOT CHANGE THIS
\usepackage[hyphens]{url}  % DO NOT CHANGE THIS
\usepackage{graphicx} % DO NOT CHANGE THIS
\usepackage{natbib}  % DO NOT CHANGE THIS AND DO NOT ADD ANY OPTIONS TO IT
\usepackage{caption} % DO NOT CHANGE THIS AND DO NOT ADD ANY OPTIONS TO IT
\usepackage{algorithm}
\usepackage{algorithmic}
\usepackage[colorlinks=true,linkcolor=blue,citecolor=blue]{hyperref}
\newcommand{\tO}[1]{\tilde{\mathcal{O}}(}

\newcommand{\mathbbm}[1]{\text{\usefont{U}{bbm}{m}{n}#1}}
\usepackage{newfloat}
\usepackage{listings}
\DeclareCaptionStyle{ruled}{labelfont=normalfont,labelsep=colon,strut=off} % DO NOT CHANGE THIS
\floatstyle{ruled}
\newfloat{listing}{tb}{lst}{}
\floatname{listing}{Listing}
\usepackage{amsfonts,amsmath,amssymb,amsthm,mathtools}

\usepackage{color}
\usepackage{tabularx}

\DeclareMathOperator*{\argmax}{arg\,max}
\DeclareMathOperator*{\argmin}{arg\,min}

\DeclareMathOperator*{\subopt}{SubOpt}

\usepackage{accents}

\theoremstyle{plain}
\newtheorem{theorem}{Theorem}
\newtheorem{lemma}{Lemma}

\newtheorem{corollary}{Corollary}

\theoremstyle{definition}
\newtheorem{defn}{Definition}

\newtheorem{assumption}{Assumption}
\newtheorem{remark}{Remark}

\usepackage{enumitem}   

\usepackage{afterpage}

\usepackage{minitoc}
\allowdisplaybreaks
\usepackage[table]{xcolor}
\usepackage{booktabs} % for "\addlinespace" macro

\usepackage{longtable}
\usepackage{makecell}
\usepackage{multirow}

\title{On Instance-Dependent Bounds for Offline Reinforcement Learning \\with Linear Function Approximation}
\author{
    Written by AAAI Press Staff\textsuperscript{\rm 1}\thanks{With help from the AAAI Publications Committee.}\\
    AAAI Style Contributions by Pater Patel Schneider,
    Sunil Issar,\\
    J. Scott Penberthy,
    George Ferguson,
    Hans Guesgen,
    Francisco Cruz\equalcontrib,
    Marc Pujol-Gonzalez\equalcontrib
}
\affiliations{
    \textsuperscript{\rm 1}Association for the Advancement of Artificial Intelligence\\
    1900 Embarcadero Road, Suite 101\\
    Palo Alto, California 94303-3310 USA\\
    publications23@aaai.org
}

\usepackage{bibentry}
\begin{document}

\maketitle

\begin{abstract}
Sample-efficient offline reinforcement learning (RL) with linear function approximation has been studied extensively recently. However, prior work has focused on obtaining $\frac{1}{\sqrt{K}}$-type sub-optimality bound, which is also minimax optimal, with $K$ being the number of episodes in the offline data. In this work, aiming to understand instance-dependent bounds for offline RL with function approximation, we present an algorithm named Bootstrapped and Constrained Pessimistic Value Iteration (BCP-VI) which leverages data bootstrapping and uses a constraint optimization to further encode pessimism. With our algorithm, we show that under the optimal-policy concentrability, offline RL speeds up the rate to $\tO(\frac{1}{K})$ when there is a positive sub-optimality gap in the optimal $Q$-value function, even when the offline data were collected adaptively. Moreover, when the linear features of the optimal actions in the states reachable by an optimal policy spans those reachable by the behavior policy and the optimal actions are unique, offline RL achieves absolute zero sub-optimality error when $K$ exceeds an (finite) instance-dependent threshold. To the best of our knowledge, these are the first $\tO(\frac{1}{K})$-type bound and zero bound for offline RL with linear function approximation under the partial data coverage and adaptive data. We also provide instance-agnostic and instance-dependent information-theoretical bounds to complement our upper bounds and a numerical simulation to verify the algorithm. 

% This paper studies instance-adaptive bounds for offline reinforcement learning from adaptive data with linear function approximation. 
\end{abstract}

\section{Introduction}
% \thanh{include model selection, can I add actor critic to solve the constrained optimization as in the new version of \url{https://arxiv.org/pdf/2106.06926.pdf}?}

% {\color{red} TODO list
% \begin{itemize}
%     \item $\kappa$ to $\sqrt{ \kappa}$ 
%     \item Write the bounds in terms of $\kappa = \max_{h} \kappa$ 
%     \item Lower bound in linear MDP: do we have $d$? 
%     \item Discuss \citep{yin2021towards} in depth 
%     \item In the minimax case, can we improve the lower bound to include $d$ in the lower bound, like this \url{https://arxiv.org/pdf/2206.11489.pdf} and this \citep{he2021logarithmic}? 
% \end{itemize}
% }

We consider the problem of offline reinforcement learning (offline RL), where the goal is to learn an optimal policy from a fixed dataset generated by some unknown behavior policy 
% \raman{Why do we need to introduce the notation here? Also, do we need to say ``without any further exploration?'' Isnn't that implied?}
~\citep{lange2012batch,levine2020offline}. The offline RL problem has recently attracted much attention from the research community. It provides a practical setting where logged datasets are abundant but exploring the environment can be costly due to computational, economic, or ethical reasons. It finds applications in a number of important domains including healthcare~\citep{gottesman2019guidelines,nie2021learning}, recommendation systems~\citep{strehl2010learning,thomasAAAI17,zhang2022two}, econometrics~\citep{Kitagawa18,athey2021policy}, and more. 

% \begin{todo}
% \begin{itemize}
%     \item Write down and revise the results 
%     \item Organize story 
% \end{itemize}
% \end{todo}
% We consider the problem of offline reinforcement learning (offline RL), where the goal is to learn an optimal policy from a fixed dataset generated by some unknown behavior policy $\mu$ a priori,  without any further exploration~\citep{lange2012batch,levine2020offline}. The offline RL problem has recently attracted growing research interest as it provides a practical setting where logged dataset are abundant but exploring the environment can be costly (either computationally, economically or even ethically), with many important applications in  healthcare \citep{gottesman2019guidelines,nie2021learning}, recommendation systems \citep{strehl2010learning,thomasAAAI17}, and econometrics \citep{Kitagawa18,athey2021policy}. \\

% Though pessimism is not always the best approach for offline RL (at least in offline bandits) depending on the behaviour policy that generates the offline data (e.g. when the offline data does not support good actions, an optimistic approach tends to be better \citep{xiao2021optimality}), 
% In practice, 

% \thanh{Introduce the challenges that this paper aims to address}
A large body of literature is devoted to providing 
% \raman{We should simply say generalization bounds or clarify what we mean by suboptimality in this context.} 
generalization bounds for offline reinforcement learning with linear function approximation, wherein the reward and transition probability functions are parameterized as linear functions of a given feature mapping.  
% These works parameterize rewards and transition probabilities as linear functions of a given feature mapping. 
% These works directly parameterize the reward function and the transition probability function as linear functions of a given feature mapping. 
% For instance, \citet{jin2021pessimism} design the first sample-efficient pessimistic algorithm namely PEVI for linear MDP. 
For such linear MDPs, \citet{jin2021pessimism} present a pessimistic value iteration (PEVI) algorithm and show that it is sample-efficient. 
In particular, \citet{jin2021pessimism} provide a sample complexity bound for PEVI such that under the assumption that each trajectory is independently sampled and the behaviour policy is uniformly explorative in all dimensions of the feature mapping, the complexity bound improves to $\tiO(\frac{d^{3/2} H^2}{\sqrt{K}} )$ where $d$ is the dimension of the feature mapping, $H$ is the episode length, and $K$ is the number of episodes in the offline data. 
In a follow-up work, \citet{Xiong2022NearlyMO,yinnear} leverage variance reduction (to derive a variance-aware bound) and data-splitting (to circumvent the uniform concentration argument) to further improve the result in \citet{jin2021pessimism} by a factor of $\mathcal{O}(\sqrt{d} H)$.
% \citet{Xiong2022NearlyMO,yinnear} leverage variance reduction (to derive a variance-aware bound) and use data-splitting (to circumvent the uniform concentration argument) to further improve the result in \citet{jin2021pessimism} by a factor of $\mathcal{O}(\sqrt{d} H)$.
\citet{xie2021bellman} propose a pessimistic framework with general function approximation, and their bound improves that of \cite{jin2021pessimism} by a factor of $\sqrt{d}$ when the action space is finite, and the function approximation is linear. 
\citet{uehara2021pessimistic} also 
% obtain a \raman{I do not like the terminolgoy here of xx-type bound. Is it standard? I am not used to it. Can we not say that ``... obtain a a convergence rate of $\frac{1}{\sqrt{K}}$ for offline RL''?}
obtain the $\frac{1}{\sqrt{K}}$ rate for offline RL with general function approximation, but like \cite{xie2021bellman}, their results are, in general, not computationally tractable as they require an optimization subroutine over a general function class. Although the $\frac{1}{\sqrt{K}}$ rate is minimax-optimal, in practice, assuming a worst-case setting is too pessimistic. Indeed, several empirical works suggest that in such natural settings, we can learn at a rate that is much faster than $\frac{1}{\sqrt{K}}$ (e.g., see Figure \ref{fig:fast_rates_sim}). We argue that to circumvent these lower bounds and explain the rates we observe in practical settings, we should consider the intrinsic instance-dependent structure of the underlying MDP.
Furthermore, existing works establishing the the minimax-optimal $\frac{1}{\sqrt{K}}$ rate still require a strong assumption of uniform feature coverage and trajectory independence. \footnote{The only exception is \cite{jin2021pessimism}, but their bound is generic, and they do not show if they can achieve a rate of $\frac{1}{\sqrt{K}}$ under a partial data coverage assumption.} This motivates us to study tighter instance-dependent bounds for offline RL with the mildest data coverage condition possible. 

Instance/gap-dependent bounds have been extensively studied in \emph{online} bandit and reinforcement learning literature~\citep{simchowitz2019non,yang2021q,xu2021fine,he2021logarithmic}. % These works typically rely on an instance-dependent quantity such as the minimum positive sub-optimality gap between an optimal action and the sub-optimal ones. However, to the best of our knowledge, it is still largely unclear how to leverage such instance-dependent structure to improve offline RL, especially due to the unique challenge of distributional shift in offline RL as compared to the online case. A couple of recent works \citep{hu2021fast,wang2022gap} give gap-dependent bounds for offline RL; however, these works either require a strong uniform feature coverage assumption or only work for tabular MDPs. In addition, these works~\citep{hu2021fast,wang2022gap} require that the trajectories are collected independently across episodes -- an assumption that is not very realistic as the data might have been collected by some online learning algorithms that interact with the MDPs \cite{fu2020d4rl}. 
% As far as we know, there does not exist 
% We are not aware of any work that leverages such instance-dependent/gap-dependent structure for offline RL with adaptive data and linear function approximation. Therefore, a natural question to ask is:
These works typically rely on an instance-dependent quantity, such as the minimum positive sub-optimality gap between an optimal action and the sub-optimal ones. However, to the best of our knowledge, it is still largely unclear how to leverage such an instance-dependent structure to improve offline RL, especially due to the unique challenge of distributional shift in offline RL as compared to the online case. A few recent works \citep{hu2021fast,wang2022gap} give gap-dependent bounds for offline RL; however, these works either require a strong uniform feature coverage assumption or only work for tabular MDPs. In addition, they require that the trajectories are collected independently across episodes -- an assumption that is not very realistic as the data might have been collected by some online learning algorithms that interact with the MDPs \citep{fu2020d4rl}. 
We are unaware of any existing work that leverages an instance/gap-dependent structure for offline RL with adaptive data and linear function approximation, which motivates the following question we consider in this paper. 
\begin{center}
\textit{Can we derive instance/gap-dependent bounds for offline RL with linear representations?}
\end{center}

% \raman{Not sure I understand what you mean by ``gap raised by the concurrent work.''} 
We answer the above question affirmatively and thus narrow the literature gap that were discussed in the recent work of \cite{wang2022gap}. In particular, we use $\Delta_{\min}$ to denote the minimum positive sub-optimality gap between the optimal action and the sub-optimal ones \citep{simchowitz2019non,yang2021q,he2021logarithmic}. The larger the $\Delta_{\min}$, the faster we can learn in an online setting since the actions with larger rewards are likely to be optimal, thereby reducing the time needed for exploration. Similarly, offline learning with uniform data coverage can benefit from the gap information as the entire state-action space is already fully explored by the offline policy \citep{hu2021fast}. However, it remains elusive as \emph{how an offline learner can benefit from the gap information where the learner cannot explore the environment anymore, and the offline data does not fully cover the state-action space}. 

% In particular, we use $\Delta_{\min}$ to denote the minimum positive sub-optimality gap between the optimal action and the sub-optimal ones \citep{simchowitz2019non,yang2021q,he2021logarithmic}. The gap information, $\Delta_{\min}$, is more intuitive as a benefit to online learning as the exploration problem in online learning becomes easier if the gap is large (e.g., actions with larger rewards are likely the optimal action due to the large gap, thus less time on exploration). Similarly, offline learning with uniform data coverage can benefit from the gap information as the entire state-action space is already fully explored by the offline policy \citep{hu2021fast}. However, it is not clear how an offline learner can benefit from the gap information where the learner cannot explore the environment anymore and the offline data did not fully cover the state-action space. 
\begin{figure*}[h!]
    \centering
    \includegraphics[scale=0.29]{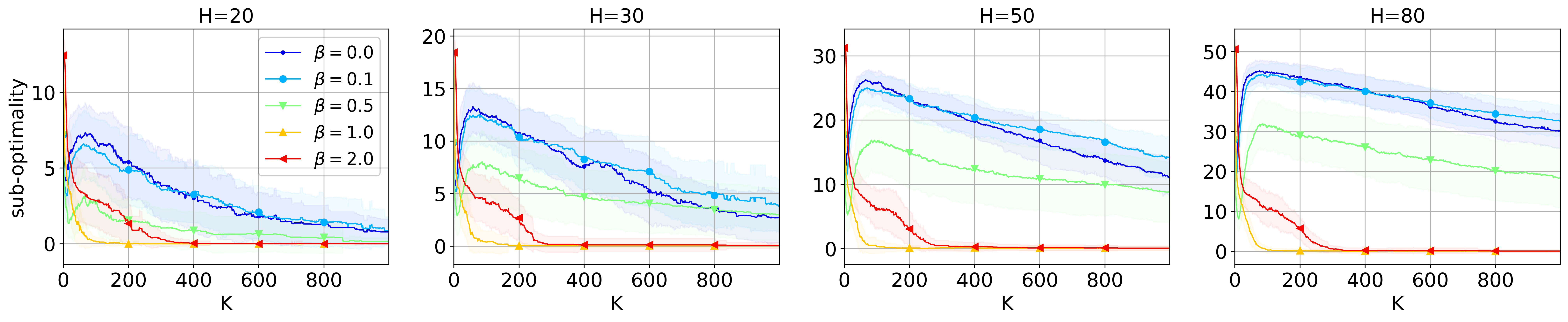}
    \caption{A comparison of BCP-VI with its non-pessimistic variant (i.e., $\beta=0$, where $\beta$ is defined at line~\ref{bpvi:lcb} of Algorithm~\ref{alg:bpvi}). The plots show the sub-optimality of $\hat{\pi}^K$ returned by Algorithm~\ref{alg:bpvi} for $K \in [1,\ldots, 1000]$ and various values of episode lengths $H \in \{20,30,50,80\}$. }
    % Note that $\beta = 0$ corresponds to the non-pessimistic variant.  
    % two variants of offline reinforcement learning with pessimism $\hat{\pi}_{unif}$ and $\hat{\pi}_{PEVI}$, defined in Subsection \ref{subsection:bpvi}, under a linear MDP with different horizon length $H \in \{20, 30, 50\}$. Here $\hat{\pi}_{unif}$ exhibits a fast rate of $\mathcal{O}(K^{-1})$ (\textcolor{blue}{blue}) and $\hat{\pi}_{PEVI}$ achieves the zero sub-optimality (\textcolor{red}{red}) after some finite value of $K$. See Appendix Section \ref{section:simulation} for the experimental setup. \thanh{Is FQI good under partial coverage?} 
    \label{fig:fast_rates_sim}
\end{figure*}

% \end{todo}

% \begin{figure*}[t]
%     \centering
%     \includegraphics[scale=0.34]{figs/linmdp.png}
%     \caption{Comparison of BCP-VI with its non-parametric variant (i.e. $\beta=0$). The sub-optimality of $\hat{\pi}^K$ returned by Algorithm \ref{alg:bpvi} for $K \in [1,\ldots, 1000]$ and various values of episode lengths $H \in \{20,30,50,80\}$. The hyperparameter $\beta$ is the pessimism parameter given in Line \ref{bpvi:lcb} of Algorithm \ref{alg:bpvi}. Note that $\beta = 0$ corresponds to the non-pessimistic variant.  
%     % two variants of offline reinforcement learning with pessimism $\hat{\pi}_{unif}$ and $\hat{\pi}_{PEVI}$, defined in Subsection \ref{subsection:bpvi}, under a linear MDP with different horizon length $H \in \{20, 30, 50\}$. Here $\hat{\pi}_{unif}$ exhibits a fast rate of $\mathcal{O}(K^{-1})$ (\textcolor{blue}{blue}) and $\hat{\pi}_{PEVI}$ achieves the zero sub-optimality (\textcolor{red}{red}) after some finite value of $K$. See Appendix Section \ref{section:simulation} for the experimental setup. \thanh{Is FQI good under partial coverage?} 
%     }
%     \label{fig:fast_rates_sim}
% \end{figure*}

% \thanh{Correct terms according to \citep{ito2022adversarially}}
\subsection*{Our Contributions}
\begin{table*}[]
    \centering
    \def\arraystretch{1.9}%
    \resizebox{\textwidth}{!}{
    \begin{tabular}{|c|c|c|c|c|}
    \hline
    {\bf Algorithm} & {\bf Condition} & {\bf Upper Bound} & {\bf Lower Bound} & {\bf Data} \\[5pt]
    \hline 
    % \hline 
      PEVI & Uniform  &  \makecell{$\tilde{\mathcal{O}}  \left( \frac{ H^2 d^{3/2}}{\sqrt{K}} \right)$} & \makecell{$\Omega \left( \frac{H}{ \sqrt{K}} \right)$ } & Independent \\[1mm] 
      \hline
    %   \hline 
      \multirow{3.2}{*}{BCP-VI} & \cellcolor{lightgray}OPC & \cellcolor{gray} $\tilde{\mathcal{O}} \left( \frac{ H^2 d^{3/2} \kappa_*}{\sqrt{K}} \right)$ & \cellcolor{gray} $\Omega \left( \frac{H \sqrt{\kappa_{\min}}}{\sqrt{K}} \right)$ & Adaptive \\
    %   \cline{2-5}
      & \cellcolor{lightgray} OPC, $\Delta_{min}$ & \cellcolor{gray} $\tilde{\mathcal{O}} \left( \frac{ d^3 H^5 \kappa_*^{3}}{\Delta_{\min} \cdot K} \right)$ & \cellcolor{gray} $\Omega \left( \frac{ H^2 \kappa_{\min}}{ \Delta_{\min} \cdot K } \right)$ & Adaptive \\ 
    %   \cline{2-5}
      & \cellcolor{lightgray} OPC, $\Delta_{min}$, UO-SF, $K \geq k^*$ & \cellcolor{gray} $0$ & $0$ & Adaptive \\[5pt] 
      \hline
    %   \hline 
      \multirow{2.5}{*}{BCP-VTR} & \cellcolor{lightgray} OPC & \cellcolor{gray} $\tilde{\mathcal{O}} \left( \frac{ H^2 d \kappa_*}{\sqrt{K}} \right)$ & \cellcolor{gray}$\Omega \left( \frac{H \sqrt{\kappa_{\min}}}{\sqrt{K}} \right)$ & Adaptive \\[5pt]
    %   \cline{2-5}
      & \cellcolor{lightgray} OPC, $\Delta_{min}$ & \cellcolor{gray} $\tilde{\mathcal{O}} \left( \frac{ d^2 H^5 \kappa_*^{3}}{\Delta_{\min} \cdot K} \right)$ & \cellcolor{gray} $\Omega \left( \frac{ H^2 \kappa_{\min}}{ \Delta_{\min} \cdot K } \right)$ & Adaptive \\[5pt]
      \hline 
    \end{tabular}
    }
    \caption{Bounds on the sub-optimality of offline RL with linear function approximation under different conditions and data coverage assumptions. Cells in gray are our contributions. The results
    % \raman{Does the algorithm need the gap information or can it adapt to the gap? The reason I am asking is that it may be more appropriate to write ``under a gap assumption'' rather than ``given the gap information.''} 
    in the first line were obtained in \cite{jin2021pessimism} under ``sufficient'' data coverage. Here,
    $K$ is the number of episodes in the offline dataset, $d$ is the dimension of the known linear mapping, $H$ is the episode length, 
    OPC stands for optimal policy concentrability (Assumption \ref{assumption:single_policy_concentration}), $\kappa_* = \max_{h \in [H]}\kappa_h$ where $\kappa_h$ is the OPC coefficient defined in Assumption \ref{assumption:lower bound density}, $\kappa_{\min} = \min_{h \in [H]} \kappa_h$, $k^*$ is defined in Eq. (\ref{equation: k threshold}), ``Uniform'' means uniform data coverage, ``Independent'' and ``Adaptive'' mean the episodes of the offline data were collected independently and adaptively, respectively, and UO-SF stands for unique optimality and spanning features in Assumption \ref{assumption:spaning_feature}. BCP-VTR is a model-based offline RL method for linear mixture MDPs which is presented in Section \ref{model-based offline RL}.} 
    \label{tab: result for comparison} 
\end{table*}

We propose a novel bootstrapped and constrained pessimistic value iteration (BCP-VI) algorithm to leverage the gap information for an offline learner under partial data coverage, adaptive data, and linear function approximation. The key idea is to apply constrained optimization to the pessimistic value iteration (PEVI) algorithm of \citet{jin2021pessimism} to ensure that each policy estimate has the same support as the behaviour policy. We then repeatedly apply the resulting algorithm to a sequence of partial splits bootstrapped from the original data to form an ensemble of policy estimates. Our key contributions are: 

\begin{enumerate}

\item We show that BCP-VI adapts to the instance-dependent quantity, $\Delta_{\min}$, to achieve a fast rate of $\mathcal{O}(\frac{\log K}{K})$ where $K$ is the number of episodes in the offline data. Our result holds under the single-policy concentration coverage even when the offline data were adaptively collected. 

% This fills the gap in the literature about instance/gap-dependent bounds for offline RL with linear function approximation and adaptive data. 

\item As a byproduct, we also derive strong data-adaptive bounds for offline RL with linear function approximation 
% \raman{I am a bit confused here when comparing it to the prior work.} \thanh{I am thinking of removing the second sentence in this paragraph would make it clearer? We could just state the result without comparison here (as the comparison is not really significant in this case)?} 
under the single-policy concentrability assumption, which readily turns into a $\frac{1}{\sqrt{K}}$ bound with the single-policy concentration coefficients (without the gap information). 
% The previous works in offline linear MDPs~\citep{jin2021pessimism,xie2021bellman} only show how to \raman{turn what into?} turn their learning bounds into a $\frac{1}{\sqrt{K}}$ bound with the uniform data feature coverage assumption but not with the optimal-policy concentrability. 

\item Under an additional condition that the linear features for optimal actions in states reachable by the behavior policy span those in states reachable by an optimal policy, we show that the policies returned by BCP-VI obtain an absolute zero sub-optimality when $K$ is larger than some problem-dependent constant. 

\item We accompany our main result with information-theoretic lower bounds, which show that our gap-dependent bounds for offline RL are nearly optimal up to a polylog factor in terms of $K$ and $\Delta_{\min}$. We summarize our results in Table \ref{tab: result for comparison}.

\end{enumerate}

\section{Related Work}
% \thanh{What is clear distinction between related works and intro? I seem to still mess them up!}

% \paragraph{Offline RL with general function approximation.} The finite sample complexity of offline RL was initially studied with Fitted-Q Iteration with general function approximation by \cite{munos2003error,szepesvari2005finite,antos2007fitted}. \citet{chen2019information,le2019batch,nguyentang2021sample,xie2020q} follow this direction and improve upon the sample complexity. However, they rely on a strong assumption called uniform concentrability coefficient and are only information-theoretic due to the computational intractability of the optimization over a general function class. Later, other works \citep{xie2021bellman,zhan2022offline,Chen2022OfflineRL,uehara2021pessimistic} impose weaker data coverage assumption and their results are the standard worst-case $\frac{1}{\sqrt{K}}$-bound which is not instance-adaptive. 

\paragraph{Offline RL with (linear) function approximation.} While there has been much focus on provably efficient RL under linear function approximation, \citet{jin2021pessimism} were the first to show that pessimistic value iteration is provably efficient for offline linear MDPs. \citet{Xiong2022NearlyMO} and \citet{yinnear} improve  upon \citet{jin2021pessimism} by leveraging variance reduction and data splitting. \citet{xie2021bellman} consider a Bellman-consistency assumption with general function approximation, which improves the bound of \citet{jin2021pessimism} by a factor of $\sqrt{d}$ when realized to finite action spaces and linear MDPs. On the other hand, \citet{wang2020statistical} study the statistical hardness of offline RL with linear representation, suggesting that only realizability and strong uniform data coverage are insufficient for sample-efficient offline RL. Beyond linearity, the sample complexity of offline RL were studied with general, nonparametric or parametric, function approximation, either based on Fitted-Q Iteration (FQI) \citep{DBLP:journals/jmlr/MunosS08,DBLP:conf/icml/0002VY19,chen2019information,duan2021risk,duan2021optimal,hu2021fast,,nguyentang2021sample,ji2022sample} or pessimism principle \citep{uehara2021pessimistic,nguyen2021offline,jin2021pessimism}. However, all of the results above yield a worst-case bound of $\frac{1}{\sqrt{K}}$ without taking into account the structure of a problem instance.

% There has been more focus on studying provably efficient RL under linear function approximation. \citet{jin2021pessimism} were the first to show that pessimistic value iteration is provably efficient for offline linear MDPs. \citet{Xiong2022NearlyMO,yinnear} tighten the bound of \citet{jin2021pessimism} by leveraging variance reduction and data splitting. \citet{xie2021bellman} propose Bellman-consistent assumption with general function approximation which improves the bound of \citet{jin2021pessimism} by a factor of $\sqrt{d}$ when realized to finite action space and linear MDPs. However, all these results provide a worst-case bound of $\frac{1}{\sqrt{K}}$ without leveraging particular problem instances. On the other hand, \citet{wang2020statistical} study the statistical hardness of offline RL with linear representation suggesting that only realizability and strong uniform data coverage are not sufficient for sample-efficient offline RL. 

\paragraph{Instance-dependent bounds for offline RL.}

The gap assumption (Assumption \ref{assumption:margin}) has been studied extensively in online RL~\citep{bubeck2012regret,lattimore_szepesvari_2020}, yielding gap-dependent logarithmic regret bounds for bandits, tabular MDPs \citep{yang2021q} and MDPs with linear representation \citep{he2021logarithmic}. In online RL, when learning MDPs with linear rewards, under an additional assumption that the linear features of optimal actions span the space of the linear features of all actions~\citep{papini2021reinforcement}, we can bound the regret by a constant. However, instance-dependent results for offline RL are still sparse and limited, mainly due to the unique challenge of distributional-shift in offline RL. There are only two instance-dependent works that we are aware of in the context of offline RL. The work of \citet{hu2021fast} establishes a relationship between pointwise error rate of an estimate of $Q^*$ and the rate of the resulting policy in Fitted Q-Iteration (FQI) and Bellman residual minimization under (a probabilistic version of) the minimum positive sub-optimality gap. 
% \raman{Confused by the discussion of the rates in the sentence that follows. Can you please clarify?} 
\citet{hu2021fast} showed that under the uniform feature coverage, i.e. $ \lambda_{\min} \left(\mathbb{E}_{(s_h, a_h) \sim d^{\mu}_h} \left[ \phi_h(s_h, a_h) \phi_h(s_h, a_h)^T  \right] \right) > 0$ and the assumption that gap information is uniformly bounded away from zero with high probability, i.e. $\sup_{\pi} \sP_{s \sim d^{\pi}}(0 < \Delta(s) < \delta) \leq (\delta / \delta_0)^{\alpha}$ for some constants $\delta_0 > 0, \alpha \in [0, \infty]$ and any $\delta > 0$, FQI yields a rate of $\gO(\frac{1}{K})$ in linear MDP and $\mathcal{O}(e^{- K})$ in tabular MDP, respectively. 
% linear FQI can yield a rate of $\mathcal{O}(1/K)$ rate and can be even as fast as $\mathcal{O}(e^{-c \cdot K})$. However, such rates require the uniform feature coverage: $ \lambda_{\min} \left(\mathbb{E}_{(s_h, a_h) \sim d^{\mu}_h} \left[ \phi_h(s_h, a_h) \phi_h(s_h, a_h)^T  \right] \right) > 0$, which is among the strongest data coverage assumption for offline RL. In addition, it is not clear if the exponential rate $\mathcal{O}(e^{-c \cdot K})$ is the fastest achievable rate for offline RL. 
A more recent work of \citet{wang2022gap} obtained gap-dependent bounds for offline RL; however, the results and technique (i.e. so-called the deficit thresholding technique) are limited only to independent data and tabular settings. 

\paragraph{Offline RL from adaptive data.} 

A common assumption for sample-efficient guarantees of offline RL is the assumption that the trajectories of different episodes are collected independently. However, it is quite common in practice that offline data is collected adaptively, for example, using contextual bandits, Q-learning, and optimistic value iteration. Thus, it is natural to study sample-efficient RL from adaptive data. Most initial results with adaptive data are for offline contextual bandits~\citep{Zhan2021,zhan2021policy,nguyen2021offline,zhang2021statistical}. Pessimistic value iteration (PEVI) \citep{jin2021pessimism} works in linear MDP for the general data compliance assumption (see  \cite[Definition 2.1]{jin2021pessimism}), which is essentially equivalent to assuming that the data were adaptively collected. 
% \raman{I am not sure what you mean by the sentence that follows: }
However, when deriving the explicit $\frac{1}{\sqrt{K}}$ bound of their algorithm, they made the assumption that the trajectories are independent (see their Corollary 4.6). 
% ; we argue that it is not necessary. 
% they did not show if their obtained bound under the compliance data assumption can turn into the explicit $\frac{1}{\sqrt{K}}$ bound  unless the the trajectory independence is assumed (see their Corollary 4.6). 
The recent work of \citet{wang2022gap} derives a gap-dependent bound for offline tabular MDP but still requires that trajectories are collected independently.

\section{Problem Setting}

% \raman{I'M HERE}

\paragraph{Episodic time-inhomogenous Markov decision processes (MDPs).}
A finite-horizon Markov decision process (MDP) is denoted as the tuple $\mathcal{M} = (\mathcal{S}, \mathcal{A}, \mathbb{P},r, H, d_1)$, where $\mathcal{S}$ is an arbitrary state space, $\mathcal{A}$ is an arbitrary action space, $H$ the episode length, and $d_1$ the initial state distribution. A time-inhomogeneous transition kernel $\mathbb{P} = \{\mathbb{P}_h\}_{h=1}^H$, where $\mathbb{P}_h: \mathcal{S} \times \mathcal{A} \rightarrow \mathcal{P}(\mathcal{S})$ (where $\mathcal{P}(\mathcal{S})$ denotes the set of probability measures over $\mathcal{S}$) maps each state-action pair $(s_h, a_h)$ to a probability distribution $\mathbb{P}_h(\cdot|s_h, a_h)$ (with the corresponding density function $p_h(\cdot|s_h, a_h)$  with respect to the Lebesgue measure $\rho$ on $\mathcal{S}$), and $r = \{r_h\}_{h=1}^H$ where $r_h: \mathcal{S} \times \mathcal{A} \rightarrow [0,1]$ is the mean reward function at step $h$. A policy $\pi = \{\pi_h\}_{h=1}^H$ assigns each state $s_h \in \mathcal{S}$ to a probability distribution, $\pi_h(\cdot|s_h)$, over the action space  and induces a random trajectory $s_1, a_1, r_1, \ldots, s_H, a_H, r_H, s_{H+1}$ where $s_1 \sim d_1$, $a_h \sim \pi_h(\cdot|s_h)$, $s_{h+1} \sim \mathbb{P}_h(\cdot | s_h, a_h)$. 

\paragraph{$V$-values and $Q$-values.} For any policy $\pi$, the $V$-value function $V^{\pi}_h \in \mathbb{R}^{\mathcal{S}}$ and the $Q$-value function $Q^{\pi}_h \in \mathbb{R}^{\mathcal{S} \times \mathcal{A}}$ are defined as: 
% \begin{align*}
   $ Q^{\pi}_h(s,a) = \mathbb{E}_{\pi}[\sum_{t=h}^H r_t | s_h=s, a_h=a]$, 
    $V^{\pi}_h(s) = \mathbb{E}_{a \sim \pi(\cdot |s)} [Q^{\pi}_h(s,a)]$. We also define 
% \begin{align*}
    $({P}_h V)(s,a) := \mathbb{E}_{s' \sim \mathbb{P}_h(\cdot|s,a)}[V(s')]$,  
    $(T_h V)(s,a) := r_h(s,a) + (P_hV)(s,a)$,
% \end{align*}
We have: $Q^{\pi}_h = T_h V_{h+1}^{\pi}$ (the Bellman equation), $V^{\pi}_h(s) = \mathbb{E}_{a \sim \pi(\cdot|s)}[Q^{\pi}_h(s,a)]$, $Q^*_h = T_h V_{h+1}^*$ (the Bellman optimality equation), and $V^*_h(s) = \max_{a \in \mathcal{A}} Q^*_h(s,a)$. Let $\pi^* = \{\pi^*_h\}_{h \in [H]}$ be any deterministic, optimal policy, i.e., $\pi^* \in \argmax_{\pi} Q^{\pi} $ and denote $v^* = v^{\pi^*}$. Moreover, let $d^{\mathcal{M},\pi}_h$ be the marginal state-visitation density for policy $\pi$ at step $h$ with respect to the Lebesgue measure $\rho$ on $\mathcal{S}$, i.e., $ \int_{B} d^{\mathcal{M}, \pi}_h(s_h) \rho(ds_h) =$ $\mathbb{P}\left(s_h \in B |  d_1, \pi,   \mathbb{P} \right)$. We overload the notation $ d^{\mathcal{M}, \pi}_h(s_h, a_h) = d^{\mathcal{M}, \pi}_h(s_h) \pi(a_h | s_h)$ for the state-action visitation density when the context is clear. We abbreviate $d^{\mathcal{M},*}_h = d^{\mathcal{M},\pi^*}_h$. Let $\mathcal{S}^{\mathcal{M},\pi}_h := \{s_h: d^{\mathcal{M}, \pi}_h(s_h) > 0\}$ and $\mathcal{S}\mathcal{A}^{\mathcal{M},\pi}_h := \{(s_h,a_h): d^{\mathcal{M}, \pi}_h(s_h, a_h) > 0\}$ be the set of feasible states and feasible state-action pairs, respectively at step $h$ under policy $\pi$. Denote by $\mathcal{S}^{\mathcal{M}}_h = \cup_{\pi} \mathcal{S}^{\mathcal{M},\pi}_h $ and $\mathcal{S} \mathcal{A}^{\mathcal{M}}_h = \cup_{\pi} \mathcal{S} \mathcal{A}^{\mathcal{M},\pi}_h $ the set of all feasible states and feasible state-action pairs, respectively at step $h$. When the underlying MDP is clear, we drop the superscript $\mathcal{M}$ in $d^{\mathcal{M}, \pi}$, $d^{\mathcal{M}, *}$, $\mathcal{S}^{\mathcal{M}, \pi}$, and $\mathcal{S} \mathcal{A}^{\mathcal{M}, \pi}$ to become $d^{\pi}$, $d^{ *}$, $\mathcal{S}^{ \pi}$, and $\mathcal{S}\mathcal{A}^{\pi}$ respectively. We assume bounded marginal state(-action) visitation density functions and without loss of generality, we assume $d^{\pi}_h(s_h, a_h) \leq 1, \forall (h,s_h, a_h, \pi)$.\footnote{This trivially holds when $\mathcal{S}$ and $\mathcal{A}$ are discrete regardless of how large they are). When either $\mathcal{S}$ or $\mathcal{A}$ are continuous, we assume $d^{\pi}_h(s_h, a_h) \leq B < \infty$ and assume $B=1$ for simplicity.}

\paragraph{Linear MDPs.} When the state space is large or continuous, we often use a parametric representation for value functions or transition kernels. A standard parametric representation is linear models with given feature maps. In this paper, we consider such linear representation with the linear MDP \citep{yang2019sample,jin2020provably} where the transition kernel and the rewards are linear with respect to a given $d$-dimensional feature map: $\phi_h: \mathcal{S} \times \mathcal{A} \rightarrow \mathbb{R}^d$. 
% \footnote{We leave the extension of our analysis for the related linear mixture model \citep{cai2020provably,zhou2021nearly} for future work.}
% The formal definition of linear MDP is given in Definition \ref{definition:linear_mdp}. 
\begin{defn}[Linear MDPs]
An MDP has a linear structure if for any $(s,a,s',h)$,  
\begin{align*}
    r_h(s,a) = \phi_h(s,a)^T \theta_h, \mathbb{P}_h(s'|s,a) = \phi_h(s,a)^T \mu_h(s'),
\end{align*}
for some $\theta_h \in \mathbb{R}^d$ and some $\mu_h: \mathcal{S} \rightarrow \mathbb{R}^d$. For simplicity, we further assume that $\|\theta_h\|_2 \leq \sqrt{d}$, $ \| \int \mu_h(s) v(s) ds \|_2 \leq \sqrt{d} \| v\|_{\infty}$ for any $v: \mathcal{S} \rightarrow \mathbb{R}$ and $\| \phi_h(s,a) \|_2 \leq 1$.
\label{definition:linear_mdp}
\end{defn}
\begin{remark}
The linear MDP can be made practical with contrastive representation learning \citep{zhang2022making}. 
We only consider linear MDP in the main paper but also consider a linear mixture model \citep{cai2020provably,zhou2021nearly} in Section \ref{model-based offline RL}. 
\end{remark}
%  \raman{add notation to the preliminaries.}

% \begin{todo}
% Say why linear MDPs is an important class. 
% \end{todo}

% \begin{todo}
% Extend to mixture model for model-based methods. But if time not allowed, do this for a journal version after NeurIPS.
% \end{todo}

\paragraph{Offline Regime.}
In the offline learning setting, the goal is to learn the policy $\pi$ that maximizes $v^{\pi}$, given the historical data $\mathcal{D} = \{(s^t_h, a^t_h, r^t_h)\}^{t \in [K]}_{h \in [H]}$ generated by some unknown behaviour policy $\mu = \{\mu_h\}_{h \in [H]}$. Here, 
% we do not assume that episode trajectories $\{(s^t_h, a^t_h, r^t_h)\}_{h \in [H]}$ are generated independently but allow them to have been collected in an adaptive manner, i.e. 
we allow the trajectory at any episode $k$ to depend on the trajectories at all the previous episodes $t < k$.
% $a_h^k \sim \mu_h(\cdot| s_h^k, \{(s_{h'}^t,a_{h'}^t,r_{h'}^t)\}_{h' \in [H]}^{t \in [k-1]})$ (and $r^k_h \sim r_h(s^k_h, a^k_h)$ and $s^k_{h+1} \sim \mathbb{P}_h(\cdot|s^k_h, a^k_h)$) for any $(k,h) \in [K] \times [H]$.
This reflects many practical scenarios where  episode trajectories are collected adaptively by some initial online learner (e.g., $\epsilon$-greedy, Q-learning, and LSVI-UCB).
% \footnote{Even the offline data can be collected adaptively, once the (active) data collection is finished, $\mu_h(a_h | s_h, \mathcal{D})$ is Markovian and for simplicity we drop the conditional $\mathcal{D}$ when writing $\mu$.}
% To be more precise, $\mu_h(a_h|s_h) = \mu_h(a_h|s_h, \{(s_{h'}^t,a_{h'}^t,r_{h'}^t)\}_{h' \in [H]}^{t \in [k-1]}))$ if $s_h = s_h^k$ and $\mu_h(a_h|s_h) = \mu_h(a_h|s_h, \mathcal{D})$ if $s_h \notin \{s_h^k\}_{k \in [K]}$. This technical nuance does not affect our result and for simplicity we always abbreviate $\mu$ as $\mu_h(a_h|s_h)$.} 

In this paper, we assume that the support of $\mu_h(\cdot | s_h)$ for each $s_h$ and $h$, denoted by $\textrm{supp}(\mu_h(\cdot | s_h))$ is known to the learner. We also denote the $\mu$-supported policy class at stage $h$, denoted by $\Pi_h(\mu)$ as the set of policies whose supports belong to the support of the behavior policy: 
\begin{align}
    \Pi_h(\mu)\! :=\! \{\pi_h\!: \textrm{supp}(\pi_h(\cdot|s_h)) \!\subseteq\!  \textrm{supp}(\mu_h(\cdot|s_h)),\! \forall s_h\! \in\! \mathcal{S}_h \}.
    \label{eq:constrained_policy_class}
\end{align}

\paragraph{Performance metric.} We measure the performance of policy $\hat{\pi}$ via the sub-optimality metric: $\subopt( \hat{\pi} ) : = \mathbb{E}_{s_1 \sim d_1} \left[ \subopt(\hat{\pi}; s_1) \right]$, where
    $\subopt(\hat{\pi}; s) := V^{\pi^*}_1(s_1) - V^{\hat{\pi}}_1(s_1)$. 
% $$\subopt(\pi; s) := V^{\pi^*}_1(s) - V^{\pi}_1(s_1)$$ which measures the sub-optimality gap of taking actions according to $\hat{\pi}$ in state $s$.
% ~\footnote{Note that many existing offline RL works often use a stronger notion of sub-optimality $\subopt(\pi; s) := V^{\pi^*}_1(s) - V^{{\pi}}(s)$. Our weaker notion is more consistent with the regret notion in online learning and can capture more behaviors of offline RL. It is also easy to convert the bounds of existing works into our sub-optimality notion for comparison.} 
% Define $\subopt(\pi) = \mathbb{E}_{s \sim d_1} \left[ \subopt(\pi; s) \right]$. 
As $\hat{\pi}$ is learnt from the offline data $\mathcal{D}$, $\subopt(\pi)$ is random (with respect to the randomness of $\mathcal{D}$ and possibly internal randomness of the offline algorithm). The goal of offline RL is to learn $\hat{\pi}$ from $\mathcal{D}$ such that $\subopt(\hat{\pi})$ is small with high probability. 

% \raman{w.h.p.} \\

% , and $\mathbb{E} \left[\subopt(\pi) | \mathcal{D}\right]$, where $\mathbb{E}$ is taken over any additional randomness of $\pi$ (e.g. via randomized design) conditioned on $\mathcal{D}$. Note that  $\mathbb{E} \left[\subopt(\pi) | \mathcal{D}\right]$ is random as $\mathcal{D}$ is random, and $\subopt(\pi) = \mathbb{E}_{s \sim d_1} \left[ \subopt(\pi; s) \right]$. 

% $\mathbb{E} \left[ \subopt(\hat{\pi}) \right]$ is as small as possible. Equivalently, we want to minimize $L(\hat{\pi}; K) := \mathbb{E} \left[\subopt(\hat{\pi}) | \mathcal{D}_K \right]$ with high probability (over the randomness of $\mathcal{D}$), where here $\mathbb{E}$ is taken over any additional randomness of $\hat{\pi}$ (e.g. by algorithmic design) conditioned on $\mathcal{D}$.   \\

% \subsection{Linear representation}

% \subsubsection{Model-free RL}
% We consider model-free RL with linear function approximation. 

\section{Bootstrapped and Constrained Pessimistic Value Iteration}
\label{section:instance_agnostic_main}
In this section, we describe our main algorithm and establish both instance-agnostic and instance-dependent bounds for offline RL from adaptive data with linear function approximation. Through this algorithm, we show that offline RL achieves a generic data-dependent bound under the optimal-policy concentrability and adapts to the gap information to accelerate to the $\frac{\log K}{K}$ bound and even obtain zero sub-optimality when the optimal linear features under the behavior policy spans those under an optimal policy.

% for offline RL with pessimism and establish its instance-agnostic bounds (i.e. worst-case bounds). 

\subsection{Algorithm}
\label{subsection:bpvi}
% \paragraph{High-level idea.} 
We build upon the Pessimistic Value Iteration (PEVI) algorithm \citep{jin2021pessimism} with two additional modifications: bootstrapping and constrained optimization, thus the name Bootstrapped and Constrained Pessimistic Value Iteration (BCP-VI) in Algorithm \ref{alg:bpvi}. The constrained optimization in Line \ref{bpvi:greedy} ensures that the extracted policy is supported by the behaviour policy. The bootstrapping part divides the offline data in a progressively increasing split and applies the constrained version of PEVI in each split to form an ensemble (Line \ref{bcpvi:ensemble}).\footnote{To be precise, this is not exactly bootstrapping in the traditional sense where the data is sampled with replacement and the ensemble is used to estimate uncertainty. Here we instead use progressive data splits to deal with adaptive data and form an ensemble of policy estimates.}
% \footnote{\citet{nguyen2021offline} first use this bootstrapping idea (known as a streaming technique in their work) to handle non-pointwise uncertainty of neural nets for generalization in offline contextual bandits. Here we instead use it to leverage the gap information.} 
The additional modifications are highlighted in \textcolor{red}{red} in  Algorithm \ref{alg:bpvi}.  

Overall, BCP-VI estimates the optimal action-value functions $Q^*_h$ leveraging its linear representation. In Line~\ref{bpvi:least_square}, it solves the regularized least-squares regression on $\mathcal{D}^{k-1}$:   
\begin{align*}
    \hat{w}_h :=\argmin_{w \in \mathbb{R}^d}  \sum_{i=1}^{k} [ \langle \phi(s_h^i, a_h^i), w \rangle - r_h^i - V_{h+1}(s^i_{h+1}) ]^2 + \lambda \| w \|_2^2.
\end{align*}    
In Line \ref{bpvi:lcb}, BCP-VI computes the action-value functions using $\hat{w}_h$, then offsets it with a bonus function $b_h$ to ensure a  pessimistic estimate. In Line \ref{bpvi:greedy}, we extract policy $\hat{\pi}_h$ that is most greedy with respect to $\hat{Q}_h$ among the set of all policies $\Pi_h(\mu)$. 
% defined within the support of the behaviour policy. 

\paragraph{Policy execution.} Given the policy ensemble $\{\hat{\pi}^k: k \in [K + 1]\}$, we consider two policies from the ensemble as the execution policy: the \emph{mixture} policy $\hat{\pi}^{mix}$ and the \emph{last-iteration} policy $\hat{\pi}^{last}$, defined as:
% \begin{align*}
    $\hat{\pi}^{mix} := \frac{1}{K} \sum_{k=1}^K \hat{\pi}^k, \text{ and } \hat{\pi}^{last} := \hat{\pi}^{K+1}$. 
% \end{align*}
Note that $\hat{\pi}^{last}$ is similar to the PEVI policy in \cite{jin2021pessimism}.

\begin{algorithm}
\caption{Bootstrapped and Constrained Pessimistic Value Iteration (BCP-VI) }
\begin{algorithmic}[1]
\State \textbf{Input:} Dataset $\mathcal{D} = \{(s^t_h, a^t_h, r^t_h)\}_{h \in [H]}^{t \in [K]}$, uncertainty parameters $\{\beta_k\}_{k \in [K]}$, regularization hyperparameter $\lambda$, \textcolor{red}{$\mu$-supported policy class $\{\Pi_h(\mu)\}_{h \in [H]}$}. 
% \State Set $\Sigma_h^0$
\For{ \textcolor{red}{$k = 1, \ldots, K + 1$}}
    \State $\hat{V}_{H+1}^k(\cdot) \leftarrow 0$.
    \For{step $h = H, H-1, ..., 1$}
        \State $\Sigma_h^k \leftarrow \sum_{t=1}^{k-1} \phi_h(s^t_h, a^t_h) \cdot \phi_h(s^t_h, a^t_h)^T  + \lambda \cdot I$. 
        \State $\hat{w}_h^k \leftarrow (\Sigma_h^k)^{-1} \sum_{t=1}^{k-1} \phi_h(s_h^t, a_h^t) \cdot (r^t_h + \hat{V}_{h+1}^k(s^t_{h+1}))$. 
        \label{bpvi:least_square}
        \State $b_h^k(\cdot,\cdot) \leftarrow \beta_k \cdot  \| \phi_h(\cdot, \cdot) \|_{(\Sigma_h^k)^{-1}}$. 
        \label{bpvi:lcb}
        \State $\bar{Q}_h^k(\cdot, \cdot) \leftarrow \langle \phi_h(\cdot, \cdot), \hat{w}_h^k \rangle - b_h^k(\cdot, \cdot)$. 
        \State $\hat{Q}_h^k(\cdot, \cdot) \leftarrow \min\{\bar{Q}_h^k(\cdot, \cdot), H - h +1\}^+$. 
        \State $\hat{\pi}_h^k \leftarrow \displaystyle \textcolor{red}{ \argmax_{\pi_h \in \Pi_h(\mu)} \langle \hat{Q}_h^k, \pi_h \rangle}$
        % $\Pi_h(\mu)$ 
        % defined in Eq. (\ref{eq:constrained_policy_class})
        \label{bpvi:greedy}
        \State $\hat{V}_h^k(\cdot) \leftarrow \langle \hat{Q}_h^k(\cdot, \cdot), \pi_h^k(\cdot|\cdot) \rangle$.
    \EndFor
\EndFor
% \Ensure $\theta^{(J)}$
\State \textbf{Output:} Ensemble $\{\hat{\pi}^k: k \in [K + 1]\}$. 
\label{bcpvi:ensemble}
\end{algorithmic}
\label{alg:bpvi}
\end{algorithm}

\paragraph{Practical consideration.} In practice where the action space is large, the constrained optimization in Line \ref{bpvi:greedy} could be relaxed into the regularization optimization $\max_{\pi_h}  \langle \hat{Q}_h^k, \pi_h \rangle + \gamma \mathrm{KL}[\pi_h \| \mu_h]$ for some $\gamma > 0$ (in the setting the behavior policy $\mu$ is not given, it can be simply estimated from the data). The relaxed regularization optimization assures that $\hat{\pi}^k_h$ is supported by $\mu_h$ and can be solved efficiently using an actor-critic framework. It is possible to include the optimization error of this actor-critic framework with a more involved analysis \citep{xie2021bellman,zanette2021provable,cheng2022adversarially}; we, however, ignore this here for simplicity.

\subsection{Data-dependent minimax bounds}
\label{subsection:main_single_concentrability}
% We discuss data coverage assumptions here. 

% \begin{assumption}[Positive cumulative reward coverage]
% At any stage $h \in [H]$, for any feasible state $s_h$ under $\mu$ (i.e. $d^{\mu}_h(s_h) > 0$), there exists a trajectory $(s_h, a_h, \ldots, s_H, a_H)$ feasible under $\mu$ such that the cumulative reward is positive (i.e. $\sum_{i=h}^H r_i(s_i,a_i) > 0$).
% \label{assumption:positive_cumulative_reward_coverage}
% \end{assumption}
% Assumption \ref{assumption:positive_cumulative_reward_coverage} is a minimal assumption as otherwise $\mu$ would cover only zero-reward actions at a stage $h$, thus it should not be used for offline learning at the first place because in this case it provides zero information about any positive-reward actions. \footnote{As a concrete example, if $\mu$ does not satisfy Assumption \ref{assumption:positive_cumulative_reward_coverage}, then $r^h =$}

% \subsubsection{Data coverage}
Sample-efficient offline reinforcement learning is not possible without certain data-coverage assumptions \citep{wang2020statistical}. In this work, we rely on the optimal-policy concentrability (Assumption \ref{assumption:single_policy_concentration}) which ensures that $d^{\mu}$ covers the trajectory of some optimal policy $\pi^*$ and can be agnostic to other locations.

\begin{assumption}[Optimal-Policy Concentrability (OPC) \citep{DBLP:conf/uai/LiuSAB19}]
There is an optimal policy $\pi^*$: $ \forall (h,s_h, a_h)$, 
% \begin{align*}
    $d^{\pi^*}_h(s_h, a_h) > 0 \implies d^{\mu}_h(s_h, a_h) > 0.$
% \end{align*}
% \footnote{By measure-theoretic terminology, this says that $d^{\mu}$ dominates $d^{\pi^*}$, thus a more descriptive name for this assumption should be ``density dominance'' assumption. However, we have decided to use ``single-policy concentrability'' as concentrability seems to be a more widely adopted term  in the offline RL community.}
\label{assumption:single_policy_concentration}
\end{assumption}
\begin{remark}
Consider any $s_h \in \mathcal{S}^{\pi^*}_h$. If $\pi^*_h(a_h | s_h) > 0$, then $d^{\pi^*}_h(s_h, a_h) > 0$, and thus $d^{\mu}_h(s_h, a_h) > 0$ by Assumption \ref{assumption:single_policy_concentration} which implies that $\mu_h(a_h | s_h) > 0$. For any $s_h \notin \mathcal{S}_h^{\pi^*}$, $\pi^*_h(\cdot| s_h)$ has no impact on the optimal value function $\{V^*_h \}_ {h \in [H] }$. Thus, without loss of generality, we can assume that $\textrm{supp}(\pi^*_h(\cdot| s_h)) \subseteq \textrm{supp}(\mu_h(\cdot| s_h)), \forall s_h \notin \mathcal{S}_h^{\pi^*}$. Overall, we have
% \begin{align*}
   $ \pi^*_h \in \Pi_h(\mu), \forall h \in [H]$.
% \end{align*}
\label{remark: optimal policy belongs to the constrained class}
\end{remark}
Assumption \ref{assumption:single_policy_concentration} is arguably the weakest data coverage assumption for sample-efficient offline RL, i.e., to ensure an optimal policy is statistically learnable from offline data (see Appendix \ref{subsection:single_concentrability_is_necessary} for a proof that the OPC condition is necessary). As such, Assumption \ref{assumption:single_policy_concentration} is significantly weaker than \textit{uniform} data coverage assumption which features in most existing works in offline RL e.g., the uniform feature coverage~\citep{duan2020minimax,yinnear}, i.e., for all $h \in [H]$, 
% \begin{align*}
    $\lambda_{\min} \left(\mathbb{E}_{(s_h, a_h) \sim d^{\mu}_h} \left[ \phi_h(s_h, a_h) \phi_h(s_h, a_h)^T  \right] \right) > 0, \text{ or } 
    \min_{h,s_h, a_h} d^{\mu}_h(s_h, a_h) > 0$, 
% \end{align*}
and the classical uniform concentrability \citep{szepesvari2005finite,chen2019information,nguyentang2021sample}, i.e., $ \sup_{\pi,h, s_h, a_h} \frac{d^{\pi}_h(s_h, a_h)} { d^{\mu}_h(s_h, a_h)}  < \infty$.

We further assume the positive occupancy density under $\mu$ is bounded away from $0$. 
\begin{assumption}
$\kappa_h^{-1} := \inf_{(s_h, a_h): d^{\mu}_h(s_h,a_h) > 0} d^{\mu}_h(s_h,a_h) > 0$, $\forall h \in [H]$.  
\label{assumption:lower bound density}
\end{assumption}
% \begin{remark}

% \label{remark: lower bound for density of state-only}
% \end{remark}
Here, the infimum is over only the feasible state-action pairs under $\mu$ and it is agnostic to other locations. For example, the assumption is automatically satisfied when the state-action space is finite (but can be exponentially large). We remark that Assumption \ref{assumption:lower bound density} is significantly milder than the uniform data coverage assumption $d_m := \inf_{h, s_h ,a_h} d^{\mu}_h(s_h, a_h) > 0$ in \cite{yin2020near} as the infimum in the latter is uniformly over all states and actions. \footnote{We notice that, interestingly, this assumption were also independently used in the prior work of \cite{yin2021near}.} Note that Assumption \ref{assumption:lower bound density} also implies that $d^{\mu}_h(s_h) = \frac{d^{\mu}_h(s_h, a_h)}{\mu_h(a_h | s_h)} \geq \kappa_h^{-1}$ for any $s_h \in \mathcal{S}^{\mu}_h$. Combing with Assumption \ref{assumption:single_policy_concentration}, $\kappa_h$ can be seen as (an upper bound on) the \emph{OPC coefficient} at stage $h$ as we have 
% \begin{align*}
    $\frac{d_h^{\pi^*}(s_h, a_h)}{d_h^{\mu}(s_h, a_h)} \leq \kappa_h, \forall (h,s_h, a_h) \in [H] \times \mathcal{S} \times \mathcal{A}$.

We fix any $\delta \in (0,1]$ and set $\lambda = 1$ for simplicity and  $\beta_k = \beta_k(\delta) := c_1  \cdot dH \log(dHk/\delta)$ for some absolute constant $c_1 > 0, \forall k \in [K]$ in Algorithm \ref{alg:bpvi}. We now present our data-dependent bound. 
\begin{theorem}[Data-dependent bound]
Under Assumption \ref{assumption:single_policy_concentration}, with probability at least $1 - 4 \delta$ over the randomness of $\mathcal{D}$, we have: 
\begin{align*}
\subopt(\hat{\pi}^{mix}) \lor
\subopt(\hat{\pi}^{last})  
% \subopt(\hat{\pi}_{PEVI}) 
\leq  \frac{4 \beta({\delta})}{K} \sum_{h=1}^H  \sum_{k=1}^K  \frac{d^*_h(s_h^k,a_h^k)}{d^{\mu}_h(s_h^k,a_h^k)} \|  \phi_h(s^k_h, a^k_h)  \|_{(\Sigma^k_h)^{-1}} \\
+\frac{4 \beta({\delta})}{K} \sum_{h=1}^H \sqrt{\log \left(\frac{H}{\delta} \right) \sum_{k=1}^K \left(\frac{d^*_h(s_h^k,a_h^k)}{d^{\mu}_h(s_h^k,a_h^k)} \right)^2 }  
+ \frac{2}{K} +  \frac{16 H}{3 K}\log\left( \frac{\log_2(KH)}{ \delta} \right).
\end{align*}
\label{theorem:sublinear_subopt}
\end{theorem}
% \thanh{Need to organize the proof.} \thanh{It is very crucial to condition on $s^1_k$ to obtain weighted elliptical potentials $\| \phi_h(s_h^k, a_h^k) \|_{(\Sigma^k_h)^{-1}}$}. \thanh{How to build an asymptotic lower bound for offline RL?}

% Theorem \ref{theorem:sublinear_subopt} establishes a generic data-dependent bound for both $\hat{\pi}_{unif}$ and $\hat{\pi}_{PEVI}$ from the BPVI ensemble under the single-policy concentrability, where
\begin{remark}
The first term in the bound in Theorem \ref{theorem:sublinear_subopt} is the elliptical potential that emerges from pessimism and is dominant while the other terms are generalization errors by concentration phenomenon and peeling technique.
\end{remark}
The sub-optimality bound in Theorem \ref{theorem:sublinear_subopt} explicitly depends on the observed data in the offline data via the marginalized density ratios $\frac{d^*_h(s_h^k,a_h^k)}{d^{\mu}_h(s_h^k,a_h^k)}$ (which is valid thanks to Assumption \ref{assumption:single_policy_concentration}).  One immediate consequence of the data-dependent bound in Theorem \ref{theorem:sublinear_subopt} is that the bound can turn into a weaker yet more explicit rate of $\frac{1}{\sqrt{K}}$ in Corollary \ref{corollary:sublinear_subopt}. 
% \raman{By upto a factor of $\sqrt{K}$?} \raman{What other remarks can we make here before we move to the corollary? Can we say something about what terms are dominant and/or what is the potential source for each term in the bound?}

% The bound in Theorem \ref{theorem:sublinear_subopt} is data-adaptive as it explicitly depends on the density ratio $\frac{d^*_h(s_h^k,a_h^k)}{d^{\mu}_h(s_h^k,a_h^k)}$ that can be much smaller than the conservative coefficient $\kappa_{\infty} := \max_{h, s_h, a_h} \frac{d^*_h(s_h,a_h)}{d^{\mu}_h(s_h, a_h)}$ at many locations.

\begin{corollary}
% Let $\kappa_h := \sup_{s_h, a_h} \frac{d^*_h(s_h,a_h)}{d^{\mu}_h(s_h, a_h)}, \forall h \in [H]$ and $\kappa := \sum_{h = 1}^H$. \footnote{Note that \citet{rashidinejad2021bridging,xie2021policy} consider this assumption with a single concentrability coefficient $\max_{h, s_h, a_h} \frac{d^*_h(s_h,a_h)}{ d^{\mu}_h(s_h, a_h)} < \infty$.  A single concentrability coefficient can be too conservative as the density ratio at different step $h$ can be much smaller than $\max_{h, s_h, a_h} \frac{d^*_h(s_h,a_h)}{ d^{\mu}_h(s_h, a_h)}$.} 
Under Assumptions  \ref{assumption:single_policy_concentration}-\ref{assumption:lower bound density}, with probability at least $1 - \Omega(1/K)$ over the randomness of $\mathcal{D}$, we have: 
% \raman{Keep the terms in the bound in the same order as in the Theorem above.}
% \begin{align*}
%      &\subopt(\hat{\pi}^{mix}) \lor \subopt(\hat{\pi}^{last}) \\
%      &\leq  \frac{4 \beta(\delta) \kappa }{ \sqrt{K}} \sqrt{2 d \log\left(1 + \frac{K}{d} \right)} + \frac{4 \beta(\delta) \kappa}{ \sqrt{K}}  \sqrt{ \log\left(\frac{H}{\delta}\right) }  \\
%      &+ \frac{2}{K} + \frac{16 H}{3 K}\log\left( \frac{\log_2(KH)}{ \delta} \right),
% \end{align*}
\begin{align*}
    \mathbb{E} \left[ \subopt(\hat{\pi}^{mix}) \right] \lor \mathbb{E} \left[ \subopt(\hat{\pi}^{last}) \right] = \tilde{\mathcal{O}} \left( \frac{\kappa_* H^2 d^{3/2}}{\sqrt{K}} \right),
\end{align*}
where $\kappa_* := \max_{h \in [H]} \kappa_h$. 
\label{corollary:sublinear_subopt}
\end{corollary}
% \begin{remark}
% If we set the $\delta$ in Corollary \ref{corollary:sublinear_subopt} as $\delta = \mathcal{O}(1/ \sqrt{K})$, for the expected sub-optimality bound, we have:
% \begin{align*}
%     \mathbb{E} \left[ \subopt(\hat{\pi}^{mix}) \right] \lor \mathbb{E} \left[ \subopt(\hat{\pi}^{last}) \right] = \tilde{\mathcal{O}} \left( \frac{\kappa H d^{3/2}}{\sqrt{K}} \right).
% \end{align*}
% \end{remark}

% We give a high-level overview of the proof of Theorem \ref{theorem:sublinear_subopt} in Section \thanh{ref} and defer the full proof to Appendix \thanh{ref}. 

As the $\frac{1}{\sqrt{K}}$ bound is minimax, we compare our result to other existing works. 
% To our knowledge, Corollary \ref{corollary:sublinear_subopt} is the first explicit $\frac{1}{\sqrt{K}}$-type bound of offline RL in linear MDPs under the OPC assumption. We now compare our data-adaptive bound in Theorem \ref{theorem:sublinear_subopt} and Corollary \ref{corollary:sublinear_subopt} to the existing related works: 
% \thanh{I AM HERE!}
% Here, the bound can be simplified further into a sub-linear bound $\tilde{\mathcal{O}}( \kappa H d^{3/2} K^{-1/2})$. 
% % As $\beta_k = c_1  \cdot dH \log(dHk/\delta)$ and the first term is dominant, the bound in Corollary \ref{corollary:sublinear_subopt} can be further simplified into $\tilde{\mathcal{O}}( \kappa H d^{3/2} K^{-1/2})$.
% % Theorem \ref{theorem:sublinear_subopt} (and Corollary \ref{corollary:sublinear_subopt}) prepares the ground for fast rates of offline RL in the remainder of this work, and 
% To our knowledge, Corollary \ref{corollary:sublinear_subopt} is the first to provide the near-optimal sample-efficient bound of offline RL in linear MDPs with the single-policy concentrability. We compare our bound with the existing related works: 

% \footnote{The results in Theorem \ref{theorem:sublinear_subopt} and Corollary \ref{corollary:sublinear_subopt} are in fact proved for the strong notion of sub-optimality used in \citep{yin2021towards,jin2021pessimism}, thus they are comparable.}\\

% Finally, we compare our bound with the related works. \\
\paragraph{Comparing with \cite{yin2021towards}.} \citet{yin2021towards} also use OPC to establish the intrinsic offline learning bound with pessimism and leverage the variance information to obtain a tight dependence on $H$. Their result is valid only for tabular MDPs with the finite state space and finite action space and cannot generalize to linear MDPs. 
\paragraph{Comparing with \cite{jin2021pessimism}.} Similarly, \citet{jin2021pessimism} also consider linear MDPs with pessimism and provide a generic bound under arbitrary data coverage. They then realize their generic bound in the uniform feature coverage assumption \citep{duan2020minimax,yin2021towards} to obtain a sub-optimality bound of $\tilde{\mathcal{O}}(\frac{d^{3/2}H^2}{\sqrt{K}})$. However, the uniform feature coverage is not necessary to obtain the $\frac{1}{\sqrt{K}}$ bound; in our result, we demonstrate that OPC is sufficient to obtain the $\frac{1}{\sqrt{K}}$ bound.

% However, the uniform feature coverage does not need pessimism to obtain such a bound and \citet{duan2020minimax,yin2021towards} do not answer the question of how to obtain the $\frac{1}{\sqrt{K}}$ bound with only OPC.

\paragraph{Comparing with \cite{xie2021bellman}.} 
% ,zhan2022offline,Chen2022OfflineRL
\citet{xie2021bellman} consider Bellman-consistent pessimism for offline RL with general function approximation, where they maintain a version space of all functions that have small Bellman evaluation error and select a function from the version space that has the smallest initial value. Their algorithm is however computationally intractable in general. When realized to linear MDPs, they do have a tractable algorithm but its guarantee requires the behaviour policy to be explorative all dimensions of the feature mapping, i.e., $\mathbb{E}_{\mu} [\phi(s,a) \phi(s,a)^T] \succ 0$. We do not require such assumption in our guarantee. 
% is strictly positive definite. 

% \paragraph{Comparing with \citep{xie2021bellman,zhan2022offline,Chen2022OfflineRL}.} \citet{xie2021bellman,zhan2022offline,Chen2022OfflineRL} consider offline RL with general function approximation under OPC. While their results are valuable for the offline RL literature, their algorithms are more complex and, in general, intractable, in practice. When realized to linear MDPs, the algorithm in \cite{xie2021bellman} is computationally tractable but their guarantee requires the behaviour policy to be explorative all dimensions of the feature mapping i.e. $\mathbb{E}_{\mu} [\phi(s,a) \phi(s,a)^T]$ is strictly positive definite. 

Theorem \ref{theorem:sublinear_subopt} is a byproduct that sets the stage for our instance-dependent bounds in the following section; nonetheless, it is  the first, to the best of our knowledge, that provides an explicit $\frac{1}{\sqrt{K}}$ bound for linear MDPs with OPC.

% Before finishing this section, we discuss two extensions to our algorithm.
% However, they do not derive a strong data-adaptive bound and their algorithms are in general intractable. \thanh{Discuss this more}
% are concurrent works and they work for general function approximation. They focus on learning density-ratio functions to remove completeness assumptions for offline RL. We have a strong data-adaptive bound. 
% \paragraph{BCP-VI under arbitrary data coverage.}
\begin{remark}
We show in Appendix \ref{subsection:single_concentrability_is_necessary} that OPC is necessary to guarantee sublinear sub-optimality bound for offline RL. When OPC fails to hold, we show in Appendix \ref{subsection:arbitrary_data_coverage} that BCP-VI suffers a constant sub-optimality incurred at optimal locations that are not supported by the behavior policy. By intuition, such a constant sub-optimality occurs due to the off-support actions and vanishes when the behaviour policy covers the trajectories of at least one optimal policy (Assumption \ref{assumption:single_policy_concentration}). 
\label{remark: BCP-VI under arbitrary data coverage}
\end{remark}
\subsection{Instance-dependent bounds}
\label{section:instance_dependent_structures}
% \subsection{Instance-dependent structures}
% \thanh{Need feature coverage or deterministic transitions}

We now show that BCP-VI automatically exploits various types of instance-dependent structures of the underlying MDP to speed up the sub-optimality rate. 

\subsubsection*{Gap-dependent bounds}
The first measure of the hardness of an MDP instance is the minimum positive action gap (Assumption \ref{assumption:margin}) which determines how hard it is to distinguish optimal actions from sub-optimal ones. % The sub-optimality rate depends on how hard it is to distinguish optimal actions from sub-optimal ones in the underlying MDP, which is also known as the minimum positive action gap.
\begin{defn}
For any $(s,a,h) \in \mathcal{S} \times \mathcal{A} \times [H]$, the sub-optimality gap $\Delta_h(s,a)$ is defined as:
% \begin{align*}
    $\Delta_h(s,a) := V^*_h(s) - Q^*_h(s,a)$,
% \end{align*}
and the minimal sub-optimality gap is defined as: 
\begin{align*}
    \Delta_{\min} := \min_{s,a,h} \{\Delta_h(s,a) | \Delta_h(s,a) \neq 0\}. 
\end{align*}
\end{defn}
In this paper, we assume that the minimal sub-optimality gap is strictly positive, which is a common assumption for gap-dependent analysis \citep{simchowitz2019non,yang2021q,he2021logarithmic}. 

\begin{assumption}[Minimum positive sub-optimality gap]
 $\Delta_{\min} > 0$. 
% $\Delta_{\min} := \min_{s,a,h} \{\Delta_h(s,a) | \Delta_h(s,a) > 0\}$ exists, where $\Delta_h(s,a) := V^*_h(s) - Q^*_h(s,a)$. 
% We further assume that optimal actions are unique, i.e. $\forall (h, s_h) \in [H] \times \mathcal{C}^*_h, |\{a_h: Q^*_h(s_h, a_h) = V^*_h(s_h)\}| = 1$. 
% and further assume that the optimal action is unique, i.e. $|\argmax_{a \in \mathcal{A}} Q^*_h(s,a)| = 1, \forall (s,h)$. 
\label{assumption:margin}
\end{assumption}
 We now present the sub-optimality bound under the gap information. 
\begin{theorem}[$\frac{\log K}{K}$ sub-optimality bound]
Under Assumptions \ref{assumption:single_policy_concentration}-\ref{assumption:lower bound density}-\ref{assumption:margin}, with probability at least $1 - (1 + 3  \log_2(H/\Delta_{\min})) \delta$, 
\begin{align*}
    \subopt(\hat{\pi}^{mix})  &\lesssim  2 \frac{ d^3 H^5 \kappa_*^{3}}{\Delta_{\min} \cdot K} \log^3(dKH/\delta) 
    + \frac{16 H \kappa_*}{3 K} \log \log_2(K H \kappa_* / \delta) + \frac{2}{K},
\end{align*}
where $\kappa_* := \max_{h \in [H]} \kappa_h$. 
\label{theorem:logarithimic_regret}
\end{theorem}
\begin{remark}
If we set the $\delta$ in Theorem \ref{theorem:logarithimic_regret} as $\delta = \Omega(1/K)$, then for the expected sub-optimality bound, we have: 
\begin{align*}
    \mathbb{E} \left[ \subopt(\hat{\pi}^{mix}) \right] = \tilde{\mathcal{O}} \left( \frac{ d^3 H^5 \kappa_*^{3}}{\Delta_{\min} \cdot K} \right). 
\end{align*}
\end{remark}

The sub-optimality bound in Theorem \ref{theorem:logarithimic_regret} depends on $\Delta_{\min}$ inversely. It is independent of the state space $\mathcal{S}$, action space $\mathcal{A}$, and is logarithmic in the number of episodes $K$. This suggests that our offline algorithm is sample-efficient for MDPs with large state and action spaces. To our knowledge, this is the first theoretical result that can leverage the gap information $\Delta_{\min}$ to accelerate to a $\mathcal{O}\left( \frac{\log K}{K} \right)$ bound for offline RL with linear function approximation, partial data coverage and adaptive data. 

We now provide the information-theoretic lower bound of learning offline linear MDPs under Assumptions \ref{assumption:single_policy_concentration}-\ref{assumption:lower bound density}-\ref{assumption:margin}. 

\begin{theorem}
Fix any $H \geq 2$. For any algorithm $\textrm{Algo}(\mathcal{D})$, and any concentrability coefficients $\{\kappa_h \}_{h \geq 1}$ such that $\kappa_h \geq 2$, there exists a linear MDP $\mathcal{M} = (\mathcal{S}, \mathcal{A}, H, \mathbb{P}, r, d_0)$ with a positive minimum sub-optimality gap $\Delta_{\min} > 0$ and dataset $\mathcal{D} = \{(s_h^t, a_h^t, r_h^t)\}_{h \in [H]}^{t \in [K]} \sim \mathcal{P}(\cdot |\mathcal{M}, \mu)$ where $\sup_{h, s_h, a_h }\frac{d^{\mathcal{M},*}_h(s_h,a_h)}{d^{\mathcal{M}, \mu}_h(s_h, a_h)} \leq \kappa_h, \forall h \in [H]$ such that: 
\begin{align*}
    \mathbb{E}_{\mathcal{D} \sim \mathcal{M}} \left[\subopt(\textrm{Algo}(\mathcal{D}); \mathcal{M}) \right] = \Omega \left( \frac{\kappa_{\min} H^2}{ K \Delta_{\min}} \right), 
\end{align*}
where  $\kappa_{\min} = \min\{\kappa_h: h \in [H]\}$. 
% where $\kappa_{\max} = \max\{\kappa_h: h \in [H]\}$. 
\label{theorem:lower_bound}
\end{theorem}

% The detailed proof of Theorem \ref{theorem:lower_bound} is deferred to Appendix Section \ref{section:proof_lowerbound}.
Theorem \ref{theorem:lower_bound} implies that any offline algorithm suffers the expected sub-optimality of $\Omega \left( \frac{\kappa_{\min} H^2}{ K \Delta_{\min}} \right)$ under a certain linear MDP instance and behaviour policy that satisfy the minimum positive action gap and the single concentrability. Thus, the result suggests that our algorithm is optimal in terms of $K$ and $\Delta_{\min}$ up to log factors.

% The detailed proof of Theorem \ref{theorem:logarithimic_regret} is deferred to Appendix Section \ref{section:logarithimic_regret_proof}. Theorem \ref{theorem:logarithimic_regret} implies that $\hat{\pi}_{unif}$ obtains the $\textrm{polylog}K$-type sub-optimality bound, i.e. $\mathcal{O}\left( \frac{\textrm{polylog}K}{K} \right)$, where $\textrm{polylog}K$ denotes a polynomial of $\log K$. As $\hat{\pi}_{unif}$ is the average over the ensemble $\{ \hat{\pi}^k \}_{ k \in [K]}$, there is some $k \in [K]$ such that $\hat{\pi}^k$ also enjoys the same sub-optimality rate as $\hat{\pi}_{unif}$. 

% Note that Theorem \ref{theorem:logarithimic_regret} does not require Assumption \ref{assumption:unique_optimal}. When the optimal uniqueness is met, the error rates of offline RL with pessimism even get sped up further, as characterized in Theorem \ref{theorem:constant_regret}. \\

% When we additionally have unique optimal actions at each stage, offline RL speed up to a fast rate of $1/K$ and even obtains zero sub-optimality. 
% Now we are ready to state the main result in this section. 

% \thanh{the polylog result seems to be harder. Can we prove that the support of $d^{\hat{\pi}_{PEVI}}$ is within that of $d^{*}$?}

% needs to cover the learnable region. 
\subsubsection*{Zero sub-optimality}
We introduce extra assumptions on the linear mapping which our algorithm can exploit further to accelerate the rate. For that, we assume the unique optimality and the spanning feature property. Denote $\textrm{span}(\mathcal{X})$ the vector space spanned by all linear combination of elements in $\mathcal{X}$. 
\begin{assumption}[Unique Optimality and Spanning features] We assume that
\begin{enumerate}
    \item (Unique Optimality - UO): The optimal actions are unique, i.e. 
    \begin{align*}
        |\textrm{supp}(\hat{\pi}^*_h(\cdot |s_h))| = 1, \forall (h, s_h) \in [H] \times \mathcal{S}_h^*.
    \end{align*}
    
    \item (Spanning Features - SF): Let $\phi^*_h(s) := \phi_h(s, \pi^*_h(s))$. For any $h \in [H]$, 
\begin{align*}
   \textrm{span}\{\phi^*_h(s_h): \forall s_h \in \mathcal{S}_h^{\mu}\} \subseteq \textrm{span}\{\phi^*_h(s_h): \forall  s_h \in \mathcal{S}_h^{*}\}.
\end{align*}
\end{enumerate}
\label{assumption:spaning_feature}
\end{assumption}
Intuitively, the features of optimal actions in states reachable by an optimal policy provide all information about those in states reachable by the behaviour policy $\mu$. Note that Assumption \ref{assumption:spaning_feature}.2 is much milder than the uniform feature coverage assumption as it does not impose any constraint on the linear features with respect to the offline policy and does not require $\textrm{span}\{\phi^*_h(s): \forall s \in \mathcal{S}, d^*_h(s) > 0\}$ to span the entire $\mathbb{R}^d$. In online regime, a similar assumption called ``universally spanning optimal features'' is used to obtain constant regrets \citep{papini2021reinforcement}. However, their assumption is strictly stronger than ours as they require $ \textrm{span}\{\phi^*_h(s): \forall s \in \mathcal{S}, d^*_h(s) > 0\}$ to span all the features of all actions and states reachable by \emph{any} policy.  Assumption \ref{assumption:spaning_feature}.2 instead requires such condition only over optimal actions and states reachable by the behavior policy. 
% Note that Assumption \ref{assumption:spaning_feature}.2 also does not require the linear feature under $d^*_h$ to span the full space $\mathbb{R}^d$. 

Let $\lambda_h^+$ be the smallest positive eigenvalue of $\Sigma_h^* := \mathbb{E}_{(s_h, a_h) \sim d^{\pi^*}_h} [ \phi_h(s_h, a_h) \phi_h(s_h, a_h)^T ],$ let $\kappa_{1:h} := \prod_{i=1}^h \kappa_i$, and define:
\vspace*{-5pt}
\begin{align}
    k^* = \max_{h \in [H]} \bar{k}_h \lor \tilde{k}_h,
    \label{equation: k threshold}
\end{align} where 
% \begin{align*}
    $\bar{k}_h := \tilde{\Omega} \left( \frac{d^6 H^{10} \kappa_*^6}{\Delta_{\min}^4 (\lambda_h^+)^2} + \frac{ \kappa_{1:h}}{\lambda_h^+} \right) \land \tilde{\Omega} \left( \frac{\kappa_{1:h}^2 \kappa_*^2 H^4 d^3}{ (\lambda_h^+)^2} \right)$, 
    $\tilde{k}_h := \tilde{\Omega} \left( \frac{d^2 H^4 \kappa_{1:h}}{\Delta_{\min}^2 (\lambda_h^+)^3} \right), \forall h$. 
% \end{align*}

\begin{theorem}
Under Assumption \ref{assumption:single_policy_concentration}-\ref{assumption:lower bound density}-\ref{assumption:margin}-\ref{assumption:spaning_feature}, then with probability at least $1 - 4 \delta$, we have:
% \begin{align*}
    $\subopt(\hat{\pi}^k) = 0, \forall k \geq k^*$,
% \end{align*}
where $k^*$ is defined in Eq. (\ref{equation: k threshold}).
% $1 - (3 + 2 \log_2(H/\Delta_{\min}))\delta$, 
\label{theorem:constant_regret}
\end{theorem}
\begin{remark}
The thresholding value $k^*$ defined in Eq. (\ref{equation: k threshold}) is independent of $K$, and it scales with the inverse of $\Delta_{\min}$ and the distributional shift measures $\kappa_h$. 
% Note that it is independent of the number of episodes $K$.
\end{remark}
Theorem \ref{theorem:constant_regret} suggests that when the linear feature at the optimal actions are sufficiently informative and when the number of episodes is sufficiently large exceeding a instance-dependent threshold specified by $k^*$, the policies $\hat{\pi}^k$ in the ensemble precisely match the (unique) optimal policy with high probability. 

\section{Proof Overview}
In the following, we provide a brief overview of the key proof ideas. The detailed proofs are deferred to the appendix. 
\paragraph{For Theorem \ref{theorem:sublinear_subopt}.} With the extended value difference and the constrained optimization in Line \ref{bpvi:greedy} of Algorithm \ref{alg:bpvi}, we can convert bounding $\subopt(\hat{\pi}^k)$ to bounding $2 \mathbb{E}_{\pi^*} [ \sum_{h=1}^H b_h^k(s_h, a_h) ]$. We then use the marginalized importance sampling to convert  $2 \mathbb{E}_{\pi^*} [ \sum_{h=1}^H b_h^k(s_h, a_h) ]$ to the dominant term $ \beta({\delta}) \sum_{h=1}^H   \frac{d^*_h(s_h^k,a_h^k)}{d^{\mu}_h(s_h^k,a_h^k)} \|  \phi_h(s^k_h, a^k_h)  \|_{(\Sigma^k_h)^{-1}}$. For $\hat{\pi}^{last}$, the key observation is that  $\Sigma^k_h \preceq \Sigma^{K+1}_h$, thus
\begin{align*}
    2 \mathbb{E}_{\pi^*} [ \sum_{h=1}^H \|  \phi_h(s_h, a_h)  \|_{(\Sigma^{K+1}_h)^{-1}} ]\leq \frac{2}{K} \displaystyle \sum_{k=1}^K  \mathbb{E}_{\pi^*} [ \sum_{h=1}^H \|  \phi_h(s_h, a_h)  \|_{(\Sigma^{k}_h)^{-1}} ].
\end{align*}

% \vspace*{-5pt}
\paragraph{For Theorem \ref{theorem:logarithimic_regret}.} We convert bounding $\subopt(\hat{\pi}^{mix})$ to bounding its empirical quantity $\frac{1}{K} \sum_{k=1}^K \subopt(\hat{\pi}^k; s_1^k)$ plus a generalization term. Using the original online-to-batch argument \citep{DBLP:journals/tit/Cesa-BianchiCG04} only gives a $\frac{1}{\sqrt{K}}$ generalization error which prevents us from obtaining $\tiO(\frac{1}{K})$ bound. Instead, we propose an improved online-to-batch argument (Lemma \ref{lemma:improved_online_to_batch}) with $\tiO(\frac{1}{K})$ generalization error which is of independent interest. Then,  $\subopt(\hat{\pi}^k; s_1)$ is expressed through decomposition $\subopt(\hat{\pi}^k; s_1) = \mathbb{E}_{\hat{\pi}^k} [ \sum_{h=1}^H \Delta_h(s_h, a_h) | \mathcal{F}_{k-1}, s_1 ]$ (Lemma \ref{lemma: suboptimality as gaps}). To handle the gap terms, the key observation is that $\hat{\pi}^k$ belongs to the $\mu$-supported policy class (Lemma \ref{lemma:concentrability for constrained policy class}), thus the concentrability coefficients (Assumption \ref{assumption:lower bound density}) apply and so does the marginalized importance sampling. The next step is to count the number of times the empirical gaps exceed a certain value, $\sum_{k=1}^K \mathbbm{1}\{ \Delta_h(s_h^k, \hat{\pi}^k_h(s_h^k)) \geq \Delta \} \lesssim  \frac{ d^3 H^2 \iota^{-2}}{\Delta^2} \log^3(dKH/\delta)$ (Lemma \ref{lemma:positive_subopt_count}).  
% \vspace*{-5pt}
\paragraph{For Theorem \ref{theorem:constant_regret}.} A key observation is that $\lambda_{\min}(\Sigma^k_h) \gtrsim k \lambda_h^+$ (Lemma \ref{lemma:bound_Sigma_by_Sigma_star}) where $\lambda_h^+$ is the minimum positive eigenvalue of $\Sigma_h^*$. Thus, for any $v \in \textrm{span}(\{\phi^*_h(s)| s \in \mathcal{S}_h^* \})$, $\| v \|_{(\Sigma^k_h)^{-1}} \leq \mathcal{O}(1/\sqrt{k})$ (Lemma \ref{lemma:bound_feature_norm_by_smallest_positive_eigen_of_Sigma_star}). Under Assumption \ref{assumption:spaning_feature}, $\forall s_h \in \mathcal{S}^{\mu}_h$,
$\Delta_h(s_h, \hat{\pi}^k_h(s_h)) \leq 2 \beta_k \mathbb{E}_{\pi^*} [ \sum_{h'=h}^H \| \phi_h(s_h, a_h) \|_{(\Sigma_{h'}^k)^{-1}} | \mathcal{F}_{k-1}, s_h ]$ $= \mathcal{O}(\frac{1}{\sqrt{k}}) < \Delta_{\min}$, for sufficiently large $k$.  
\paragraph{For Theorem \ref{theorem:lower_bound}.} We reduce the lower bound construction to the statistical testing using the Le Cam method, and construct a hard MDP instance based on the construction of  \citet{jin2021pessimism} with a careful design of the behavior policy to incorporate the OPC coefficients and the gap information $\Delta_{\min}$.

\vspace*{-5pt}
\section{Discussion}
We study offline RL with linear function approximation and contribute the first instance-dependent bounds of $\tilde{\mathcal{O}}\left( \frac{ d^3 H^2 \kappa^{3}}{\Delta_{\min} \cdot K} \right)$ additionally using the gap information and of zero sub-optimality bound using the spanning features in the linear space. Notably, these bounds hold with partial data coverage and adaptive data. We also derive the minimax bounds with the optimal-policy concentrability coefficient (OPC). Our information-theoretic lower bounds suggest that our upper bounds are optimal in terms of $K$ and $\Delta_{\min}$ (up to log factors). However, there is still a gap between upper bounds and lower bounds in terms $\kappa$, in both instance-dependent and instance-agnostic settings with OPC and adaptive data. It remains an open problem to close such a gap.

\section*{Acknowledgements}
This research was supported, in part, by  DARPA~GARD award HR00112020004, NSF CAREER award IIS-1943251, an award from the Institute of Assured Autonomy, and Spring 2022 workshop on ``Learning and Games'' at the Simons Institute for the Theory of Computing. MY was partially supported by NSF Award \#2007117 and \#2003257. SV is the recipient of an ARC Australian Laureate Fellowship (FL170100006). We thank our anonymous reviewers at AAAI'23 for the constructive comments.

%%%%%%%%%%%%%%%%%%%%%%%%%%%%%%%%%%%
%%%%%% SUPPLEMENT (OPTIONAL) %%%%%%
%%%%%%%%%%%%%%%%%%%%%%%%%%%%%%%%%%%
\bibliographystyle{plainnat}
\bibliography{main}

\clearpage
\appendix

% \thispagestyle{empty}

% \includeonly{notation}
% For one-column format, uncomment the following:
\onecolumn 
% For two-column format, uncomment the following:
%\twocolumn[ \makesupplementtitle ]

\setcounter{secnumdepth}{3}
\setcounter{section}{0}
\setcounter{subsection}{0}
\setcounter{subsubsection}{0}

\section{Notations}
\begin{longtable}{|l|l|}
% PEVI & Pessimistic Value Iteration \\ 
% BPVI & Bootstrapped Pessimistic Value Iteration \\ 
% \hline 
\hline
Notations & Meaning \\ 
\hline 
\hline 
$\mathcal{M}$ & MDP instance, $\mathcal{M} = (\mathcal{S}, \mathcal{A}, \mathbb{P},r, H, d_1)$ \\
$\mathcal{S}$ & The arbitrary state space in $\mathbb{R}^d$ \\
$\mathcal{A}$ & The arbitrary action space \\
$H$ & Episode length \\
$K$ & Number of episodes \\ 
$\mathbb{P}_h(\cdot|s_h, a_h)$ & Next state probability distribution \\
$p_h(\cdot|s_h, a_h)$ & Next state density function (with respect to the Lebesgue measure)\\
$r_h(s_h, a_h)$ & Reward function in $[0,1]$ \\
$P_hV$ & $(P_hV)(s,a) = \mathbb{E}_{s' \sim \mathbb{P}_h(\cdot|s,a)} \left[ V(s')\right]$\\
$T_hV$ & $r_h + P_hV$ \\ 
$\pi^* = \{\pi^*_h\}_{h \in [H]}$ & Optimal policy \\
$Q^{\pi} = \{Q^{\pi}_h\}_{h \in [H]}$ & Action-value functions under policy $\pi$\\
$V^{\pi} = \{V^{\pi}_h\}_{h \in [H]}$ & Value functions under policy $\pi$\\
$d^{\pi}_h$ & Marginal state-visitation density (with respect to the Lebesgue measure) \\
$d_1$ & Initial state density \\
$d^*_h$ & $d^{\pi^*}_h$ \\
$\phi_h(s_h, a_h)$ & Feature at step $h$ \\
$\phi^*_h(s_h)$ &  $\phi_h(s_h, a_h^*)$ where $a_h^* \sim \pi^*_h(\cdot | s_h)$ (thus $\phi^*_h(s_h)$ is random w.r.t. $\pi^*_h$)\\
$\Sigma_h^*$ & $\mathbb{E}_{s_h \sim d^*_h} \left[ \phi^*_h(s_h) \phi^*_h(s_h)^T\right]$\\
$\Delta_h(s_h,a_h)$ & $V^*_h(s_h) - Q^*(s_h,a_h)$\\
$\Delta_{ \min}$ & $\min_{h,s_h, a_h} \Delta_h(s_h, a_h)$ \\
$\mu$ & Behavior policy \\
$\kappa_h$ & $\max_{s,a} \frac{d^*_h(s,a)}{d^{\mu}_h(s,a)}$ \\
$\kappa$ & $\sum_{h=1}^H \kappa_h$ \\ 
$\kappa_{1:h}$ & $\prod_{i=1}^h \kappa_h$ \\
$\tau_h$ & $\{(s_i,a_i,r_i)\}_{i \in [h]}$ \\
$\lambda$ & Regularization parameter \\
$\mathcal{D}$ & The offline data, $\{ (s^k_h, a^k_h, r^k_h) \}^{k \in [K]}_{h \in [H]}$ \\
$\mathcal{F}_k$ & $\sigma(\{(s^t_h, a^t_h, r^t_h)\}_{h \in [H]}^{t \in [k]})$ where $\sigma(\cdot)$ is the $\sigma$-algebra generated by $(\cdot)$\\
$\subopt(\pi; s)$ & $V^*_1(s) - V^{\pi}(s)$ \\
$\subopt(\pi)$ & $\mathbb{E}_{s \sim d_1}\left[ \subopt(\pi; s) \right]$ \\
$\hat{\pi}^{mix}_{ h}(a_h|s_h)$ & $\frac{1}{K} \sum_{k=1}^K \hat{\pi}^k_h(a_h | s_h)$ \\ 
$\hat{\pi}^{last}_{h}(a_h|s_h)$ & $\hat{\pi}^{K+1}$ \\ 
$\delta$ & Failure probability \\
$\beta_k({\delta})$ & $C \cdot dH \log(dH k/\delta)$\\
% $\sqrt{K}$-type bound & $\mathbb{E}\left[ \subopt(\hat{\pi}) | \mathcal{D} \right] = \tilde{\mathcal{O}}(1/\sqrt{K})$ \\
% $\textrm{poly}(\log K)$-type bound & $\mathbb{E}\left[ \subopt(\hat{\pi}) | \mathcal{D} \right] = {\mathcal{O}}(\textrm{poly}(\log K)/K)$ \\
% $\mathcal{O}(1)$-type bound & $\mathbb{E}\left[ \subopt(\hat{\pi}) | \mathcal{D} \right] = \mathcal{O}(1)/K$ \\
$\|A\|_2$ & The spectral norm of matrix $A$, i.e. $\lambda_{\max}(A)$ \\
$A \succeq B$ & $A - B$ is positive semi-definite (p.s.d.) \\
$\textrm{supp}(p)$ & The support set of density $p$, i.e. $\{s: p(s) > 0\}$ \\ 
$\| v \|_2$ & $\sqrt{\sum_{i=1}^d v_i^2} $ \\ 
$\| v \|_{\infty} $ & $\max_{i \in [d]} v_i$ \\
WLOG & Without loss of generality \\
$\textrm{poly} \log K$ & A polynomial of $\log K$ \\
$\mathcal{O}(\cdot)$ & Big-Oh notation \\ 
$\tilde{\mathcal{O}}(\cdot)$ & Big-Oh notation with hidden log factors \\ 
$\Omega(\cdot)$ & Big-Omega notation \\
$\tilde{\Omega}(\cdot)$ & Big-Omega notation with hidden log factors \\

\hline
% \hline
% $\|v \|_{p,2}$ & $\mathbb{E}_{u \sim p} \langle u ,v \rangle$
\label{table: notations}
\end{longtable}

In the following, we provide complete proofs to all the results (including our remarks and the general case with an arbitrary data coverage) in the main paper.

\section{Numerical Simulation}
\label{section:simulation}
In this appendix, we provide the details for the numerical simulation for Figure \ref{fig:fast_rates_sim} in the main paper. 
% \thanh{Construct data explicitly from $\kappa$?}

\paragraph{Linear MDP construction.}
We construct a simple linear MDP following \citep{yinnear,min2021variance}. We consider an MDP instance with $\mathcal{S} = \{0,1\}$, $\mathcal{A} = \{0,1, \cdots, 99\}$, and the feature dimension $d = 10$. Each action $a \in [99]$ is represented by its binary encoding vector $u_a \in \mathbb{R}^8$ with entry being either $-1$ or $1$. We define 
\begin{align*}
    \delta(s,a) = 
    \begin{cases}
    1 & \text{ if } \mathbbm{1}\{s = 0\} = \mathbbm{1}\{a = 0\}, \\
    0 & \text{ otherwise}. 
    \end{cases}
\end{align*}
\begin{itemize}
    \item The feature mapping $\phi(s,a)$ is given by 
    \begin{align*}
        \phi(s,a) = [u_a^T, \delta(s,a), 1 - \delta(s,a)]^T \in \mathbb{R}^{10}.    
    \end{align*}
    
    \item The true measure $\nu_h(s)$ is given by 
    \begin{align*}
        \nu_h(s) = [0,\cdots, 0, (1 - s) \oplus \alpha_h, s \oplus \alpha_h]
    \end{align*}
    where $\{\alpha_h\}_{h \in [H]} \in \{0,1\}^{H}$ and $\oplus$ is the XOR operator. We define 
    \begin{align*}
        \theta_h = [0, \cdots, 0, r, 1 - r]^T \in \mathbb{R}^{10},
    \end{align*}
    where $r = 0.99$. Recall that the transition follows $P_h(s'|s,a) = \langle \phi(s,a), \nu_h(s') \rangle$ and the mean reward $r_h(s,a) = \langle \phi(s,a), \theta_h \rangle$. 
    
    %   \item Behavior policy: At state $s = 0$, choose action $a = 0$ with probability $p$ and the other actions uniformly with probability $(1 - p)/ 99$; at state $s = 1$, never choose action $a = 0$ and choose the other actions uniformly with probability $1/99$. 

\end{itemize}

\paragraph{Behavior policy.} At state $s = 0$, choose action $a = 0$ with probability $p$ and action $a=1$ with probability $(1 - p)$; at state $s = 1$, choose action $a = 0$ with probability $p$ and choose the other actions uniformly with probability $(1-p)/99$. This behavior policy does not uniformly cover all state-space pairs but only need to satisfy Assumption \ref{assumption:single_policy_concentration}. 

\paragraph{Experiment.} We computed $\subopt(\hat{\pi}^K)$ for each $K \in \{1, \ldots, 1000\}$ where $\hat{\pi}^K$ is returned by Algorithm \ref{alg:bpvi}. We tested  for different values of $\beta \in \{0, 0.1, 0.2, 0.5, 1, 2\}$ with different episode lengths $H \in \{20, 30, 50, 80\}$. We run each experiment for $30$ times and plot the mean and standard variance of the sub-optimality in Figure \ref{fig:fast_rates_sim}. We observe that $\beta = 1$ gives the best performance in all cases of $H$. It also confirms the benefit of being properly pessimistic (i.e. $\beta=1$) versus being non-pessimistic (i.e. $\beta = 0$) for offline RL. In the case of $\beta = 1$, we observe both phenomenon in the main paper: fast rate in the first $100$ episodes and zero sub-optimality in the later stage. 

% \begin{todo}
% Demonstrate the fast rate in simulation.
% \end{todo}

\begin{table}[h!]
\centering
\def\arraystretch{1.}%
\resizebox{\textwidth}{!}{
\begin{tabular}{ll}
% \hline
Notations & Meaning \\ 
\hline 
% \hline 
$\mathcal{M}$ & MDP instance, $\mathcal{M} = (\mathcal{S}, \mathcal{A}, \mathbb{P},r, H, d_1)$ \\
$\mathcal{S}$ & The arbitrary state space in $\mathbb{R}^d$ \\
$\mathcal{A}$ & The arbitrary action space \\
$H$ & Episode length \\
$K$ & Number of episodes \\ 
$\mathbb{P}_h(\cdot|s_h, a_h)$ & Next state probability distribution \\
$p_h(\cdot|s_h, a_h)$ & Next state density function (with respect to the Lebesgue measure)\\
$r_h(s_h, a_h)$ & Reward function in $[0,1]$ \\
$P_hV$ & $(P_hV)(s,a) = \mathbb{E}_{s' \sim \mathbb{P}_h(\cdot|s,a)} \left[ V(s')\right]$\\
$T_hV$ & $r_h + P_hV$ \\ 
$\pi^* = \{\pi^*_h\}_{h \in [H]}$ & Optimal policy \\
$Q^{\pi} = \{Q^{\pi}_h\}_{h \in [H]}$ & Action-value functions under policy $\pi$\\
$V^{\pi} = \{V^{\pi}_h\}_{h \in [H]}$ & Value functions under policy $\pi$\\
$d^{\pi}_h$ & Marginal state-visitation density (with respect to the Lebesgue measure) \\
$d_1$ & Initial state density \\
$d^*_h$ & $d^{\pi^*}_h$ \\
$\phi_h(s_h, a_h)$ & Feature at step $h$ \\
$\phi^*_h(s_h)$ &  $\phi_h(s_h, a_h^*)$ where $a_h^* \sim \pi^*_h(\cdot | s_h)$ (thus $\phi^*_h(s_h)$ is random w.r.t. $\pi^*_h$)\\
$\Sigma_h^*$ & $\mathbb{E}_{s_h \sim d^*_h} \left[ \phi^*_h(s_h) \phi^*_h(s_h)^T\right]$\\
$\Delta_h(s_h,a_h)$ & $V^*_h(s_h) - Q^*(s_h,a_h)$\\
$\Delta_{ \min}$ & $\min_{h,s_h, a_h} \Delta_h(s_h, a_h)$ \\
$\mu$ & Behavior policy \\
$\kappa_h$ & $\max_{s,a} \frac{d^*_h(s,a)}{d^{\mu}_h(s,a)}$ \\
$\kappa$ & $\sum_{h=1}^H \kappa_h$ \\ 
$\kappa_{1:h}$ & $\prod_{i=1}^h \kappa_h$ \\
$\tau_h$ & $\{(s_i,a_i,r_i)\}_{i \in [h]}$ \\
$\lambda$ & Regularization parameter \\
$\mathcal{D}$ & The offline data, $\{ (s^k_h, a^k_h, r^k_h) \}^{k \in [K]}_{h \in [H]}$ \\
$\mathcal{F}_k$ & $\sigma(\{(s^t_h, a^t_h, r^t_h)\}_{h \in [H]}^{t \in [k]})$ where $\sigma(\cdot)$ is the $\sigma$-algebra generated by $(\cdot)$\\
$\subopt(\pi; s)$ & $V^*_1(s) - V^{\pi}(s)$ \\
$\subopt(\pi)$ & $\mathbb{E}_{s \sim d_1}\left[ \subopt(\pi; s) \right]$ \\
$\hat{\pi}^{mix}_{ h}(a_h|s_h)$ & $\frac{1}{K} \sum_{k=1}^K \hat{\pi}^k_h(a_h | s_h)$ \\ 
$\hat{\pi}^{last}_{h}(a_h|s_h)$ & $\hat{\pi}^{K+1}$ \\ 
$\delta$ & Failure probability \\
$\beta_k({\delta})$ & $C \cdot dH \log(dH k/\delta)$\\
% $\sqrt{K}$-type bound & $\mathbb{E}\left[ \subopt(\hat{\pi}) | \mathcal{D} \right] = \tilde{\mathcal{O}}(1/\sqrt{K})$ \\
% $\textrm{poly}(\log K)$-type bound & $\mathbb{E}\left[ \subopt(\hat{\pi}) | \mathcal{D} \right] = {\mathcal{O}}(\textrm{poly}(\log K)/K)$ \\
% $\mathcal{O}(1)$-type bound & $\mathbb{E}\left[ \subopt(\hat{\pi}) | \mathcal{D} \right] = \mathcal{O}(1)/K$ \\
$\|A\|_2$ & The spectral norm of matrix $A$, i.e. $\lambda_{\max}(A)$ \\
$A \succeq B$ & $A - B$ is positive semi-definite (p.s.d.) \\
$\textrm{supp}(p)$ & The support set of density $p$, i.e. $\{s: p(s) > 0\}$ \\ 
$\| v \|_2$ & $\sqrt{\sum_{i=1}^d v_i^2} $ \\ 
$\| v \|_{\infty} $ & $\max_{i \in [d]} v_i$ \\
WLOG & Without loss of generality \\
$\textrm{poly} \log K$ & A polynomial of $\log K$ \\
$\mathcal{O}(\cdot)$ & Big-Oh notation \\ 
$\tilde{\mathcal{O}}(\cdot)$ & Big-Oh notation with hidden log factors \\ 
$\Omega(\cdot)$ & Big-Omega notation \\
$\tilde{\Omega}(\cdot)$ & Big-Omega notation with hidden log factors \\
% \hline 
\end{tabular}
}
\caption{Notations}
\label{table: notations}
\end{table}

\section{Proofs for Section \ref{subsection:main_single_concentrability}}
\label{section:instance_agnostic_appendix}

In this section, we provide the detailed proofs for all the results stating in Section \ref{subsection:main_single_concentrability}, including Theorem \ref{theorem:sublinear_subopt}, Corollary \ref{corollary:sublinear_subopt}, and Remark \ref{remark: BCP-VI under arbitrary data coverage}. For convenience, we present all notations in Table \ref{table: notations}. 
% \subsection{Preparations}
% For all the results in this paper, we simply choose $\hat{\pi} = \argmax_{\pi \in \{\hat{\pi}^k\}_{k \in [K]}} v^{\pi}$. \footnote{The results also hold for $\hat{\pi} \sim \textrm{Uniform}(\{\hat{\pi}^k: k \in [K]\})$. We focus on a single policy $\hat{\pi}$ chosen from the ensemble because we observe in practice that such a choice often gives a better performance than the uniform sampling over the ensemble. } 

% We state the key lemmas that lead to Theorem \ref{theorem:sublinear_subopt} in this overview and defer the detailed proofs of the lemmas to the appendix. \\

% \textbf{Step 1: Bounding $\subopt(\hat{\pi}^k; s), \forall (k,s) \in [K] \times \mathcal{S}$}. We first construct pointwise bounds for the sub-optimality of each bootstrapped policies $\hat{\pi}^k$. The techniques for this step are quite standard that use pessimism, linear representation and self-normalized martingale concentrations. 

\subsection{Proof of Theorem \ref{theorem:sublinear_subopt}}
\subsubsection{For the mixture policy}

We first provide the proof for the bound of $\hat{\pi}^{mix}$ in Theorem \ref{theorem:sublinear_subopt}. We decompose the proof into three main steps.  

\paragraph{Step 1: Bounding sub-optimality of each $\hat{\pi}^k$.}

We first construct pointwise bounds for the sub-optimality of each bootstrapped policies $\hat{\pi}^k$ to bound $\subopt(\hat{\pi}^k; s), \forall (k,s) \in [K] \times \mathcal{S}$. The techniques for this step are quite standard that use pessimism, linear representation and self-normalized martingale concentrations. Define the Bellman error
\begin{align*}
    \zeta_i^k(s,a) := (T_i \hat{V}_{i+1}^k)(s,a) - \hat{Q}_i^k(s,a), \forall (i, k, s, a) \in [H] \times [K] \times \mathcal{S} \times \mathcal{A}. 
\end{align*}

% The next lemma shows that it suffices to bound the uncertainty of sample Bellman equations. 
We show that if the uncertainty quantifier function $b^k_h$ bounds the error of the empirical Bellman operator, then the Bellman error is non-negative and is bounded above by a constant factor of the uncertainty quantifier function. 
\begin{lemma}
$\forall (h,k,s,a) \in [H] \times [K] \times \mathcal{S} \times \mathcal{A}$, if $|(T_h \hat{V}_{h+1}^k)(s,a) - (\hat{T}^k_h \hat{V}_{h+1}^k)(s,a) | \leq b^k_h(s,a)$, then $\zeta_h^k(s,a) \leq 2 b^k_h(s,a)$. 
% In addition, for any policy $\pi$, with probability at least $1-\delta$, we have: 
% \begin{align*}
%     \forall (s,h,k) \in \mathcal{S} \times [H] \times [K] , V_h^{\pi}(s) - V_h^{\hat{\pi}^k}(s) \leq 2 \mathbb{E}_{\pi} \left[\sum_{i=h}^H  b_i^k(s_i, a_i) |s_h = s\right]. 
% \end{align*}
\label{lemma:bound_with_uncertainty}
\end{lemma}
\begin{proof}[Proof of Lemma \ref{lemma:bound_with_uncertainty}]
We fix any $(h,k,s,a) \in [H] \times [K] \times \mathcal{S} \times \mathcal{A}$. First, we show that: $\zeta_h^k(s,a) \geq 0 $. Indeed, if $\bar{Q}^k_i(s,a) < 0$, $\hat{Q}_i^k(s,a) = 0$, thus $\zeta_i^k(s,a) = (T_i \hat{V}_{i+1}^k)(s,a) - \hat{Q}_i^k(s,a) = (T_i \hat{V}_{i+1}^k)(s,a) \geq 0$. If $\bar{Q}^k_i(s,a) \geq 0$, we have:
\begin{align*}
    \zeta_i^k(s,a) &= (T_i \hat{V}_{i+1}^k)(s,a) - \hat{Q}_i^k(s,a) \geq (T_i \hat{V}_{i+1}^k)(s,a) - \bar{Q}_i^k(s,a) \\
    &= (T_i \hat{V}_{i+1}^k)(s,a) - (\hat{T}^k_h \hat{V}_{h+1}^k)(s,a) + b^k_t(s,a) \geq 0. 
\end{align*}
Next we show that: $\zeta_h^k(s,a) \leq 2 b^k_h(s,a)$. We have $\bar{Q}_h^k(s,a) =  (\hat{T}^k_h \hat{V}_{h+1}^k)(s,a) - b^k_h(s,a) \leq (\hat{T}^k_h \hat{V}_{h+1}^k)(s,a) \leq H - h +1$. Thus, $\hat{Q}_i^k(s,a) = \max\{\bar{Q}_i^k(s,a), 0\}$. Thus, we have:
\begin{align*}
    \zeta_i^k(s,a) &= (T_i \hat{V}_{i+1}^k)(s,a) - \hat{Q}_i^k(s,a) \leq (T_i \hat{V}_{i+1}^k)(s,a) - \bar{Q}_i^k(s,a) \\
    &= (T_i \hat{V}_{i+1}^k)(s,a) - (\hat{T}^k_h \hat{V}_{h+1}^k)(s,a) + (\hat{T}^k_h \hat{V}_{h+1}^k)(s,a) - \bar{Q}_i^k(s,a) \leq 2 b^k_h(s,a).
\end{align*}
\end{proof}
% \begin{question}
% For small $k$, this lemma can be problematic. How to deal this this?
% \end{question}

We show that the stage-wise sub-optimality is bounded by the sum of the uncertainty quantifier functions along the trajectories of the optimal policy.
\begin{lemma}
Suppose that with probability at least $1-\delta$, we have $|(T_h \hat{V}_{h+1}^k)(s,a) - (\hat{T}^k_h \hat{V}_{h+1}^k)(s,a) | \leq b^k_h(s,a), \forall (h,k,s,a) \in [H] \times [K] \times \mathcal{S} \times \mathcal{A}$. Then, under Assumption \ref{assumption:single_policy_concentration}, w.p.a.l. $1-\delta$, we have: 
\begin{align*}
    \forall (h,s_h,k) \in \mathcal{S} \times [H] \times [K] , V_h^{\pi^*}(s_h) - V_h^{\hat{\pi}^k}(s_h) \leq 2 \mathbb{E}_{\pi^*} \left[\sum_{i=h}^H  b_i^k(s_i, a_i) |s_h \right]. 
\end{align*}
\label{lemma:bound sub-opt with uncertainty}
\end{lemma}
\begin{proof}[Proof of Lemma \ref{lemma:bound sub-opt with uncertainty}]
Consider any $(s_h, a_h) \in \mathcal{S}\mathcal{A}^{\pi^*}_h$ (i.e. $d^*_h(s_h, a_h) > 0$). Recall from Line \ref{bpvi:greedy} of Algorithm \ref{alg:bpvi} that $\hat{\pi}_i^k = \argmax_{\pi_i \in \Pi_i(\mu)} \langle \hat{Q}_i^k, \pi_i \rangle, \forall i$, and from Remark \ref{remark: optimal policy belongs to the constrained class} (under Assumption \ref{assumption:single_policy_concentration}), that $\pi^*_i \in \Pi_i(\mu), \forall i \in [H]$. Thus, we have: 
\begin{align*}
    \langle \hat{Q}_i^k, \hat{\pi}^k_i \rangle \geq  \langle \hat{Q}_i^k, {\pi}^*_i \rangle, \forall i \in [H]. 
\end{align*}

Then, by  the value decomposition lemma (Lemma \ref{lemma:regret_decomposition}), $\forall (h, s_h)$, we have 
\begin{align*}
    V_h^{\pi^*}(s_h) - V_h^{\hat{\pi}^k}(s_h) &\leq \sum_{i=h}^H \mathbb{E}_{\pi^*} \left[ \underbrace{\zeta_i^k(s_i, a_i)}_{\leq 2 b^k_i(s_i, a_i)} |s_h  \right] - \sum_{i=h}^H \mathbb{E}_{\hat{\pi}^k} \left[ \underbrace{\zeta_i^k(s_i, a_i)}_{\geq 0} |s_h  \right] \\
    &+ \sum_{i=h}^H \mathbb{E}_{\pi^*} \left[ \underbrace{\langle \hat{Q}_i^k(s_i, \cdot), \pi_i^*(\cdot|s_i) - \hat{\pi}^k_i(\cdot|s_i) \rangle}_{\leq 0} | s_h \right]
     \nonumber \\
    % &\leq \sum_{i=h}^H \mathbb{E}_{\pi^*} \left[ \zeta_i^k(s_i, a_i) |s_h  \right] - \sum_{i=h}^H \mathbb{E}_{\hat{\pi}} \left[ \zeta_i^k(s_i, a_i) |s_h  \right] \\
    &\leq 2 \mathbb{E}_{\pi^*} \left[\sum_{i=h}^H  b_i^k(s_i, a_i) |s_h \right]
    % \label{eq:value_difference_decompose}
\end{align*}
where the second inequality also follows from Lemma \ref{lemma:bound_with_uncertainty}.
\end{proof}

Now we prove that $b^k_h = \beta_k \cdot \| \phi_h(\cdot, \cdot) \|_{(\Sigma_h^{k})^{-1}}$ is indeed a valid uncertainty quantifer function.  
\begin{lemma}
There exists an absolute constant $c_1 > 0$ such that for any $\delta > 0$, if we choose $\beta_k = \beta_k(\delta) := c_1  \cdot dH \log(dH k/\delta)$ in Algorithm \ref{alg:bpvi}, then with probability at least $1 - \delta$: 
\begin{align*}
    \forall (k,h,s,a) \in [K] \times [H] \times \mathcal{S} \times \mathcal{A}, |(T_h \hat{V}_{h+1}^k)(s,a) - (\hat{T}^k_h \hat{V}_{h+1}^k)(s,a) | \leq \beta_k \cdot \| \phi_h(s,a) \|_{(\Sigma_h^{k})^{-1}}. 
\end{align*}
\label{lemma:uncertainty_quantifier}
\end{lemma}
\begin{proof}
Let $w_h^k$ such that $T_h \hat{V}_{h+1}^k = \langle \phi_h, w_h^k \rangle$ (such a $w_h^k$ exists due to Lemma \ref{lemma:linear_bellman}). Recall that $\hat{T}^k_h \hat{V}_{h+1}^k = \langle \phi_h, \hat{w}_h^k \rangle$. We have 
\begin{align*}
    &(T_h \hat{V}_{h+1}^k)(s,a) - (\hat{T}^k_h \hat{V}_{h+1}^k)(s,a) = \phi_h(s,a)^T (w_h^k - \hat{w}_h^k) \\ 
    &=\phi_h(s,a)^T w_h^k - \phi_h(s,a)^T (\Sigma_h^{k})^{-1} \sum_{t=1}^{k-1} \phi_h(s_h^t, a_h^t)(r_h^t + \hat{V}^k_{h+1}(s^t_{h+1})) \\
    &= \underbrace{\phi_h(s,a)^T w_h^k - \phi_h(s,a)^T (\Sigma_h^{k})^{-1} \sum_{t=1}^{k-1} \phi_h(s_h^t, a_h^t) \cdot (T_h\hat{V}^k_{h+1})(s_h^t, a_h^t)}_{(i)} \\
    &+ \underbrace{\phi_h(s,a)^T (\Sigma_h^{k})^{-1} \sum_{t=1}^{k-1} \phi_h(s_h^t, a_h^t) \cdot \left[ (T_h\hat{V}^k_{h+1})(s_h^t, a_h^t) - r_h^t - \hat{V}^k_{k+1}(s^t_{h+1}) \right]}_{(ii)}. 
\end{align*}

We bound term $(i)$ by 
\begin{align*}
    |(i)| &= \left| \phi_h(s,a)^T w_h^k - \phi_h(s,a)^T (\Sigma_h^{k})^{-1} \sum_{t=1}^{k-1} \phi_h(s_h^t, a_h^t) \phi_h(s_h^t, a_h^t)^T w^k_h \right| \\ &= \left|\lambda \phi_h(s,a)^T (\Sigma_h^{k})^{-1} w^k_h \right| \leq \lambda \cdot \| \phi_h(s,a) \|_{(\Sigma_h^{k})^{-1}} \cdot \| w^k_h \|_{(\Sigma_h^{k})^{-1}} \\
    &\leq \lambda \cdot \| \phi_h(s,a) \|_{(\Sigma_h^{k})^{-1}} \cdot \| w^k_h \|_2 \sqrt{\| (\Sigma_h^{k})^{-1} \|} \\
    &\leq  2  H \sqrt{d\lambda} \cdot \| \phi_h(s,a) \|_{(\Sigma_h^{k})^{-1}}. 
\end{align*}
Let $\eta_h^t = (T_h\hat{V}^t_{h+1})(s_h^t, a_h^t) - r_h^t - \hat{V}^t_{h+1}(s^t_{h+1})$. We have 
\begin{align*}
    |(ii) | = \left| \phi_h(s,a)^T (\Sigma_h^{k})^{-1} \sum_{t=1}^{k-1} \phi_h(s_h^t, a_h^t) \cdot \eta_t \right| \leq \| \phi_h(s,a) \|_{(\Sigma_h^{k})^{-1}} \cdot \underbrace{\| \sum_{t=1}^{k-1} \phi_h(s_h^t, a_h^t) \cdot \eta^t_h  \|_{(\Sigma_h^{k})^{-1}}}_{(iii)}. 
\end{align*}

% Let $\mathcal{F}_h^t = \sigma \left( \{(s_h^i, a_h^i)\}_{i=1}^{\max\{t+1,K\}} \cup  \{(r_h^i, s_{h+1}^i)\}_{i=1}^t \right)$ where $\sigma(\cdot)$ denotes the $\sigma$-algebra generated by a set of random variables. \thanh{use this in the proof for adaptive data.} 

By Lemma \ref{lemma:bound_on_weights_algo}-\ref{lemma:covering}-\ref{lemma:uniform_self_normalized}, for any $\delta > 0$, with probability at least $1-\delta$, we have:
\begin{align*}
    (iii)^2 \leq 4H^2 \left[ \frac{d}{2} \log \left( \frac{k + \lambda}{\lambda} \right) + d\log(1 + 4 H \sqrt{dk/\lambda} / \epsilon) + d^2 \log(1 + 8 \sqrt{d} \beta_k^2 /(\lambda \epsilon^2)) + \log(1/\delta) \right] \\
    + \frac{8k^2 \epsilon^2}{ \lambda}. 
\end{align*}
Choosing $\epsilon = dH/k$ and $\lambda = 1$, combining terms $(i)$, $(ii)$, $(iii)$ and using the union bound complete the proof.
\end{proof}
% It remains to bound $|(T_h \hat{V}_{h+1}^k)(s,a) - (\hat{T}^k_h \hat{V}_{h+1}^k)(s,a) |$. 

Combing Lemma \ref{lemma:bound_with_uncertainty} and Lemma \ref{lemma:uncertainty_quantifier} via the union bound, we have the following main lemma for this step.  
\begin{lemma}
There exists an absolute constant $c_1 > 0$ such that if we choose $\beta_k = \beta_k(\delta) := c_1 \cdot dH \log(dHK /\delta)$ in Algorithm \ref{alg:bpvi}, under Assumption \ref{assumption:single_policy_concentration}, then with probability at least $1 - 2 \delta$, $\forall (s_1, h, k) \in \mathcal{S} \times [H] \times [K]$, we have:
\begin{align*}
     \subopt(\hat{\pi}^k; s_1) \leq 2 \beta_k(\delta) \mathbb{E}_{\pi^*} \left[ \sum_{h=1}^H \| \phi_h(s_h, a_h) \|_{(\Sigma_h^k)^{-1}} | \mathcal{F}_{k-1}, s_1 \right].
\end{align*}
\label{lemma:bound_ensemble_with_bonus}
\end{lemma}
\paragraph{Step 2: Generalization.}
Next, we bound $\subopt(\hat{\pi})$ in terms of its empirical bootstraps $\subopt(\hat{\pi}^k; s_1^k)$ (plus a generalization error). Working with individual observed initial states $s^k_1$ is helpful in constructing a complete trajectory $(s^k_1, a^k_1, \ldots, s^k_H, a^k_H)$ of $\mu$ which can then be connected to the trajectory of an optimal policy $\pi^*$ via the optimal-policy concentrability assumption in the next step. We first state and prove the an online-to-batch argument which improves the generalization error of the original online-to-batch argument \citep{DBLP:journals/tit/Cesa-BianchiCG04} from $\mathcal{O}(\frac{1}{\sqrt{K}})$ to $ \mathcal{O}(\frac{\log \log K}{K})$. This result could be of independent interest.    

\begin{lemma}[Improved online-to-batch argument]
Let $\{X_k\}$ be any real-valued stochastic process adapted to the filtration $\{\mathcal{F}_k\}$, i.e. $X_k$ is $\mathcal{F}_k$-measurable. Suppose that for any $k$, $X_k \in [0,H]$ almost surely for some $H > 0$. For any $K > 0$, with probability at least $1 - \delta$, we have:
\begin{align*}
    \sum_{k=1}^K \mathbb{E} \left[ X_k | \mathcal{F}_{k-1} \right] \leq 2 \sum_{k=1}^K X_k + \frac{16}{3} H \log(\log_2(KH)/\delta) + 2 . 
\end{align*}
\label{lemma:improved_online_to_batch}
\end{lemma}
\begin{proof}[Proof of Lemma \ref{lemma:improved_online_to_batch} ]
Let $Z_k = X_k - \mathbb{E} \left[ X_k | \mathcal{F}_{k-1} \right]$ and $f(K) = \sum_{k=1}^K \mathbb{E} \left[ X_k | \mathcal{F}_{k-1} \right]$. We have $Z_k$ is a real-valued difference martingale with the corresponding filtration $\{\mathcal{F}_k\}$ and that
\begin{align*}
    V: = \sum_{k=1}^K \mathbb{E} \left[ Z_k^2 | \mathcal{F}_{k-1} \right] \leq \sum_{k=1}^K \mathbb{E} \left[ X_k^2 | \mathcal{F}_{k-1} \right] \leq H \sum_{k=1}^K \mathbb{E} \left[ X_k | \mathcal{F}_{k-1} \right]
    = H f(K). 
\end{align*}
Note that $|Z_k| \leq H$  and $f(K) \in [0, KH]$ and let $m = \log_2(KH)$.
Also note that $f(K) = \sum_{k=1}^K X_k - \sum_{k=1}^K Z_k \geq - \sum_{k=1}^K Z_k$. Thus if $\sum_{k=1}^K Z_k \leq -1$, we have $f(K) \geq 1$. For any $t > 0$, leveraging the peeling technique \citep{bartlett2005local}, we have:
\begin{align*}
    &\mathbb{P} \left( \sum_{k=1}^K Z_k \leq -\frac{2 H t}{3} - \sqrt{ 4 H f(K) t } - 1 \right) = \mathbb{P} \left( \sum_{k=1}^K Z_k \leq -\frac{2 H t}{3} - \sqrt{ 4 H f(K) t  } - 1, f(K) \in [1, KH]\right) \\ 
    &\leq \sum_{i=1}^m \mathbb{P} \left( \sum_{k=1}^K Z_k \leq -\frac{ 2 H t}{3} - \sqrt{ 4 H f(K) t } - 1, f(K) \in [2^{i-1}, 2^i)\right) \\ 
    &\leq \sum_{i=1}^m \mathbb{P} \left( \sum_{k=1}^K Z_k \leq -\frac{2 H t}{3} - \sqrt{ 4 H 2^{i-1} t } - 1, V \leq H 2^i, f(K) \in [2^{i-1}, 2^i)\right) \\ 
    &\leq \sum_{i=1}^m \mathbb{P} \left( \sum_{k=1}^K Z_k \leq -\frac{2 H t}{3} - \sqrt{ 2 H 2^i t }, V \leq H 2^i \right) \\ 
    % &\leq \sum_{i=1}^m \mathbb{P} \left( -\sum_{k=1}^K Z_k \geq \frac{2 H t}{3} - \sqrt{H 2^i t }, V \geq H 2^{i-1}, f(K) \in [2^{i-1}, 2^i)\right) \\ 
    &\leq \sum_{i=1}^m e^{-t} = m e^{-t},
\end{align*}
where the first equation is by that $\sum_{k=1}^K Z_k \leq -\frac{2 H t}{3} - \sqrt{ 4 H f(K) t } - 1 \leq -1$ thus $f(K) \geq 1$, the second inequality is by that $V \leq H f(K)$, and the last inequality is by Lemma \ref{lemma:freedman}. Thus, with probability at least $1 - m e^{-t}$, we have:
\begin{align*}
    \sum_{k=1}^K X_k - f(K) = \sum_{k=1}^K Z_k \geq -\frac{2 H t}{3} - \sqrt{ 4 H f(K) t } - 1.
    \label{eq:temp}
\end{align*}
% \begin{align}
%     \frac{2 H t}{3} - \sqrt{ H f(K) t } \geq -\sum_{k=1}^K Z_k = f(K)- \sum_{k=1}^K X_k. 
%     \label{eq:temp_158}
% \end{align}
The above inequality implies that $f(K) \leq 2 \sum_{k=1}^K X_k + 4Ht/3 + 2 + 4 Ht $, due to the simple inequality: if $ x \leq a \sqrt{x} + b$, $x \leq a^2 + 2b$. Then setting $t = \log(m/\delta)$ completes the proof. 
\end{proof}

Now, we state and prove the main lemma for the generalization step. 
\begin{lemma}
With probability at least $1 - \delta$, we have: 
\begin{align*}
    \sum_{k=1}^K \subopt(\hat{\pi}^k)  \leq 2 \sum_{k=1}^K \subopt(\hat{\pi}^k; s_1^k) + \frac{16}{3} H \log(\log_2(KH)/\delta) + 2.
\end{align*}
\label{lemma:expected_to_sample_subopt}
\end{lemma}
\begin{remark}
Lemma \ref{lemma:expected_to_sample_subopt} is similar to \citep[Lemma~A.3]{nguyen2021offline} but here we obtain the $\log \log(K)$ error while the latter only obtains the $\sqrt{K}$ error. The key for this improvement is to use the peeling technique and the variance information via Freedman inequality. 
\end{remark}
% \thanh{Can I use the margin condition to have the estimation error at $K^{-1}$ instead of $K^{-1/2}$?}
\begin{proof}[Proof of Lemma \ref{lemma:expected_to_sample_subopt}]
Let $X_k = \subopt(\hat{\pi}^k; s_1^k)$ and recall that $\mathcal{F}_k = \sigma \left(\{(s^t_h, a^t_h, r^t_h)\}_{h \in [H]}^{t\in [k]} \right)$. As $\hat{\pi}^k$ is $\mathcal{F}_{k-1}$-measurable, and $s_1^k$ is $\mathcal{F}_{k}$-measurable, we have that $X_k$ is $\mathcal{F}_{k}$-measurable and $\mathbb{E} \left[ X_k | \mathcal{F}_{k-1} \right] = \mathbb{E} \left[ \subopt(\hat{\pi}^k; s_1^k) | \mathcal{F}_{k-1} \right] = \subopt(\hat{\pi}^k)$. Note that $X_k \in [0,H]$. Thus, by Lemma \ref{lemma:improved_online_to_batch}, with probability at least $1-\delta$, we have:
\begin{align*}
    \sum_{k=1}^K \mathbb{E} \left[ X_k | \mathcal{F}_{k-1} \right] \leq 2 \sum_{k=1}^K X_k + \frac{16}{3} H \log(\log_2(KH)/\delta) + 2. 
\end{align*}
As $\hat{\pi}$ is uniformly sampled from $\{\hat{\pi}^k\}_{k \in [K]}$ conditioned on $\mathcal{D}$, 
\begin{align*}
    K \cdot \mathbb{E} \left[ \subopt(\hat{\pi}) | \mathcal{D} \right] = \sum_{k=1}^K \subopt(\hat{\pi}^k) =  \sum_{k=1}^K \mathbb{E} \left[ X_k | \mathcal{F}_{k-1} \right]. 
\end{align*}
Thus we can complete the proof.
\end{proof}

\paragraph{Step 3: Marginalized importance sampling.} 
This step is the key in our proof to handle the distributional shift under the OPC. The high-level idea is to use importance sampling for the marginalized visitation density functions. 
\begin{lemma}
Under Assumption \ref{assumption:single_policy_concentration}, for any $h \in [H]$, with probability at least $1 - \delta$, we have: 
% \begin{align*}
%     \sum_{k=1}^K \subopt(\hat{\pi}^k; s_1^k) \leq 2 \kappa \beta_{\delta/2} \left[  \sum_{k=1}^K \sum_{h=1}^H \| \phi_h(s^k_h, a^k_h)  \|_{(\Sigma^k_h)^{-1}} + 2 H \sqrt{\frac{K}{\lambda} \log(4H/\delta)} \right].
% \end{align*}
\begin{align*}
    \sum_{k=1}^K \mathbb{E}_{(s_h,a_h) \sim d^*_h(\cdot, \cdot)} \left[ \| \phi_h(s_h, a_h) \|_{(\Sigma_h^k)^{-1}} \bigg| \mathcal{F}_{k-1}, s_1^k  \right] \leq \sum_{k=1}^K \frac{d^*_h(s_h^k,a_h^k)}{d^{\mu}_h(s_h^k,a_h^k)} \| \phi_h(s^k_h, a^k_h)  \|_{(\Sigma^k_h)^{-1}} \\
    +  \sqrt{\frac{1}{\lambda} \log(1/\delta)} \sqrt{\sum_{k=1}^K \left(\frac{d^*_h(s_h^k,a_h^k)}{d^{\mu}_h(s_h^k,a_h^k)} \right)^2 }. 
\end{align*}
\label{lemma:sum_sample_subopt}
\end{lemma}
\begin{proof}[Proof of Lemma \ref{lemma:sum_sample_subopt}]
Let $Z^k_h := \frac{d^*_h(s_h^k,a_h^k)}{d^{\mu}_h(s_h^k,a_h^k)} \| \phi_h(s^k_h, a^k_h)  \|_{(\Sigma^k_h)^{-1}}$. We have $Z^k_h$ is $\mathcal{F}_k$-measurable, and by Assumption \ref{assumption:single_policy_concentration}, we have, 
\begin{align*}
    |Z^k_h| \leq \frac{d^*_h(s_h^k,a_h^k)}{d^{\mu}_h(s_h^k,a_h^k)} \| \phi_h(s^k_h, a^k_h)  \|_2 \sqrt{\| (\Sigma^k_h)^{-1} \|} \leq 1/\sqrt{\lambda} \frac{d^*_h(s_h,a_h)}{d^{\mu}_h(s_h^k,a_h^k)} < \infty,
\end{align*}
and 
\begin{align*}
    \mathbb{E} \left[ Z^k_h | \mathcal{F}_{k-1}, s^k_1 \right] = \mathbb{E}_{(s_h,a_h) \sim d^{\mu}_h(\cdot, \cdot)} \left[ \frac{d^*_h(s_h,a_h)}{d^{\mu}_h(s_h,a_h)} \| \phi_h(s_h, a_h) \|_{(\Sigma_h^k)^{-1}} \bigg| \mathcal{F}_{k-1}, s_1^k  \right].
\end{align*}
Thus, by Lemma \ref{lemma:azuma}, for any $h \in [H]$, with probability at least $1 - \delta$, we have: 
\begin{align}
    &\sum_{k=1}^K \mathbb{E}_{(s_h,a_h) \sim d^*_h(\cdot, \cdot)} \left[ \| \phi_h(s_h, a_h) \|_{(\Sigma_h^k)^{-1}} \bigg| \mathcal{F}_{k-1}, s_1^k  \right] \nonumber \\ &=\sum_{k=1}^K \mathbb{E}_{(s_h,a_h) \sim d^{\mu}_h(\cdot, \cdot)} \left[ \frac{d^*_h(s_h,a_h)}{d^{\mu}_h(s_h,a_h)} \| \phi_h(s_h, a_h) \|_{(\Sigma_h^k)^{-1}} \bigg| \mathcal{F}_{k-1}, s_1^k  \right] \nonumber \\
    &\leq \sum_{k=1}^K \frac{d^*_h(s_h^k,a_h^k)}{d^{\mu}_h(s_h^k,a_h^k)} \| \phi_h(s^k_h, a^k_h)  \|_{(\Sigma^k_h)^{-1}} +  \sqrt{\frac{1}{\lambda} \log(1/\delta)} \sqrt{\sum_{k=1}^K \left(\frac{d^*_h(s_h^k,a_h^k)}{d^{\mu}_h(s_h^k,a_h^k)} \right)^2 }. 
    % \label{eq:expected_to_sample_cov2}
\end{align}
where the second equation is valid due to Assumption \ref{assumption:single_policy_concentration}. 

\end{proof}

% Now we are ready to prove Theorem \ref{theorem:sublinear_subopt}. 
% \begin{proof}[Proof of Theorem \ref{theorem:sublinear_subopt}]
Theorem \ref{theorem:sublinear_subopt} is the direct combination of Lemma \ref{lemma:bound_ensemble_with_bonus}-\ref{lemma:expected_to_sample_subopt}-\ref{lemma:sum_sample_subopt} via the union bound. 
% \end{proof}

\subsubsection{For the last-iteration policy}
In this part, we provide the proof for the bound of $\hat{\pi}_{PEVI}$ in Theorem \ref{theorem:sublinear_subopt}. The proof is similar to that of $\hat{\pi}_{unif}$ except that we directly reason on the elliptical potential $\beta(\delta)  \sum_{h=1}^H \mathbb{E}_{\pi^*} \left[ \| \phi_h(s_h, a_h) \|_{(\Sigma_h^k)^{-1}} | s_1^k, \mathcal{F}_{k-1} \right]$ rather than the empirical sub-optimality metric $\subopt(\hat{\pi}^k; s_1^k)$. We establish the proof steps in the following. 

\paragraph{Step 1: The sub-optimality bound of the last-iteration policy $\hat{\pi}^{last}$.} The first step directly follows the original proof of PEVI \citep{jin2021pessimism} with a simple modification. In particular, with probability at least $1 - 2 \delta$, we have:
\begin{align*}
    \subopt(\hat{\pi}^{last}) \leq \min \bigg\{H, 2 \beta(\delta) \mathbb{E}_{s_1 \sim d_1} 
    \left[ \sum_{h=1}^H \mathbb{E}_{\pi^*} \left[ \| \phi_h(s_h, a_h) \|_{\Sigma_h^{-1}} | s_1 \right] \right] \bigg\},
\end{align*}
where $\mathbb{E}_{\pi^*}$ is with respect to the trajectory $(s_1,a_1, \ldots, s_H, a_H)$ induced by $\pi^*$. \\

\paragraph{Step 2: Generalization.} The key idea for this step is to bound the expectation $\mathbb{E}_{s_1 \sim d_1} [\cdot | s_1]$ in terms of the empirical quantities $[\cdot | s_1^k]$ (plus a generalization error). The empirical quantities $[\cdot | s_1^k]$ is useful in constructing a complete trajectory $(s^k_1, a^k_1, \ldots, s^k_H, a^k_H)$ of $\mu$ which can then be connected to the trajectory of the optimal policy $\pi^*$ via the density dominance assumption. However, as $\Sigma_h$ depends on $\{s^k_1\}_{k \in [K]}$, the generalization error from the expectation $\mathbb{E}_{s_1 \sim d_1} [\cdot | s_1]$ to the empirical quantities $[\cdot | s_1^k]$ cannot guarantee. \footnote{Or at least we must use the uniform convergence argument to overcome this data-dependent structure which makes up a large and unnecessary generalization error.} Instead we use a simple trick to break such data dependence and form a valid martingale for strong generalization. Let $Z^k_h := 2 \beta(\delta)  \sum_{h=1}^H \mathbb{E}_{\pi^*} \left[ \| \phi_h(s_h, a_h) \|_{(\Sigma_h^k)^{-1}} | s_1^k, \mathcal{F}_{k-1} \right]$. We have: 
\begin{align*}
    &K \cdot \subopt(\hat{\pi}^{last}) \leq \sum_{k=1}^K \min \bigg\{H, 2 \beta(\delta) \mathbb{E}_{s_1 \sim d_1} 
    \left[ \sum_{h=1}^H \mathbb{E}_{\pi^*} \left[ \| \phi_h(s_h, a_h) \|_{\Sigma_h^{-1}} \bigg| s_1, \mathcal{F}_{K} \right] \right] \bigg\} \\ 
    &\leq  \sum_{k=1}^K \min \bigg\{H, 2 \beta(\delta) \mathbb{E}_{s_1 \sim d_1} 
    \left[ \sum_{h=1}^H \mathbb{E}_{\pi^*} \left[ \| \phi_h(s_h, a_h) \|_{(\Sigma_h^k)^{-1}} \bigg| s_1, \mathcal{F}_{k-1} \right] \right] \bigg\} \\ 
    &=  \sum_{k=1}^K \min \bigg\{H, \mathbb{E} \left[Z_k | \mathcal{F}_{k-1} \right]  \bigg\} \\
    &\leq \sum_{k=1}^K \mathbb{E} \left[ \min \{H, Z_k  \} | \mathcal{F}_{k-1} \right]  \\
    &\leq 2 \sum_{k=1}^K \min \{H, Z_k  \} + \frac{16}{3} H \log (\log_2(KH)/\delta) + 2 \\ 
    &=  2 \sum_{k=1} ^K \min \bigg\{H, 2\beta(\delta)  \sum_{h=1}^H \mathbb{E}_{\pi^*} \left[ \| \phi_h(s_h, a_h) \|_{(\Sigma_h^k)^{-1}} | s_1^k, \mathcal{F}_{k-1} \right] \bigg\} + \frac{16}{3} H \log (\log_2(KH)/\delta) + 2,
\end{align*}
% \thanh{This is a nice trick to go from $s_1 \sim d_1$ to $s_1^k$?}
where the second inequality is by $\min \{a,b \} \leq \min 
\{a, c \}$ if $b \leq c$ and $\Sigma_h^{-1} \preceq (\Sigma_h^k)^{-1}, \forall k \in [K + 1]$, the third inequality is by Jensen's inequality (as $f(x) = \min\{H, x\}$ is convex), and the fourth inequality is by Lemma \ref{lemma:improved_online_to_batch} (where the use of $(\Sigma_h^k)^{-1}$ in the place of $\Sigma_h^{-1}$ is crucial to form a valid martingale for applying Lemma \ref{lemma:improved_online_to_batch} and $\min \{H , Z_k\} \leq H$). 

\paragraph{Step 3: Marginalized importance sampling.} This step is the same as the marginalized importance sampling step for $\hat{\pi}^{mix}$.

\subsection{Proof of Corollary \ref{corollary:sublinear_subopt}} 
We give a proof for Corollary \ref{corollary:sublinear_subopt}. Using Theorem \ref{theorem:sublinear_subopt}, it suffices to prove the following lemma. 
\begin{lemma}
Under Assumption \ref{assumption:single_policy_concentration}-\ref{assumption:lower bound density}, for any $h \in [H]$, we have:
\begin{align*}
    \sum_{k=1}^K \frac{d^*_h(s_h^k,a_h^k)}{d^{\mu}_h(s_h^k,a_h^k)} \|  \phi_h(s^k_h, a^k_h)  \|_{(\Sigma^k_h)^{-1}} \leq \kappa_h^{-1} \sqrt{2 K d \log(1 + K/d)}. 
\end{align*}
\label{lemma:bound_shrinking_sum}
\end{lemma}
\begin{proof}[Proof of Lemma \ref{lemma:bound_shrinking_sum}]
We have:
\begin{align*}
    \sum_{k=1}^K \frac{d^*_h(s_h^k,a_h^k)}{d^{\mu}_h(s_h^k,a_h^k)} \|  \phi_h(s^k_h, a^k_h)  \|_{(\Sigma^k_h)^{-1}} 
    &\leq \kappa_h^{-1}
    \sum_{k=1}^K \|  \phi_h(s^k_h, a^k_h)  \|_{(\Sigma^k_h)^{-1}} \leq \kappa_h^{-1} \sqrt{K \sum_{k=1}^K \|  \phi_h(s^k_h, a^k_h)  \|^2_{(\Sigma^k_h)^{-1}}} \\
    &\leq \kappa_h^{-1} \sqrt{2 K \log \frac{\det \Sigma_h^{K+1}}{\det(I)}} \leq \kappa_h^{-1} \sqrt{2 K} \sqrt{d \log(1 + K/d)}. 
\end{align*}
where the first inequality is by Assumption \ref{assumption:single_policy_concentration}, the second inequality is by Cauchy-Schwartz inequality, and the last two inequalities are by \citep[Lemma~11]{NIPS2011_e1d5be1c}.
\end{proof}

 We also provide an information-theoretical lower bound for Assumption \ref{assumption:single_policy_concentration}-\ref{assumption:lower bound density}. 

\begin{theorem}
Fix any $H \geq 2$. For any algorithm $\textrm{Algo}(\mathcal{D})$, and any concentrability coefficients $\{\kappa_h \}_{h \geq 1}$ such that $\kappa_h \geq 2$, there exist a linear MDP $\mathcal{M} = (\mathcal{S}, \mathcal{A}, H, \mathbb{P}, r, d_0)$ and dataset $\mathcal{D} = \{(s_h^t, a_h^t, r_h^t)\}_{h \in [H]}^{t \in [K]} \sim \mathcal{P}(\cdot |\mathcal{M}, \mu)$ where $\sup_{h, s_h, a_h }\frac{d^{\mathcal{M},*}_h(s_h,a_h)}{d^{\mathcal{M}, \mu}_h(s_h, a_h)} \leq \kappa_h, \forall h \in [H]$ such that: 
\begin{align*}
    \mathbb{E}_{\mathcal{D} \sim \mathcal{M}} \left[\subopt(\textrm{Algo}(\mathcal{D}); \mathcal{M}) \right] = \Omega \left( \frac{H \sqrt{\kappa_{\min}} }{ \sqrt{K} } \right), 
\end{align*}
where  $\kappa_{\min} := \min\{\kappa_h: h \in [H]\}$. 
% where $\kappa_{\max} = \max\{\kappa_h: h \in [H]\}$. 
\label{theorem:lower_bound_minimax}
\end{theorem}

\subsection{Proofs of the remarks}
\subsubsection{Learning bounds under arbitrary data coverage}
\label{subsection:arbitrary_data_coverage}
% \textbf{Arbitrary data coverage.} As we do not have control over the behavior policy, it happens in practice that the behavior policy does not fully cover all the trajectories of the optimal policy. In such a case, any algorithm should suffer constant sub-optimality gap which is only resolved with a better data coverage. 
Note that Theorem \ref{theorem:sublinear_subopt} requires the OPC assumption (Assumption \ref{assumption:single_policy_concentration}) to be valid. In practice, as we do not have control over the behavior policy, it happens that the behavior policy does not fully cover all the trajectories of the optimal policy, thus the OPC assumption might fail to hold. This raises the question of how much an offline algorithm can suffer when it learns from the offline data of arbitrary coverage. 

% \subsubsection{Result}

To tackle this issue, we construct a new MDP $\bar{\mathcal{M}}$ under which the trajectories of the optimal policy $\bar{\pi}^*$ (with respect to $\bar{\mathcal{M}}$) are best covered by the behavior policy $\mu$. Then, the sub-optimality gap incurred by the under-coverage data must be the gap between the original optimal policy $\pi^*$ and the data-supported optimal policies $\bar{\pi}^*$. 

\paragraph{Augmented MDP.} For any small positive $\bar{\epsilon} > 0$, we consider the augmented MDP $\bar{\mathcal{M}} = (\mathcal{S} \cup \{\bar{s}_{h+1}\}_{h \in [H]}, \mathcal{A}, \bar{\mathbb{P}}, \bar{r}, H, d_1)$, where for any $h \in [H]$
\begin{align*}
    \bar{\mathbb{P}}_h(\cdot | s_h, a_h ) = 
    \begin{cases}
    \mathbb{P}_h(\cdot | s_h, a_h ) & \text{ if } (s_h, a_h) \in \mathcal{C}^{\mu}_h \\
    \delta_{\bar{s}_{h+1}} & \text{ if } (s_h, a_h) \notin \mathcal{C}^{\mu}_h
    \end{cases}
    , 
    \bar{r}_h(s_h, a_h ) = 
    \begin{cases}
    r_h(s_h, a_h ) & \text{ if } (s_h, a_h) \in \mathcal{C}^{\mu}_h \\
    -\bar{\epsilon}/H & \text{ if } (s_h, a_h) \notin \mathcal{C}^{\mu}_h
    \end{cases}
\end{align*}
Here $\bar{\mathcal{M}}$ extends the original state $\mathcal{S}$ to include arbitrary states $\{\bar{s}_{h+1}\}_{h \in [H]}$ where $\bar{s}_{h+1} \notin \mathcal{S}$. The transition distributions and the reward functions are the same as the original $\mathcal{M}$ except at the infeasible state-action $(s_h, a_h) \notin \mathcal{C}^{\mu}_h$ where the augmented MDP always absorbs into the dummy state $\bar{s}_{h+1}$ and yields small negative reward. 
% \raman{I am a bit unclear about what we are augmenting the state space with.} 
Under $\bar{\mathcal{M}}$, we denote the corresponding marginal state-visitation density by $\bar{d}^{\pi}$ and the corresponding optimal policy  $\bar{\pi}^*$. We abbreviate $\bar{d}^{\bar{\pi}^*} = \bar{d}^*$. Our augmented MDP construction is similar to the construction in \cite{yin2021towards} except that we allow an arbitrary small negative reward $-\bar{\epsilon}/H$ in unsupported state-action pairs. This design guarantees that the optimal policy under $\bar{\mathcal{M}}$ is dominated by $\mu$ (Lemma \ref{lemma:off_support}). The following theorem (Theorem \ref{theorem:generic_bound_any_coverage}) provides a generic (instance-agnostic) bound that works under arbitrary data coverage. 
\begin{theorem}
Under Assumption \ref{assumption:single_policy_concentration}, w.p.a.l. $1 - 4 \delta$ over the randomness of $\mathcal{D}$, we have: 
\begin{align*}
& \subopt(\hat{\pi}^{mix}) \lor 
\subopt(\hat{\pi}^{last}) 
\leq  \textrm{gap}_{\textrm{support}} + \frac{4 \beta({\delta})}{K} \sum_{h=1}^H  \sum_{k=1}^K  \frac{\bar{d}^*_h(s_h^k,a_h^k)}{d^{\mu}_h(s_h^k,a_h^k)} \|  \phi_h(s^k_h, a^k_h)  \|_{(\Sigma^k_h)^{-1}} \\
&+\frac{4 \beta({\delta})}{K} \sum_{h=1}^H \sqrt{\log \left(\frac{H}{\delta} \right) \sum_{k=1}^K \left(\frac{\bar{d}^*_h(s_h^k,a_h^k)}{d^{\mu}_h(s_h^k,a_h^k)} \right)^2 }  
+ \frac{2}{K} +  \frac{16 H}{3 K}\log\left( \frac{\log_2(KH)}{ \delta} \right),
\end{align*}
where $gap_{support} : =\sum_{h=1}^H \int_{(\mathcal{C}_h^{\pi})^c} d^{\pi}_h(s_h, a_h) ds_h a_h + \bar{\epsilon}$ is the sub-optimality gap incurred in locations that are supported by $\pi^*$ but unsupported by $\mu$, and $\bar{d}^*$ is the marginal visitation density of the optimal policy under the augmented MDP $\bar{\mathcal{M}}$. 
\label{theorem:generic_bound_any_coverage}
\end{theorem}

Theorem \ref{theorem:generic_bound_any_coverage} is valid for any data coverage. When the behavior policy $\mu$ does not fully support the trajectories of any optimal policy $\pi^*$, the algorithm must suffer a constant sub-optimality gap $\textrm{gap}_{\textrm{support}}$ which is incurred by the total rewards at the locations supported by the optimal policy but not by the behavior policy. When the OPC assumption (Assumption \ref{assumption:single_policy_concentration}) holds, $\textrm{gap}_{\textrm{support}} = 0$ and $\bar{d}^*_h = d^*_h$ (Lemma \ref{lemma:off_support}), and Theorem \ref{theorem:generic_bound_any_coverage} reduces into Theorem \ref{theorem:sublinear_subopt}. 

% \raman{Wasn't the point of this section to describe/characterize $ \textrm{gap}_{\textrm{support}}$? As a form of sign posting for the reader, say that that is what we look into in the next section.}

% \begin{todo}
% Give a concrete example or a lower bound for this constant sub-optimality gap. 
% \end{todo}

% \begin{question}
% In a linear MDP even when $\mu_h$ does not support $(s_h,a_h)$, we can still infer $\hat{r}_h(s_h, a_h) = \phi_h(s_h, a_h)^T \hat{\theta}_h$. Thus, the constant gap for the under-coverage in a linear MDPs might be more benign than that in a tabular MDP? 
% \end{question}

Before proving Theorem \ref{theorem:generic_bound_any_coverage}, we provide and prove useful properties of the augmented MDP $\bar{\mathcal{M}}$ from the perspective of the original MDP $\mathcal{M}$.
\begin{lemma}
Consider any $(h,s_h, a_h) \in [H] \times \bar{\mathcal{S}} \times \mathcal{A}$, and any policy $\pi$.
\begin{enumerate}[label=(\alph*)]
    \item $\bar{d}^{\mu}_h(s_h) = d^{\mu}_h(s_h), \forall  (h,s_h, a_h) \in [H] \times \bar{\mathcal{S}} \times \mathcal{A}$. 
    \item For any policy $\pi$, if $(s_h, a_h) \notin \mathcal{C}^{\mu}_h$, $\bar{Q}^{\pi}_h(s_h, a_h) = - (H - h +1)$. 
    %     \bar{Q}^{\pi}_h(s_h, a_h) = 
    %     \begin{cases}
    %         Q^{\pi}_h(s_h, a_h) & \text{ if } (s_h, a_h) \in \mathcal{C}^{\mu}_h, \\ 
    %         0 & \text{ otherwise.}
    %     \end{cases}
    % \end{align*}
    
    \item For any $(h,s_h, a_h) \in [H] \times \bar{\mathcal{S}} \times \mathcal{A}$, if $\bar{d}^*_h(s_h, a_h) > 0$, then $\bar{d}^{\mu}_h(s_h, a_h) > 0$. 

    \item Under Assumption \ref{assumption:single_policy_concentration}, $\bar{d}^*_h = d^*_h, \forall h \in [H]$. 
\end{enumerate}
\label{lemma:off_support}
\end{lemma}

\begin{proof}
We provide the proof in the following. 
\paragraph{(a)} We prove $(a)$ by induction. First, we have: 
\begin{align*}
    \forall s_1 \in \bar{\mathcal{S}}, \bar{d}^{\pi}_1(s_1) = \bar{d}_1(s_1) = d_1(s_1) = d^{\pi}_1(s_1). 
\end{align*}
Suppose that for some $h \in [H]$, $\bar{d}^{\pi}_h(s_h) = d^{\pi}_h(s_h), \forall s_h \in \mathcal{C}^{\mu}_h$. Consider any $s_{h+1} \in \bar{\mathcal{S}}$. We have:
\begin{align*}
    \bar{d}^{\mu}_{h+1}(s_{h+1}) &= \int \bar{p}_h(s_{h+1}|s_h, a_h) \bar{d}_h^{\mu}(s_h, a_h) d(s_ha_h) \\
    &= \int_{\mathcal{C}^{\mu}_h} \bar{p}_h(s_{h+1}|s_h, a_h) \bar{d}_h^{\mu}(s_h, a_h) d(s_ha_h) + \int_{\bar{\mathcal{C}}^{\mu}_h} \bar{p}_h(s_{h+1}|s_h, a_h) \bar{d}_h^{\mu}(s_h, a_h) d(s_ha_h) \\
    &= \int_{\mathcal{C}^{\mu}_h} p_h(s_{h+1}|s_h, a_h) d_h^{\mu}(s_h, a_h) d(s_ha_h) + \int_{\bar{\mathcal{C}}^{\mu}_h} \bar{p}_h(s_{h+1}|s_h, a_h) \underbrace{d_h^{\mu}(s_h, a_h)}_{=0} d(s_ha_h) \\
    &= \int_{\mathcal{C}^{\mu}_h} p_h(s_{h+1}|s_h, a_h) d_h^{\mu}(s_h, a_h) d(s_ha_h) + \int_{\bar{\mathcal{C}}^{\mu}_h} p_h(s_{h+1}|s_h, a_h) \underbrace{d_h^{\mu}(s_h, a_h)}_{=0} d(s_ha_h) \\
    &= d_{h+1}^{\mu}(s_{h+1}), 
\end{align*}
where the third equation is due to that $\bar{p}_h(s_{h+1}|s_h, a_h) = p_h(s_{h+1}|s_h, a_h)$ for $(s_h, a_h) \in \mathcal{C}^{\mu}_{h}$, $\bar{d}_h^{\mu}(s_h, a_h) = d_h^{\mu}(s_h, a_h)$ for any $(s_h, a_h)$ (by the induction step), and the fourth equation is by $d^{\mu}_h(s_h,a_h) = 0$ for any $(s_h,a_h) \notin \mathcal{C}^{\mu}_h$ (by definition). 
% \paragraph{hi}
\paragraph{(b)} 
For any $(s_h, a_h) \notin \mathcal{C}^{\mu}_h$, the feasible trajectory must admit the following form: $(s_h, a_h, \bar{s}_{h+1}, a_{h+1}, \ldots, \bar{s}_{H}, a_H, \bar{s}_{H+1})$ which has zero cumulative reward under $\bar{r}$ by definition.
\paragraph{(c)} By $(a)$, we now replace $\bar{d}_h^{\mu}$ by $d^{\mu}_h$ in $(c)$ and prove $(c)$ by induction. At initial state $h = 1$, if $\bar{d}^*_1(s_1, a_1) = d_1(s_1) \bar{\pi}_1^*(a_1 | s_1) > 0$, we have $d_1(s_1) > 0$ and $\bar{\pi}_1^*(a_1 | s_1) > 0$. Thus $a_1$ must be an optimal action given $s_1$ (under $\bar{\mathcal{M}}$). Suppose that $(s_1, a_1) \notin \mathcal{C}^{\mu}_1$. By $(b)$, $\bar{Q}^*_1(s_1, a_1) = -H$. Let any $\tilde{a}_1$ such that $(s_1, \tilde{a}_1) \in \mathcal{C}^{\mu}_1$ (such $\tilde{a}_1$ exists as $s_1 \in \mathcal{C}^{\mu}_1$). Then, we have $\bar{Q}^*(s_1, \tilde{a}_1) = r_1(s_1,  \tilde{a}_1) + \mathbb{E}_{\bar{\pi}^*} \left[ \sum_{i=2}^H \bar{r}_i \bigg| s_1, \tilde{a}_1 \right] \geq -(H-1) > \bar{Q}_1^*(s_1, a_1)$. This contradicts that $a_1$ must be an optimal action given $s_1$ (under $\bar{\mathcal{M}}$). 

Now assume that we have $(c)$ for some $h \geq 1$. We will prove $(c)$ for $h+1$. Indeed, consider any $(s_{h+1}, a_{h+1})$ such that $\bar{d}^*_{h+1}(s_{h+1}, a_{h+1}) > 0$. Then we must have $s_{h+1} \in C^*_{h+1} $ and $a_{h+1}$ is an optimal action given $s_{h+1}$ (under $\bar{\mathcal{M}}$). Since $s_{h+1} \in C^*_{h+1}$, there must be some $(s_h, a_h) \in \mathcal{C}^{*}_h$ such that $\bar{p}_h(s_{h+1} | s_h, a_h) > 0$. By induction, we have $(s_h, a_h) \in \mathcal{C}^{*}_h \subseteq \mathcal{C}^{\mu}_h$. Thus, $0 < \bar{p}_h(s_{h+1} | s_h, a_h) = p_h(s_{h+1} | s_h, a_h)$. Hence, $s_{h+1} \in \mathcal{C}^{\mu}_{h+1}$. Given $s_{h+1} \in \mathcal{C}^{\mu}_{h+1}$ and $a_{h+1}$ is an optimal action given $s_{h+1}$ (under $\bar{\mathcal{M}}$), similar to the base case, we must have $(s_{h+1}, a_{h+1}) \in \mathcal{C}^{\mu}_{h+1}$. 
% \begin{align*}
%     \bar{d}^*_1(s_1, a_1)  = d_1(s_1) \bar{\pi}_1^*(a_1 | s_1) > 0.
% \end{align*}

\paragraph{(d)}
% \begin{align*}
%     \bar{d}^{*}_{h+1}(s_{h+1}) &= \int \bar{p}_h(s_{h+1}|s_h, a_h) \bar{d}_h^{*}(s_h, a_h) d(s_ha_h) \\
%     &= \int_{\mathcal{C}^{\mu}_h} \bar{p}_h(s_{h+1}|s_h, a_h) \bar{d}_h^{*}(s_h, a_h) d(s_ha_h) + \int_{\bar{\mathcal{C}}^{\mu}_h} \bar{p}_h(s_{h+1}|s_h, a_h) \bar{d}_h^{*}(s_h, a_h) d(s_ha_h) \\
%     &= \int_{\mathcal{C}^{\mu}_h} p_h(s_{h+1}|s_h, a_h) d_h^{*}(s_h, a_h) d(s_ha_h) + \int_{\bar{\mathcal{C}}^{\mu}_h} \bar{p}_h(s_{h+1}|s_h, a_h) \underbrace{d_h^{*}(s_h, a_h)}_{=0} d(s_ha_h) \\
%     &= \int_{\mathcal{C}^{\mu}_h} p_h(s_{h+1}|s_h, a_h) d_h^{*}(s_h, a_h) d(s_ha_h) + \int_{\bar{\mathcal{C}}^{\mu}_h} p_h(s_{h+1}|s_h, a_h) \underbrace{d_h^{*}(s_h, a_h)}_{=0} d(s_ha_h) \\
%     &= d_{h+1}^{\mu}(s_{h+1}), 
% \end{align*}
% where the third equation is by induction, and the fourth equation is by $d^*_h(s_h, a_h) = 0$ for $(s_h, a_h) \notin \mathcal{C}^{\mu}_{h}$ (by Part (c)). 
Under Assumption \ref{assumption:single_policy_concentration}, we can prove by induction from $H, H-1, \ldots, 1$ that: 
$\bar{Q}^*_h(s_h,a_h) = Q^*_h(s_h, a_h) $ and $\bar{V}^*_h(s_h) > \bar{Q}^*_h(s_h,\tilde{a_h})$ if $d^*_h(s_h) > 0$ and $\pi^*_h(a_h | s_h) > 0$ and $\pi^*_h(\tilde{a_h} | s_h) = 0$. This then implies $(d)$.

\end{proof}

Now we are ready to prove the result in this subsection. 
% \subsubsection{Proof of Theorem \ref{theorem:generic_bound_any_coverage}}
\begin{proof}[Proof of Theorem \ref{theorem:generic_bound_any_coverage}]
For any policy $\pi$, we have:
\begin{align*}
    \mathbb{E}_{\pi, \mathcal{M}} [r_h] - \mathbb{E}_{\pi, \bar{\mathcal{M}} } [\bar{r}_h]
    &= \int_{\mathcal{C}_h^{\pi}} r_h(s_h, a_h)  d^{\pi}_h(s_h, a_h) ds_h a_h + \int_{(\mathcal{C}_h^{\pi})^c} r_h(s_h, a_h) d^{\pi}_h(s_h, a_h) ds_h a_h \\
    &-\int_{\mathcal{C}_h^{\pi}} \bar{r}_h(s_h, a_h) \bar{d}^{\pi}_h(s_h, a_h) ds_h a_h - \int_{(\mathcal{C}_h^{\pi})^c} \bar{r}_h(s_h, a_h) \bar{d}^{\pi}_h(s_h, a_h) ds_h a_h \\ 
    &= \int_{(\mathcal{C}_h^{\pi})^c} r_h(s_h, a_h) d^{\pi}_h(s_h, a_h) ds_h a_h + \epsilon_0 \in (0, \int_{(\mathcal{C}_h^{\pi})^c} d^{\pi}_h(s_h, a_h) ds_h a_h + \bar{\epsilon} / H]
\end{align*}
where $(\mathcal{C}^{\pi}_h)^c$ denotes the complement of $\mathcal{C}^{\pi}_h$. Thus, we have:
\begin{align}
    v^{\mathcal{M}, \pi} - v^{\bar{\mathcal{M}}, \pi} = \sum_{h=1}^H \mathbb{E}_{\pi, \mathcal{M}}[r_h] - \mathbb{E}_{\pi, \bar{\mathcal{M}}}[\bar{r}_h] \in (0, \sum_{h=1}^H \int_{(\mathcal{C}_h^{\pi})^c} d^{\pi}_h(s_h, a_h) ds_h a_h + \bar{\epsilon} ].
    \label{eq:tempppp}
\end{align}
Define $gap_{support} : =\sum_{h=1}^H \int_{(\mathcal{C}_h^{\pi})^c} d^{\pi}_h(s_h, a_h) ds_h a_h + \bar{\epsilon}$, we have
\begin{align*}
    &v^{\mathcal{M}, \pi^*} - v^{{\mathcal{M}}, \pi} = (v^{\bar{\mathcal{M}}, \pi^*} - v^{{\mathcal{M}}, \pi}) + (v^{\mathcal{M}, \pi^*} -  v^{\bar{\mathcal{M}}, \pi^*}) \\ 
    &\leq (v^{\bar{\mathcal{M}}, \pi^*} - v^{\bar{\mathcal{M}}, \pi}) + gap_{support} \\ 
    &\leq (v^{\bar{\mathcal{M}}, \bar{\pi}^*} - v^{\bar{\mathcal{M}}, \pi}) + gap_{support} \\ 
    &\leq \frac{4 \beta({\delta})}{K} \sum_{h=1}^H  \sum_{k=1}^K  \frac{\bar{d}^*_h(s_h^k,a_h^k)}{\bar{d}^{\mu}_h(s_h^k,a_h^k)} \|  \phi_h(s^k_h, a^k_h)  \|_{(\Sigma^k_h)^{-1}} +\frac{4 \beta({\delta})}{K} \sum_{h=1}^H \sqrt{\log \left(\frac{H}{\delta} \right) \sum_{k=1}^K \left(\frac{\bar{d}^*_h(s_h^k,a_h^k)}{\bar{d}^{\mu}_h(s_h^k,a_h^k)} \right)^2 }  \\
    &+ \frac{2}{K} +  \frac{16 H}{3 K}\log\left( \frac{\log_2(KH)}{ \delta} \right) + gap_{support} \\
    &= \frac{4 \beta({\delta})}{K} \sum_{h=1}^H  \sum_{k=1}^K  \frac{\bar{d}^*_h(s_h^k,a_h^k)}{d^{\mu}_h(s_h^k,a_h^k)} \|  \phi_h(s^k_h, a^k_h)  \|_{(\Sigma^k_h)^{-1}} +\frac{4 \beta({\delta})}{K} \sum_{h=1}^H \sqrt{\log \left(\frac{H}{\delta} \right) \sum_{k=1}^K \left(\frac{\bar{d}^*_h(s_h^k,a_h^k)}{d^{\mu}_h(s_h^k,a_h^k)} \right)^2 }  \\
    &+ \frac{2}{K} +  \frac{16 H}{3 K}\log\left( \frac{\log_2(KH)}{ \delta} \right) + gap_{support},
\end{align*}
where the first inequality is by Eq. (\ref{eq:tempppp}), the second inequality is by $\bar{\pi}^*$ is an optimal policy under $\bar{\mathcal{M}}$, the third inequality is by Theorem \ref{theorem:sublinear_subopt} (as under $\bar{\mathcal{M}}$, the single concentrability holds), and the last equality is by $\bar{d}^{\mu} = d^{\mu}$ (Lemma \ref{lemma:off_support}). 
% The key is at the decomposition step. \thanh{Write it down soon}
\end{proof}

% the first lemma that says that under the augmented MDP $\bar{\mathcal{M}}$, $\bar{d}^*$ is dominated by $\bar{d}^{\mu}$. 
% \begin{lemma}
% For any $(h,s_h, a_h) \in [H] \times \bar{\mathcal{S}} \times \mathcal{A}$, if $\bar{d}^*_h(s_h, a_h) > 0$, then $\bar{d}^{\mu}_h(s_h, a_h) > 0$. 
% \end{lemma}

% \begin{proof}
% Consider any $(h,s_h, a_h) \in [H] \times \bar{\mathcal{S}} \times \mathcal{A}$ such that $\bar{d}^*_h(s_h, a_h) > 0$. Note that $\bar{d}^*_h(s_h, a_h)  = \bar{d}^*_h(s_h) \bar{\pi}^*_h(a_h | s_h) > 0$. Thus, $\bar{\pi}^*_h(a_h | s_h) > 0$, i.e. $a_h$ is an optimal action given $s_h$. This must imply $(s_h, a_h) \in \mathcal{C}_h^{\mu}$ as otherwise, by definition, $\bar{r}_h(s_h, a_h) = 0$ 

\subsubsection{OPC is necessary for learnability from offline data}
\label{subsection:single_concentrability_is_necessary}
This result affirms the necessity of the single-policy concentrability for learnability in offline RL. 
\begin{lemma}
For any offline algorithm $\textrm{Algo}(\cdot)$, there exist an MDP instance $\mathcal{M}$ and an offline dataset $\mathcal{D}$ such that $\mathcal{M}$ generates $\mathcal{D}$, the single concentrability (Assumption \ref{assumption:single_policy_concentration}) is not met, and $\textrm{Algo}(\mathcal{D})$ incurs a constant sub-optimality almost surely over the randomness of $\mathcal{D}$.
\end{lemma}

\begin{proof}
Consider a class of bandit instances parameterized by $B(q_1, q_2, q_3)$ where $\{a_1, a_2, a_3\}$ is the shared action space of the class and $q_i$ is the corresponding deterministic reward of action $a_i$ for any $i \in [3]$. Now consider two bandit instances within the above class, namely $B_1 := B(0,0,1)$ and $B_2 := B(0,1,0)$, and construct dataset $\mathcal{D} = \{(b_t, r_t)\}_{t \in [K]}$ where $b_t = a_1, \forall t \in [K]$ and $r_t = 0, \forall t \in [K]$. As $\mathcal{D}$ only selects action $a_1$ and receives reward $r_t = 0$ while the reward of $a_1$ under both $B_1$ and $B_2$ is also $0$, $\mathcal{D}$ is consistent with both $B_1$ and $B_2$ (in the sense that $\mathcal{D}$ could have been generated under either $B_1$ or $B_2$). Note that $\mathcal{D}$ does not satisfy the single concentrability (Assumption \ref{assumption:single_policy_concentration}) under both $B_1$ and $B_2$ as $\mathcal{D}$ covers only action $a_1$ while the optimal actions for $B_1$ and $B_2$ are $a_2$ and $a_3$, respectively. \\

As the dataset $\mathcal{D}$ provides no information about $a_2$ and $a_3$, for any algorithm $\pi = \textrm{Algo}(\mathcal{D}) = (\pi(a_1), \pi(a_2), \pi(a_3) ) \in \{(p_1, p_2, p_3): p_i \geq 0, p_1 + p_2 + p_3 = 1\}$, $\pi(a_2)$ and $\pi(a_3)$ do not depend on $\mathcal{D}$. Without loss of generality, suppose $\pi(a_2) \geq \pi(a_3)$. Then, $\pi(a_3) \leq \frac{\pi(a_2) + \pi(a_3)}{2} \leq \frac{\pi(a_2) + \pi(a_3) + \pi(a_1)}{2} = \frac{1}{2}$. Thus, $\pi(a_1) + \pi(a_2) = 1 - \pi(a_3) \geq \frac{1}{2}$.  Therefore, we have: 
\begin{align*}
    \subopt(\textrm{Algo}(\mathcal{D}); B_2) = \pi(a_1) + \pi(a_2) \geq \frac{1}{2}.  
\end{align*}

It is crucial to note that the above inequality holds almost surely over the randomness of $\mathcal{D}$ as $\pi(a_2)$ and $\pi(a_3)$ are agnostic to $\mathcal{D}$. Thus, any $\textrm{Algo}(\mathcal{D})$ almost surely suffers a sub-optimality at least as large as $\frac{1}{2}$ under at least $B_1$ or $B_2$. 
\end{proof}

\section{Proofs for Section \ref{section:instance_dependent_structures}}
\subsection{Proof of Theorem \ref{theorem:logarithimic_regret}}
\label{section:logarithimic_regret_proof}
In this section, we provide the detailed proof for Theorem \ref{theorem:logarithimic_regret}. We first state and prove a series of intermediate lemmas. The following lemma decomposes the sub-optimality of any policy into the gap information. 
\begin{lemma}[Sub-optimality decomposition]
% For any policy $\hat{\pi}^k = \{\pi_h^k : h \in [H]\}$ with corresponding filtration $\{\mathcal{F}_k\}_{k \in [K]}$ (i.e. $\hat{\pi}^k$ is $\mathcal{F}_k$-measurable), we have
We have:
    \begin{align*}
        \forall (s_1, k) \in \mathcal{S} \times [K], V^*_1(s_1) - V^{\hat{\pi}^k}_1(s_1) = \mathbb{E}_{\hat{\pi}^k} \left[ \sum_{h=1}^H \Delta_h(s_h, a_h) \bigg | \mathcal{F}_{k-1}, s_1 \right]. 
        % = \frac{1}{H}  \sum_{h=1}^H \mathbb{E}_{(s,a) \sim d^{\hat{\pi}}_h }\left[ \Delta_h(s,a) | s_1 = s \right]
    \end{align*}
\label{lemma: suboptimality as gaps}
\end{lemma}
\begin{proof}[Proof of Lemma \ref{lemma: suboptimality as gaps}]
% See \citep[Eq.~B.3]{he2021logarithmic}. 
Conditioned on $\mathcal{F}_{k-1}$ and $s_1$, we have:
\begin{align*}
    V^*_1(s_1) - V^{\hat{\pi}^k}_1(s_1) &= V^*_1(s_1) -  Q^*_1(s_1, \hat{\pi}^k_1(s_1)) +  Q^*_1(s_1, \hat{\pi}^k_1(s_1)) - Q^{\hat{\pi}^k}_1(s_1, \hat{\pi}^k_1(s_1)) \\
    &= \mathbb{E}_{\hat{\pi}^k} \left[ \Delta_1(s_1, a_1) | \mathcal{F}_{k-1}, s_1 \right] + \mathbb{E}_{\hat{\pi}^k} \left[ V^*_2(s_2) - Q^{\hat{\pi}^k}_2(s_2, a_2) | \mathcal{F}_{k-1}, s_1 \right] \\
    &= \mathbb{E}_{\hat{\pi}^k} \left[ \Delta_1(s_1, a_1) | \mathcal{F}_{k-1}, s_1 \right] + \mathbb{E}_{\hat{\pi}^k} \left[ V^*_2(s_2) - V^{\hat{\pi}^k}_2(s_2) | \mathcal{F}_{k-1}, s_1 \right]. 
\end{align*}
Recursively applying the above equation over $h \in [H]$ and using the telescoping sum complete the proof. 
\end{proof}
The next lemma shows that any policy in the $\mu$-supported policy class $\Pi(\mu)$ induce marginalized density functions that concentrate only within the support of the marginalized density functions under $\mu$. 
\begin{lemma}[Concentrability for the $\mu$-supported policy class]
For any $(\pi, h, s_h, a_h) \in \Pi(\mu) \times [H] \times \mathcal{S} \times \mathcal{A}$, we have:
\begin{align*}
    \frac{d_h^{\pi}(s_h, a_h)}{d_h^{\mu}(s_h, a_h)}  < \infty.
\end{align*}
\label{lemma:concentrability for constrained policy class}
\end{lemma}

\begin{proof}[Proof of Lemma \ref{lemma:concentrability for constrained policy class}]
Consider any $\pi \in \Pi(\mu)$. Note that the lemma statement is equivalent to 
\begin{align}
    \forall h \in [H], \mathcal{S}^{\pi}_h \subseteq \mathcal{S}^{\mu}_h, \text{ and } \mathcal{S}\mathcal{A}^{\pi}_h \subseteq \mathcal{S}\mathcal{A}^{\mu}_h. 
    \label{eq:constrained policies have the same support as mu}
\end{align}
We prove Eq. (\ref{eq:constrained policies have the same support as mu}) by induction with $h$. We have $\mathcal{S}^{\pi}_1 = \mathcal{S}^{\mu}_1 = \mathcal{S}_1$ by definition. For any $(s_1, a_1) \in \mathcal{S}\mathcal{A}^{\pi}_1$, we have $s_1 \in \mathcal{S}_1$ and $\pi_1(a_1 | s_1) > 0$. By the definition of $\Pi_1(\mu)$, $\mu_1(a_1 | s_1) > 0$. Thus, we have $\mathcal{S}\mathcal{A}^{\pi}_1 \subseteq \mathcal{S}\mathcal{A}^{\mu}_1$, i.e. Eq. (\ref{eq:constrained policies have the same support as mu}) holds for $h = 1$. 

Now assume that  Eq. (\ref{eq:constrained policies have the same support as mu}) holds for $h \geq 1$, we prove that Eq. (\ref{eq:constrained policies have the same support as mu}) holds for $h + 1$. Indeed, since $\mathcal{S}\mathcal{A}^{\pi}_h \subseteq \mathcal{S}\mathcal{A}^{\mu}_h$, we have:
\begin{align*}
    &\mathcal{S}^{\pi}_{h+1} = \{s_{h+1} \in \mathcal{S}_{h+1}: \exists (s_h, a_h) \in \mathcal{S} \mathcal{A}_h^{\pi} \text{ such that } p_h(s_{h+1} | s_h, a_h) > 0\} \\ &\subseteq \{s_{h+1} \in \mathcal{S}_{h+1}: \exists (s_h, a_h) \in \mathcal{S} \mathcal{A}_h^{\mu} \text{ such that } p_h(s_{h+1} | s_h, a_h) > 0\} = \mathcal{S}^{\mu}_{h+1}.
\end{align*}

Now consider any $(s_{h+1} , a_{h+1}) \in  \mathcal{S}\mathcal{A}^{\pi}_{h+1}$. Then, we have $s_{h+1} \in \mathcal{S}^{\pi}_{h+1} \subseteq \mathcal{S}^{\mu}_{h+1}$ and $\pi_{h+1}(a_{h+1} | s_{h+1}) > 0$. By the definition of $\Pi_{h+1}(\mu)$, we have $\mu_h(a_{h+1} | s_{h+1}) > 0$. Thus, $(s_{h+1} , a_{h+1}) \in  \mathcal{S}\mathcal{A}^{\mu}_{h+1}$.
\end{proof}

The next lemma uses marginalized importance sampling to handle the distributional shift of the offline data to connect the sub-optimality of each $\hat{\pi}^k$ to the sub-optimality gap $\Delta_h(s^k_h, \hat{\pi}^k_h(s^k_h))$ under the behavior policy $\mu$.  
\begin{lemma}[Marginalized importance sampling for $\hat{\pi}^k$]
Under Assumption \ref{assumption:lower bound density}, w.p.a.l. $1 - \delta$, we have
For any $k \in [K]$, we have:
\begin{align*}
    \sum_{k=1}^K \subopt(\hat{\pi}^k; s_1^k) \leq 2 \sum_{h=1}^H \kappa_h \sum_{k=1}^K   \Delta_h(s_h^k, \hat{\pi}^k_h(s_h^k)) + \frac{16}{3} H \kappa^* \log \log_2 (K H \kappa^* / \delta) + 2,
\end{align*}
where $\kappa^* := \max_{h \in [H]} \kappa_h$. 
\label{lemma: mis for constrained policies}
\end{lemma}
\begin{proof}[Proof of Lemma \ref{lemma: mis for constrained policies}] The key for the proof is to use the marginalized importance sampling due to Lemma \ref{lemma: mis for constrained policies} and then apply the improved online-to-batch argument Lemma \ref{lemma:improved_online_to_batch}. In particular, we have: 
\begin{align*}
     \sum_{k=1}^K \subopt(\hat{\pi}^k; s_1^k) &=  \sum_{k=1}^K \mathbb{E}_{\hat{\pi}^k} \left[ \sum_{h=1}^H \Delta_h(s_h, \hat{\pi}^k_h(s_h)) \bigg | \mathcal{F}_{k-1}, s_1^k \right] \\
     &\leq \sum_{k=1}^K \mathbb{E}_{\mu}  \left[ \sum_{h=1}^H \kappa_h \Delta_h(s_h, \hat{\pi}^k_h(s_h)) \bigg | \mathcal{F}_{k-1}, s_1^k \right] \\ 
     &\leq 2 \sum_{k=1}^K  \sum_{h=1}^H \kappa_h \Delta_h(s_h^k, \hat{\pi}^k_h(s_h^k)) + \frac{16}{3} H \kappa^* \log \log_2 (K H \kappa^* / \delta) + 2,
\end{align*}
where the first equation is by Lemma \ref{lemma: suboptimality as gaps}, the first inequality is by Assumption \ref{assumption:lower bound density}, and the second inequality is by Lemma \ref{lemma:improved_online_to_batch}. 
\end{proof}

The following lemma bounds the number of times a sub-optimality gap $\Delta_h(s_h^k, \hat{\pi}^k_h(s_h^k)$ exceeds a certain threshold.  
\begin{lemma}
Under Assumption \ref{assumption:single_policy_concentration}-\ref{assumption:lower bound density}, for any $\Delta > 0$, w.p.a.l. $1 - 3 \delta$, for any $h \in [H]$, we have:
\begin{align*}
    \sum_{k=1}^K \mathbbm{1}\{ \Delta_h(s_h^k, \hat{\pi}^k_h(s_h^k)) \geq \Delta \} \lesssim  \frac{ d^3 H^2 \kappa^2}{\Delta^2} \log^3(dKH/\delta). 
\end{align*}
\label{lemma:positive_subopt_count}
\end{lemma}

% \begin{proof}
% We have 
% \begin{align*}
%     \sum_{k=1}^K \Delta_h(s_h^k, a_h^k) \leq \sum_{k=1}^K \sum_{i=1}^m \mathbbm{1}\{2^{i-1} \Delta_{\min} \leq \Delta_h(s_h^k, a_h^k) \leq 2^i \Delta_{\min}  \} 2^i \Delta_{\min}
% \end{align*}
% \end{proof}

\begin{proof}[Proof of Lemma \ref{lemma:positive_subopt_count}]
Define  $K' = \sum_{k=1}^K \mathbbm{1}\{\Delta_h(s_h^k, \hat{\pi}^k_h(s_h^k)) \geq \Delta \}$. Note that $K'$ is the number of episodes $k$ where $\Delta_h(s_h^k, a_h^k)$ is bounded below by $\Delta$. Define $\{k_i\}_{i \in [K']}$ such episodes, i.e. $k_i = \min\{k \in [K]: k \geq k_{i-1}, \Delta_h(s_h^k, a_h^k) \geq \Delta \}$. Then, we have: 
\begin{align*}
    \sum_{i=1}^{K'} \Delta_h(s_h^{k_i}, a_h^{k_i}) \geq K' \Delta. 
\end{align*}

% With probability at least $1 - \delta$, we have
% \begin{align*}
%     \Delta_h(s_h^{k_i}, \hat{\pi}^{k_i}_h(s_h^{k_i})) &= V^*_h(s_h^{k_i}) - Q^*_h(s_h^{k_i},  \hat{\pi}^{k_i}_h(s_h^{k_i})) \leq V^*_h(s_h^{k_i}) - Q^{\hat{\pi}^{k_i}}_h(s_h^{k_i},  \hat{\pi}^{k_i}_h(s_h^{k_i})) \\
%     &= V^*_h(s_h^{k_i}) - V^{\hat{\pi}^{k_i}}_h(s_h^{k_i}) \\ 
%     &\leq 2 \mathbb{E}_{\pi^*} \left[ \sum_{h'=h}^H b_{h'}^{k_i}(s_{h'}, a_{h'}) | s_h^{k_i}\right] 
% \end{align*}
Thus,  with probability at least $1 - 3 \delta$, for any $h \in [H]$, we have
\begin{align*}
    &\sum_{i = 1}^{K'}  \Delta_h(s_h^{k_i}, \hat{\pi}^{k_i}_h(s_h^{k_i}))
     = \sum_{i = 1}^{K'} V^*_h(s_h^{k_i}) - Q^*_h(s_h^{k_i},  \hat{\pi}^{k_i}_h(s_h^{k_i})) \leq \sum_{i = 1}^{K'} V^*_h(s_h^{k_i}) - Q^{\hat{\pi}^{k_i}}_h(s_h^{k_i},  \hat{\pi}^{k_i}_h(s_h^{k_i})) \\
    &= \sum_{i = 1}^{K'}  V^*_h(s_h^{k_i}) - V^{\hat{\pi}^{k_i}}_h(s_h^{k_i}) \\
    &\leq 2  \sum_{h'=h}^H \sum_{i = 1}^{K'} \mathbb{E}_{\pi^*} \left[ b_{h'}^{k_i}(s_{h'}, a_{h'}) | s_h^{k_i}\right] \\ 
    &= 2 \sum_{h'=h}^H \sum_{i = 1}^{K'} \beta_{k_i}(\delta) \mathbb{E}_{\pi^*} \left[ \| \phi_{h'}(s_{h'}, a_{h'}) \|_{ (\Sigma_{h'}^{k_i})^{-1} } | s_h^{k_i}\right] \\
    &\leq 2  \beta_{K'}(\delta) \sum_{h'=h}^H \sum_{i=1}^{K'}  \frac{d^*_{h'}(s_{h'}^{k_i},a_{h'}^{k_i})}{d^{\mu}_{h'}(s_{h'}^{k_i},a_{h'}^{k_i})} \| \phi_{h'}(s^{k_i}_{h'}, a^{k_i}_{h'})  \|_{(\Sigma^{k_i}_{h'})^{-1}} +  2  \beta_{K'}(\delta) \sum_{h'=h}^H \sqrt{ \log(1/\delta)} \sqrt{\sum_{i=1}^{K'} \left(\frac{d^*_{h'}(s_{h'}^{k_i},a_{h'}^{k_i})}{d^{\mu}_{h'}(s_{h'}^{k_i},a_{h'}^{k_i})} \right)^2 } \\ 
    &\leq 2  \beta_{K'}(\delta) \sum_{h'=h}^H \kappa_{h'} \sqrt{ 2 K' d \log(1 + K'/d)} + 2  \beta_{K'}(\delta) \sqrt{K' \log(1/\delta)} \sum_{h'=h}^H \kappa_{h'} \\ 
    &\leq 2 \kappa  \beta_{K'}(\delta) (\sqrt{ 2 K' d \log(1 + K'/d)} + 2 \sqrt{K' \log(1/\delta)} ) \\ 
    &\lesssim \kappa H d^{3/2} K'^{1/2} \log^{3/2}(dK' H / \delta)
\end{align*}
where the second inequality is by Lemma \ref{lemma:bound sub-opt with uncertainty}, the third equality is by Lemma \ref{lemma:uncertainty_quantifier}, the third inequality is by Lemma \ref{lemma:sum_sample_subopt}, the fourth inequality is by Lemma \ref{lemma:bound_shrinking_sum}. Thus, we have:
\begin{align*}
    K' \lesssim  \frac{ d^3 H^2 \kappa^2}{\Delta^2} \log^3(dKH/\delta).
\end{align*}
\end{proof}

Next we bound the total sub-optimality gaps accumulated over $K$ episodes under $\mu$. 
\begin{lemma}
Under Assumption \ref{assumption:single_policy_concentration}-\ref{assumption:lower bound density}-\ref{assumption:margin}, with probability at least $1 - 3 \log_2(H / \Delta_{\min}) \delta$, for any $h \in [H]$, 
\begin{align*}
    \sum_{k=1}^K   \Delta_h(s_h^k, \hat{\pi}^k_h(s_h^k)) \lesssim  \frac{ d^3 H^2 \kappa^2}{\Delta_{\min}} \log^3(dKH/\delta).
\end{align*}
\label{lemma: accumulative gaps are bounded by polylog}
\end{lemma}

\begin{proof}[Proof of Lemma \ref{lemma: accumulative gaps are bounded by polylog}]
Let $m = \log_2(H / \Delta_{\min})$. As $\Delta_h(s_h^k, \hat{\pi}^k_h(s_h^k)) \in [0,H]$,  and $\Delta_h(s_h^k, \hat{\pi}^k_h(s_h^k)) = 0$ if $\Delta_h(s_h^k, \hat{\pi}^k_h(s_h^k) < \Delta_{\min}$, we have: 
\begin{align*}
    \sum_{k=1}^K   \Delta_h(s_h^k, \hat{\pi}^k_h(s_h^k)) &\leq \sum_{k=1}^K  \sum_{i=1}^m \mathbbm{1} \{ 2^{i-1} \Delta_{\min} \leq \Delta_h(s_h^k, \hat{\pi}^k_h(s_h^k)) < 2^i \Delta_{\min} \} 2^i \Delta_{\min} \\
    &\leq \sum_{i=1}^m 2^i \Delta_{\min} \sum_{k=1}^K \mathbbm{1} \{ \Delta_h(s_h^k, \hat{\pi}^k_h(s_h^k)) \geq 2^{i-1} \Delta_{\min} \} \\ 
    &\lesssim \sum_{i=1}^m 2^i \Delta_{\min} \frac{ d^3 H^2 \kappa^{2}}{(2^{i-1}\Delta_{\min})^2} \log^3(dKH/\delta) \\ 
    &\lesssim  \frac{ d^3 H^2 \kappa^{2}}{\Delta_{\min}} \log^3(dKH/\delta),
\end{align*}
where the first inequality is the peeling argument and the third inequality is by Lemma \ref{lemma:positive_subopt_count}. 
\end{proof}

Theorem \ref{theorem:logarithimic_regret} is a direct combination of Lemma \ref{lemma: mis for constrained policies} and Lemma \ref{lemma: accumulative gaps are bounded by polylog} via union bound.

\subsection{Proof of Theorem \ref{theorem:constant_regret}}
\label{section:constant_regret_appendix}
Let $\Xi_h$ be the set of all trajectories $\tau_h := (s_1, a_1, \ldots, s_h, a_h)$ induced by the underlying MDP and some policy $\pi$, i.e. $s_1 \sim d_1, a_i \sim \hat{\pi}_i(\cdot| s_i), s_{i+1} \sim \mathbb{P}_i(\cdot|s_i, a_i)$.  Let 
\begin{align*}
    \mathcal{E}^k_h = \{\tau_h = (s_1, a_1, \ldots, s_h, a_h) \in \Xi_h : \forall i \in [h], a_i = \pi^*_i(s_i) = \hat{\pi}_h(s_i) \},
\end{align*}
be the set of all $h$-length trajectories $(s_1, a_1, \ldots, s_h, a_h)$ at which $\hat{\pi}^k$ and $\pi^*$ agree on up to step $h$. 

% \thanh{The stage-wise unique optimal actions are necessary}

\subsection*{Support lemmas}
Next we show that the probability that $\pi^*$ and $\hat{\pi}^k$ do not agree on a $h$-length trajectory is controlled by the sub-optimality and the minimum value gap.  
\begin{lemma}
Under Assumption \ref{assumption:margin}-\ref{assumption:spaning_feature}.1, for any $(k,h) \in [K] \times [H]$, if $f(k) \geq \sum_{t=1}^k \subopt(\hat{\pi}^t)$,  we have:
\begin{align*}
    \sum_{t=1}^k \mathbb{E}_{\tau_h \sim d^{\hat{\pi}^t}} \left[ \mathbbm{1}\{\tau_h \notin \mathcal{E}^t_h \} | \mathcal{F}_{t-1} \right] \leq \frac{1}{\Delta_{\min}} f(k). 
\end{align*}
\label{lemma:pistart_and_pihat_disagree}
\end{lemma}
\begin{proof}[Proof of Lemma \ref{lemma:pistart_and_pihat_disagree}]
We have:
\begin{align*}
    \sum_{t=1}^k \mathbb{E}_{\tau_h \sim d^{\hat{\pi}^t}} \left[ \mathbbm{1}\{\tau_h \notin \mathcal{E}^t_h \} | \mathcal{F}_{t-1} \right] &\leq \sum_{t=1}^k \sum_{i=1}^h \mathbb{E}_{(s_i, a_i) \sim d_i^{\hat{\pi}^t}} \left[ \mathbbm{1}\{a_i \neq \pi^*_i(s_i) \} | \mathcal{F}_{t-1}  \right] \\
    &= \sum_{t=1}^k \sum_{i=1}^h \mathbb{E}_{(s_i, a_i) \sim d_i^{\hat{\pi}^t}} \left[ \mathbbm{1}\{ \Delta_i(s_i, a_i) \geq \Delta_{\min} \} | \mathcal{F}_{t-1} \right] \\
    &\leq \sum_{t=1}^k \sum_{i=1}^h \mathbb{E}_{(s_i, a_i) \sim d_i^{\hat{\pi}^t}} \left[ \frac{\Delta_i(s_i, a_i)}{\Delta_{\min}} \bigg| \mathcal{F}_{t-1}  \right] \\ 
    &= \frac{1}{\Delta_{\min}} \sum_{t=1}^k \mathbb{E}_{\hat{\pi}^t} \left[\sum_{i=1}^h \Delta_i(s_i, a_i) \bigg| \mathcal{F}_{t-1} \right] \\ 
    &\leq \frac{1}{\Delta_{\min}} \sum_{t=1}^k \mathbb{E}_{\hat{\pi}^t} \left[\sum_{i=1}^H \Delta_i(s_i, a_i) \bigg| \mathcal{F}_{t-1} \right] \\ 
    &=\frac{1}{\Delta_{\min}} \sum_{t=1}^k \subopt(\hat{\pi}^t) \\
    &\leq \frac{1}{\Delta_{\min}} f(k)
\end{align*}
where the first equation is by Assumption \ref{assumption:margin} and Assumption \ref{assumption:spaning_feature}.1, and the last equation is by Lemma \ref{lemma: suboptimality as gaps}. 
\end{proof}

The following lemma lower-bounds the empirical accumulated covariance matrix $\Sigma^{k+1}_h$ by the covariance matrix at the optimal actions. 
\begin{lemma}
Under Assumption \ref{assumption:single_policy_concentration}-\ref{assumption:margin}-\ref{assumption:spaning_feature}.1, for any $(k,h) \in [K] \times [H]$, with probability at least $1 - \delta$, we have:
\begin{align*}
    \Sigma^{k+1}_h \succeq \lambda I + \kappa_{1:h}^{-1} k \Sigma_h^* - \kappa_{1:h}^{-1} I \frac{1}{\Delta_{\min}} \sum_{t=1}^k \subopt(\hat{\pi}^t) - \frac{2}{3}\log(d/\delta) I - \sqrt{2 k \log(d/\delta)} I. 
\end{align*}
\label{lemma:bound_Sigma_by_Sigma_star}
\end{lemma}
\begin{proof}
Let $Z_t := \mathbbm{1}\{\tau_h^t \in \mathcal{E}^t_h \} \phi_h^*(s_h^t) \phi_h^*(s_h^t)^T -  \mathbb{E} \left[ \mathbbm{1}\{\tau_h^t \in \mathcal{E}^t_h \} \phi_h^*(s_h^t) \phi_h^*(s_h^t)^T | \mathcal{F}_{t-1} \right]$ and $\kappa_{1:h} := \prod_{i=1}^h \kappa_i$. We have:
\begin{align*}
&\Sigma_h^{k+1} - \lambda I = \sum_{t=1}^{k} \phi_h(s_h^t, a_h^t) \phi_h(s^t_h, a^t_h)^T \\
&\succeq \sum_{t=1}^{k} \mathbbm{1}\{\tau_h^t \in \mathcal{E}^t_h \} \phi_h(s_h^t, a_h^t) \phi_h(s^t_h, a^t_h)^T \\ 
&\overset{(a)}{=}  \sum_{t=1}^{k} \mathbbm{1}\{\tau_h^t \in \mathcal{E}^t_h \} \phi_h^*(s_h^t) \phi_h^*(s_h^t)^T \\ 
&=\sum_{t=1}^k \mathbb{E}_{\tau_h \sim d^{\mu}} \left[ \mathbbm{1}\{\tau_h \in \mathcal{E}^t_h \} \phi_h^*(s_h) \phi_h^*(s_h)^T |\mathcal{F}_{t-1} \right] + \sum_{t=1}^k Z_t \\ 
&\overset{(b)}{\succeq}  \sum_{t=1}^k  \kappa_{1:h}^{-1} \mathbb{E}_{\tau_h \sim d^*} \left[ \mathbbm{1}\{\tau_h \in \mathcal{E}^t_h \} \phi_h^*(s_h) \phi_h^*(s_h)^T | \mathcal{F}_{t-1} \right] + \sum_{t=1}^k Z_t \\
% &\overset{(c)}{\succeq}  \sum_{t=1}^k  \kappa^{-h} \mathbb{E}_{\tau_h^t \sim d^{\hat{\pi}^t}} \left[ \mathbbm{1}\{\tau_h^t \in \mathcal{E}^t_h \} \phi_h^*(s_h^t) \phi_h^*(s_h^t)^T | \mathcal{F}_{t-1} \right] + \sum_{t=1}^k Z_t \\
&=  \sum_{t=1}^k \kappa_{1:h}^{-1} \mathbb{E}_{\tau_h \sim d^*} \left[ \phi_h^*(s_h) \phi_h^*(s_h)^T | \mathcal{F}_{t-1} \right] -  \sum_{t=1}^k  \kappa_{1:h}^{-1} \mathbb{E}_{\tau_h \sim d^*} \left[ \mathbbm{1}\{\tau_h \notin \mathcal{E}_h \} \phi_h^*(s_h) \phi_h^*(s_h)^T | \mathcal{F}_{t-1} \right] + \sum_{t=1}^k Z_t \\
&\overset{(c)}{\succeq}  \sum_{t=1}^k \kappa_{1:h}^{-1} \Sigma^*_h -  \kappa_{1:h}^{-1} I  \sum_{t=1}^k \mathbb{E}_{\tau_h \sim d^*} \left[ \mathbbm{1}\{\tau_h \notin \mathcal{E}^t_h \} | \mathcal{F}_{t-1} \right] + \sum_{t=1}^k Z_t \\
&=  \kappa_{1:h}^{-1} k \Sigma^*_h -  \kappa_{1:h}^{-1} I \sum_{t=1}^k \mathbb{E}_{\tau_h \sim d^*} \left[ 1 - \mathbbm{1}\{\tau_h \in \mathcal{E}^t_h \} | \mathcal{F}_{t-1} \right] + \sum_{t=1}^k Z_t \\
&=  \kappa_{1:h}^{-1} k \Sigma^*_h -  \kappa_{1:h}^{-1} I \sum_{t=1}^k (1 -\mathbb{E}_{\tau_h \sim d^*} \left[\mathbbm{1}\{\tau_h \in \mathcal{E}^t_h \} | \mathcal{F}_{t-1} \right]) + \sum_{t=1}^k Z_t \\
&\overset{(d)}{=}   \kappa_{1:h}^{-1} k \Sigma^*_h -  \kappa_{1:h}^{-1} I \sum_{t=1}^k (1 -\mathbb{E}_{\tau_h \sim d^{\hat{\pi}^t}} \left[\mathbbm{1}\{\tau_h \in \mathcal{E}^t_h \} | \mathcal{F}_{t-1} \right]) + \sum_{t=1}^k Z_t \\
%  &=  \sum_{t=1}^k \kappa^{-h} \mathbb{E}_{\tau_h^t \sim d^*} \left[ \phi_h^*(s_h^t) \phi_h^*(s_h^t)^T | \mathcal{F}_{t-1} \right] - \kappa^{-h} I  \sum_{t=1}^k (1 - \mathbb{P}\{\tau_h \in \mathcal{E}^t_h | \tau_h \sim d^*, \mathcal{F}_{t-1} \}) + \sum_{t=1}^k Z_t \\
&=  \kappa_{1:h}^{-1} k \Sigma^*_h -   \kappa_{1:h}^{-1} I \underbrace{\sum_{t=1}^k \mathbb{E}_{\tau_h^t \sim d^{\hat{\pi}^t}} \left[ \mathbbm{1}\{\tau_h \notin \mathcal{E}^t_h \} | \mathcal{F}_{t-1} \right]}_{(i)} + \underbrace{\sum_{t=1}^k Z_t}_{(ii)} 
% &= \sum_{t=1}^k \mathbb{E}_{s_h \sim d^*_h} \left[ \phi_h^*(s_h) \phi_h^*(s_h)^T \right] - I \sum_{t=1}^k \mathbb{E}_{\tau_h \sim d^{\hat{\pi}^t}} \left[ \mathbbm{1}\{\tau_h^t \notin \mathcal{E}^t_h \} | \mathcal{F}_{t-1} \right] + \sum_{t=1}^k Z_t \\
% &= \Sigma^*_h - I \sum_{t=1}^k \mathbb{E}_{\tau_h^t \sim d^{\hat{\pi}^t}} \left[ \mathbbm{1}\{\tau_h^t \notin \mathcal{E}^t_h \} | \mathcal{F}_{t-1} \right] + \sum_{t=1}^k Z_t,
\end{align*}
where $(a)$ is by the definition of $\mathcal{E}^k_h$, $(b)$ is by that under Assumption \ref{assumption:single_policy_concentration}, we have
\begin{align*}
    \frac{d^{\mu}(\tau_h)}{d^*(\tau_h)} = \frac{d^{\mu}_1(s_1, a_1) \mathbb{P}_1(s_2|s_1,a_1) \ldots \mathbb{P}_{h-1}(s_h|s_{h-1},a_{h-1}) d^{\mu}_h(s_h, a_h)}{d^*_1(s_1, a_1) \mathbb{P}_1(s_2|s_1,a_1) \ldots \mathbb{P}_{h-1}(s_h|s_{h-1},a_{h-1}) d^*_h(s_h, a_h)} = \prod_{i=1}^h \frac{ d^{\mu}_i(s_i, a_i)}{d^*_i(s_i, a_i)} \geq \kappa_{1:h}^{-1},
\end{align*}
$(c)$ is by that $\phi_h^*(s_h) \phi_h^*(s_h)^T \preceq I \cdot \| \phi_h^*(s_h) \phi_h^*(s_h)^T \| \leq I \| \phi_h^*(s_h) \|_2^2 = I$, and that 
\begin{align*}
    \mathbb{E}_{\tau_h^t \sim d^*} \left[  \phi_h^*(s_h^t) \phi_h^*(s_h^t)^T | \mathcal{F}_{t-1} \right] = \mathbb{E}_{(s_h, a_h) \sim d^*_h} \left[  \phi_h^*(s_h) \phi_h^*(s_h)^T \right] = \Sigma^*_h, 
\end{align*}
and $(d)$ is by that $d^*(\tau_h | \mathcal{E}^t_h) = d^{\hat{\pi}^t}(\tau_h | \mathcal{E}^t_h)$. 

Term $(i)$ is bounded by Lemma \ref{lemma:pistart_and_pihat_disagree}. For term $(ii)$, note that $Z_t$ is $\mathcal{F}_t$-measurable, $\mathbb{E}\left[ Z_t | \mathcal{F}_{t-1} \right] = 0$, and $ \|Z_t \|_2 \leq 1$. Thus, by Lemma \ref{lemma:matrix_freedman}, with probability at least $1 - \delta$, we have:
\begin{align*}
    (ii) = \sum_{t=1}^k Z_t \geq - \frac{2}{3}\log(d/\delta) I - \sqrt{2 k \log(d/\delta)} I . 
    % - \sqrt{k \log(1/\delta)}. 
\end{align*}
% \thanh{Distribution shift happens here too!}
\end{proof}

\begin{lemma}
% Suppose that $\subopt(\hat{\pi}^t) \leq f(k), \forall k$ for some $f: \mathbb{N} \rightarrow \mathbb{R}$. 
% \raman{Seems like a reference to the streamed PEVI algorithm which has since been removed. Fix!} 
Let $\lambda_h^+$ be the smallest positive eigenvalue of $\Sigma_h^*$, $\kappa_{1:h} := \prod_{i=1}^h \kappa_h$, and define
\begin{align*}
    \bar{k}_h = \tilde{\Omega} \left( \frac{d^6 H^4 \kappa^6}{\Delta_{\min}^4 (\lambda_h^+)^2} + \frac{ \kappa_{1:h}}{\lambda_h^+} \right) \land \tilde{\Omega} \left( \frac{\kappa_{1:h}^2 \kappa^2 H^2 d^3}{ (\lambda_h^+)^2} \right). 
\end{align*}
Under Assumption \ref{assumption:single_policy_concentration}-\ref{assumption:margin}-\ref{assumption:spaning_feature}.1, w.p.a.l. $1 - 2\delta$, for any $h \in [H]$, any $k \geq \max_{h \in [H]}\bar{k}_h$, and any $v \in \textrm{col}(\Sigma^*_h)$ such that $\|v \|_2 \leq 1$, we have:
% all $(s_h,a_h)$ such that $\phi_h(s_h, a_h) \in \textrm{span}\{\phi^*_h(s) | d^*_h(s) > 0\}$, 
\begin{align*}
    \|v\|_{(\Sigma^k_h)^{-1}}= \mathcal{O} \left( \sqrt{\frac{\kappa_{1:h}}{(\lambda_h^{+})^3 k}} \right).
\end{align*}
\label{lemma:bound_feature_norm_by_smallest_positive_eigen_of_Sigma_star}
\end{lemma}
\begin{proof}
By Lemma \ref{lemma:bound_Sigma_by_Sigma_star} and the union bound, with probability at least $1 - \delta$, for any $(k,h) \in [K] \times [ H]$, we have: $\Sigma^{k+1}_h \succeq A^k_h$ where
\begin{align*}
    A^k_h := \lambda I + \kappa_{1:h}^{-1} k \Sigma_h^* - \kappa_{1:h}^{-1} I \frac{1}{\Delta_{\min}} f(k) - \frac{2}{3}\log(dKH/\delta) I - \sqrt{2 k \log(dKH/\delta)} I, 
\end{align*}
where $f(k)$ is any upper bound of $\sum_{t=1}^k\subopt(\hat{\pi}^t)$. We can choose $f(k) = \tilde{\mathcal{O}} \left( \kappa H d^{3/2}\sqrt{k} \right)$ by Corollary \ref{corollary:sublinear_subopt}, or $f(k) = \tilde{\mathcal{O}} \left( \frac{ d^3 H^2 \kappa^{3}}{\Delta_{\min} } \right)$ by Theorem \ref{theorem:logarithimic_regret}. 
% with probability at least $1 -  \delta$, we have $f(k) = \mathcal{O}(\sqrt{k})$.~\footnote{Corollary \ref{corollary:sublinear_subopt} provides guarantees for our sub-optimality notion $\subopt(\hat{\pi})$ but in its proof we directly bound a stronger notion of sub-optimality $v_1^* - v^{\hat{\pi}^k}_1$, thus all results in Corollary \ref{corollary:sublinear_subopt} applies to $v_1^* - v^{\hat{\pi}^k}_1$, too.}
% \begin{align*}
%     f(k) &\leq 2 c \frac{ d^3 H^2 \kappa^2}{\Delta_{\min} } \log^3(2 dkH  \log_2(H/\Delta_{\min})/\delta) + \frac{16 H}{3 } \log\left( \frac{2  \log_2(H/\Delta_{\min}) \log_2(kH)}{\delta} \right) + 2 \\
%     &= \mathcal{O}(\textrm{polylog}(k)). 
% \end{align*}
We fix $h$. Let $0 \leq \lambda_1 \leq \ldots \leq \lambda_d$ be the eigenvalues of $\Sigma^*_h$ with the corresponding orthonormal vectors $u_1, \ldots, u_d$, and let $\lambda_h^+$ the smallest positive eigenvalue of $\Sigma^*_h$. Then, $A^k_h$ has the eigenvalues $\lambda'_1 \leq \ldots \leq \lambda'_d$ with the same corresponding orthonormal vectors $u_1, \ldots, u_d$, where:
\begin{align*}
    \lambda'_i : = 1  + \kappa_{1:h}^{-1} k \lambda_i - \kappa_{1:h}^{-1}  \frac{1}{\Delta_{\min}} f(k) - \frac{2}{3}\log(dkH/\delta)  - \sqrt{2 k \log(dkH/\delta)} 
\end{align*}
% As $f(k) = \mathcal{O}(\textrm{polylog}(k))$, there exists $k_h = k_h(d,H,\delta, \{\kappa_h\}_{h \in [H]})$ (independent of $K$) such that for any

It is easy to verify that if $k \geq k_h$, we have:
\begin{align*}
    1   - \kappa_{1:h}^{-1}  \frac{1}{\Delta_{\min}} f(k) - \frac{2}{3}\log(dkH/\delta)  < \sqrt{2 k \log(dkH/\delta)}, \text{ and } \\
    1  + \kappa_{1:h}^{-1} k \lambda^+_h - \kappa_{1:h}^{-1}  \frac{1}{\Delta_{\min}} f(k) - \frac{2}{3}\log(dkH/\delta)  - \sqrt{2 k \log(dkH/\delta)} > 0.
\end{align*}
Thus, for any $k \geq k_h$,  $\lambda'_i \neq 0, \forall i \in [d]$, i.e. $A^k_h$ is invertible,  we have:
\begin{align*}
   \bar{\lambda}_h^+ := 1  + \kappa_{1:h}^{-1} k \lambda_h^+ - \kappa_{1:h}^{-1}  \frac{1}{\Delta_{\min}} f(k) - \frac{2}{3}\log(dkH/\delta)  - \sqrt{2 k \log(dkH/\delta)}
\end{align*}
is the smallest positive eigenvalue of $A^k_h$, and hence for any $v \in \textrm{span}\{\phi^*_h(s) | d^*_h(s) > 0\}$, such that $\|v \|_2 \leq 1$, let $x = v / \| v \|_2$. We have:
\begin{align*}
    \| v \|_{(\Sigma^{k+1}_h)^{-1}} \leq \| x \|_{(\Sigma^{k+1}_h)^{-1}} \leq \| x \|_{(A^{k}_h)^{-1}} \leq \frac{\lambda'_d}{\bar{\lambda}_h^+} \frac{1}{\| x \|_{A^{k}_h}} \leq \frac{\lambda'_d}{(\bar{\lambda}_h^+)^{3/2}} \frac{1}{\|x\|_2} = \frac{\lambda'_d}{(\bar{\lambda}_h^+)^{3/2}} = \mathcal{O} \left( \sqrt{\frac{\kappa_{1:h}}{(\lambda_h^{+})^3 k}} \right)
\end{align*}
where the first inequality is by $\| v\|_2 \leq 1$, the second inequality is by $(\Sigma^{k+1}_h)^{-1} \preceq (A^k_h)^{-1}$, the third inequality is by Lemma \ref{lemma:Kantovich}, and the fourth inequality is by Lemma \ref{lemma:smallest_positive_eigenvalue_via_column_space}, $\| v \|_{A^{k}_h} \geq \| v \|_2 \sqrt{\bar{\lambda}_h^+}$. Choosing $\bar{k} = \max_{h} \bar{k}_h$ completes the proof.
\end{proof}

\subsection*{Proof of Theorem \ref{theorem:constant_regret}}

\begin{proof}[Proof of Theorem \ref{theorem:constant_regret}]
Let $\mathcal{E}$ be the event that the inequalities in Corollary \ref{corollary:sublinear_subopt}, Theorem \ref{theorem:logarithimic_regret}, and Lemma \ref{lemma:bound_feature_norm_by_smallest_positive_eigen_of_Sigma_star} hold simultaneously. We now consider event $\mathcal{E}$ for the rest of the proof. Consider any state $s_h \in \mathcal{S}^{\mu}_h$. By Assumption \ref{assumption:spaning_feature}.2, we have $\phi_h^*(s_h) \in \textrm{col}(\Sigma^*_h)$. Thus, by Lemma \ref{lemma:bound_feature_norm_by_smallest_positive_eigen_of_Sigma_star}, if $k \geq \bar{k}_h$  
\begin{align*}
    2 \beta_k(\delta) \| \phi_h(s_h, a_h) \|_{(\Sigma_h^k)^{-1}} = \tilde{\mathcal{O}} \left( dH \sqrt{\frac{\kappa_{1:h}}{(\lambda_h^{+})^3 k}} \right). 
\end{align*}
We choose $k$ such that $k \geq \max_{h} \bar{k}_h$ and 
\begin{align*}
    \tilde{\Omega} \left( dH \sqrt{\frac{\kappa_{1:h}}{(\lambda_h^{+})^3 k}} \right) \leq \frac{\Delta_{\min}}{H}, 
\end{align*}
i.e. 
\begin{align*}
    k \geq \tilde{\Omega} \left( \frac{d^2 H^4 \kappa_{1:h}}{\Delta_{\min}^2 (\lambda_h^+)^3} \right), \forall h. 
\end{align*}
Then, we have: 
\begin{align*}
    \Delta_h(s_h, \hat{\pi}^k(s_h) ) &= V^*_h(s_h) - Q^*(s_h, \hat{\pi}^k(s_h) ) \\ 
    &\leq V^*_h(s_h) - Q^{\hat{\pi}^k}(s_h, \hat{\pi}^k(s_h) ) \\ 
    &\leq 2 \beta_k(\delta) \mathbb{E}_{\pi^*} \left[ \sum_{h'=h}^H \| \phi_h(s_h, a_h) \|_{(\Sigma_{h'}^k)^{-1}} | \mathcal{F}_{k-1}, s_h \right] \\
    &< (H - h + 1) \frac{\Delta_{\min}}{H} \leq \Delta_{\min}. 
\end{align*}
Thus, $\Delta_h(s_h, \hat{\pi}^k(s_h) ) = 0, \forall h$. Therefore, for any initial state $s_1 \sim d_1$, we have:
\begin{align*}
    \subopt(\hat{\pi}^k; s_1)  &= \mathbb{E}_{\hat{\pi}^k} \left[ \sum_{h=1}^H \Delta_h(s_h, a_h) \bigg | \mathcal{F}_{k-1}, s_1 \right] \\
    &= \mathbb{E}_{\hat{\pi}^k} \left[ \sum_{h=1}^H \Delta_h(s_h, \hat{\pi}^k_h(s_h)) \bigg | \mathcal{F}_{k-1}, s_1 \right] \\ 
    &\leq \mathbb{E}_{\mu} \left[ \sum_{h=1}^H \kappa_h \Delta_h(s_h, \hat{\pi}^k_h(s_h)) \bigg | \mathcal{F}_{k-1}, s_1 \right] = 0
\end{align*}
where the first equation is by Lemma \ref{lemma: suboptimality as gaps} and the inequality is by Lemma \ref{lemma:concentrability for constrained policy class} and Assumption \ref{assumption:lower bound density}.

% Thus we have: 
% \begin{align*}
%     \subopt
% \end{align*}

% Now, for any initial state $s_1$ and any $k \geq k^*$, by Lemma \ref{lemma:gap_decomposition}, we have:
% \begin{align*}
%   \subopt(\hat{\pi}^k; s_1 ) =  \mathbb{E}_{\hat{\pi}^k} \left[ \sum_{h=1}^H \Delta_h(s_h, a_h) \bigg | \mathcal{F}_{k-1}, s_1 \right] = 0.
% \end{align*}
% Hence, we have: $\subopt(\hat{\pi}_{PEVI}) = \subopt(\hat{\pi}^{K+1}) = 0$ for any $K \geq k^*$. Moreover, we have: 
% \begin{align*}
%   K \cdot \subopt(\hat{\pi}_{unif}) &= \sum_{k=1}^K \subopt(\hat{\pi}^k) \\ 
%   &= \sum_{k=1}^{k^*} \subopt(\hat{\pi}^k) \\ 
%   &\leq 2 c_2 \cdot \frac{ d^3 H^2 \kappa^2}{\Delta_{\min}} \log^3(d k^* H/\delta) + \frac{16}{3} H \log(\log_2(k^* H)/\delta) + 2,
% \end{align*}
% where the first equality is by the definition of $\hat{\pi}_{unif}$, the second equality is by that $\subopt(\hat{\pi}^k) = 0, \forall k \geq k^*$, and the inequality is by by Lemma \ref{lemma:expected_to_sample_subopt} applied on $k^*$.
\end{proof}

\subsection{Proof of Theorem \ref{theorem:lower_bound}}
\label{section:proof_lowerbound}
In this section, we give the proof of Theorem \ref{theorem:lower_bound}. To prove the lower bound, we construct the hard MDP instances introduced by \cite{jin2021pessimism}. The key difference is that in our construction, we need to carefully design the behavior policy $\mu$ to incorporate the optimal-policy concentrability $\{\kappa_{h}\}_{h \in [H]}$ and the minimum positive action gap $\Delta_{\min}$ into the lower bound. As a byproduct, we also construct a minimax lower bound without incorporating $\Delta_{\min}$.  

\subsection*{Construction of a hard instance}
\label{subsection:MDP_parameterization}
We construct a class of MDPs parameterized by $M(p_1, p_2)$ with horizon $H \geq 2$, action space $\mathcal{A} = \{b_i\}_{i = 1}^A$ (where $A \geq 2$), state space $\mathcal{S} = \{x_0, x_1, x_2\}$, initial state distribution $d_1(x_0) = 1$, transition kernels $\mathbb{P}_1(x_1 | x_0, p_i) = p_i$, $\mathbb{P}_1(x_2 | x_0, p_i) = 1 - p_i, \forall i \in [A]$ where $p_i := \min \{p_1, p_2\}, \forall i \geq 3$, $\mathbb{P}_h(x_1|x_1, a) = 1, \forall (h,a) \in [H] \times \mathcal{A}$, and reward functions $r_h(s_h,a_h) = 1\{s_h = x_1, h \geq 2\}$. It is not hard to see that the optimal action at the first stage is $b_{i^*}$ where $i^* = \argmax \{p_i: i \in \{1,2\}\}$ and the optimal action at any stage $h \geq 2$ is any action $a \in \mathcal{A}$. The diagram of $M(p_1, p_2)$ is depicted in Figure \ref{fig:hard_mdp}. It is easy to see that this MDP satisfies both the definition of a linear MDP (see Definition \ref{definition:linear_mdp}) and of a mixture linear MDP (see Definition \ref{definition: linear mixture mdp}). 
\begin{figure}[h!]
    \centering
    \includegraphics{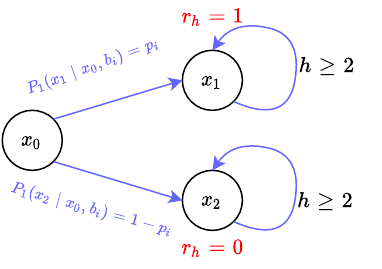}
    \caption{The diagram of the hard MDP introduced by \cite{jin2021pessimism}. }
    \label{fig:hard_mdp}
\end{figure}
By direct computation, we have: 
\begin{align*}
&\begin{cases}
    V^*_1(x_0) &= \max \{p_1, p_2 \} (H - 1), \\
    Q^*_1(x_0, b_i) &= p_i (H - 1), \forall i \in [A], \\
    Q^*_h(x_1, a) &= H - h, \forall (h, a) \in \{2, \ldots, H\} \times \mathcal{A}, \\ 
    Q^*_h(x_2, a) &= 0, \forall (h, a) \in \{2, \ldots, H\} \times \mathcal{A}.
\end{cases} \\
&\begin{cases}
    V^{\pi}_1(x_0) &= \sum_{i=1}^A \pi_1(b_i | x_0) p_i (H-1), \\
    Q^{\pi}_1(x_0, b_i) &= p_i (H - 1), \forall i \in [A], \\ 
    Q^{\pi}_h(x_1, a) &= H - h, \forall (h, a) \in \{2, \ldots, H\} \times \mathcal{A}, \\ 
    Q^{\pi}_h(x_2, a) &= 0, \forall (h, a) \in \{2, \ldots, H\} \times \mathcal{A}.   
\end{cases} \\ 
&\begin{cases}
    d^*_1(x_0) &= 1, \\ 
    d^*_h(x_1) &= \max \{p_1, p_2\}, \forall h \geq 2, \\
    d^*_h(x_2) &= 1 - \max \{p_1, p_2\}, \forall h \geq 2.
\end{cases} \\
&\begin{cases}
    d^{\pi}_1(x_0) &= 1, \\ 
    d^{\pi}_h(x_1) &= \sum_{i=1}^A \pi_1(b_i | x_0) p_i, \\ 
    d^{\pi}_h(x_2) &= \sum_{i=1}^A \pi_1(b_i|x_0) (1 - p_i).
\end{cases}
\end{align*}
Via direct computation, we have the minimum gap as $\Delta_{\min} = |p_1 - p_2| (H - 1)$. 

\subsection*{Proof of Theorem \ref{theorem:lower_bound}}
\begin{proof}[Proof of Theorem \ref{theorem:lower_bound}]
We consider two MDP instances $M_1 := M(p^*, p)$ and $M_2 := M(p, p^*)$ (the parameterization is defined in Subsection \ref{subsection:MDP_parameterization}) where $p^* > p$, whose optimal actions in the first stage are $b_1$ and $b_2$, respectively, where $b_1 \neq b_2$. The intuition for the hardness of these two instances is that any policy is sub-optimal in at least one of the instances.  We have: 
\begin{align*}
    \subopt(\pi; M_1) &= (H-1)(p^* - p)(1 - \pi_1(b_1|x_0)), \\ 
    \subopt(\pi; M_2) &= (H-1)(p^* - p)(1 - \pi_1(b_2|x_0)).
\end{align*}

Consider any policy $\pi = \textrm{Algo}(\mathcal{D})$ and let $a_1 \sim \pi_1(\cdot| x_0)$ (note that $a_1$ is a random variable). We have: 
\begin{align*}
    &2 \max_{l \in \{1,2\}} \mathbb{E}_{\mathcal{D} \sim M_l} \left[\subopt(\textrm{Algo}(\mathcal{D}); M_l) \right] \\
    &\geq \mathbb{E}_{\mathcal{D} \sim M_l} \left[ \subopt(\textrm{Algo}(\mathcal{D}); M_1) \right] + \mathbb{E}_{\mathcal{D} \sim M_2} \left[ \subopt(\textrm{Algo}(\mathcal{D}); M_2) \right] \\ 
    &= (H-1) (p^*- p)\left( \mathbb{E}_{\mathcal{D} \sim M_1} \left[ 1 - \pi_1(b_1|x_0) \right] + \mathbb{E}_{\mathcal{D} \sim M_2} \left[ 1 - \pi_1(b_2|x_0) \right] \right) \\ 
    &\geq (H-1) (p^*- p)\left( \mathbb{E}_{\mathcal{D} \sim M_1} \left[ 1 - \pi_1(b_1|x_0) \right] + \mathbb{E}_{\mathcal{D} \sim M_2} \left[ \pi_1(b_1|x_0) \right] \right) \\ 
    &= (H - 1) (p^* - p) \left( \mathbb{E}_{D \sim M_1} \left[ 1\{a_1 \neq b_1 \} \right] +  \mathbb{E}_{D \sim M_2} \left[ 1\{a_1 = b_1\} \right]  \right) \\
    &\geq (H - 1) (p^* - p)(1 - \textrm{TV}(P_{M_1}, P_{M_2})) \\ 
    &\geq (H -1) (p^* - p)(1 - \sqrt{\kl(P_{\mathcal{D} \sim M_1} \| P_{\mathcal{D} \sim M_2})/2}),
\end{align*}
where the third inequality is by the definition of the total variation distance $\textrm{TV}(P,Q) = \sup \{|P(B) - Q(B)|: \forall B \text{ is measurable}\}$, and the last inequality is by Donsker's inequality. 
\paragraph{Construction of behavior policy.} To construct the behaviour policy $\mu$ that satisfies $\sup_{h, s_h, a_h }\frac{d^{*}_h(s_h,a_h)}{d^{\mu}_h(s_h, a_h)} \leq \kappa_h, \forall h \in [H]$ in both $M(p^*, p)$ and $M(p, p^*)$, we consider $\mu_h(a | x_i) = \frac{1}{A}, \forall (h,a,i) \in \{2, \ldots, H\} \times \mathcal{A} \times \{1,2\}$. We also set $\mu_1(b_1 | x_0) = \mu_1(b_2|x_0) = q$, where $q \leq 1/2$ since $b_1 \neq b_2$. By direct computation, we have: 
\begin{align*}
\begin{cases}
     \max_{s_1, a_1} \frac{d^{M_i,*}_1(s_1, a_1)}{d^{M_i,\mu}_1(s_1, a_1)} &= \frac{1}{q} \\ 
  \max_{s_h, a_h} \frac{d^{M_i,*}_h(s_h, a_h)}{d^{M_i,\mu}_h(s_h, a_h)} &\leq \max \left\{ \frac{p^*}{ q(p^* + p)}, \frac{1 - p^*}{ q(2 - p^* - p) } \right\} = \frac{p^*}{ q(p^* + p)} \leq \frac{1}{q} 
\end{cases}
\end{align*}

As $\kappa_h \geq 2$, we can set $q = \frac{1}{\min_{h} \kappa_h} = \frac{1}{ \kappa_{\min}}\in (0, \frac{1}{2}]$, and thus we have $\sup_{h, s_h, a_h }\frac{d^{*}_h(s_h,a_h)}{d^{\mu}_h(s_h, a_h)} \leq \kappa_h, \forall h \in [H]$. 
\paragraph{Computing $\kl(\mathbb{P}_{\mathcal{D} \sim M_1} \| \mathbb{P}_{\mathcal{D} \sim M_2})$.} We consider dataset $\mathcal{D} = \{(s^t_h, a^t_h, r^t_h)\}_{h \in [H]}^{t \in [K]}$ such that $s^t_1 = x_0, \forall t \in [K]$ and $\mathcal{D}$ agrees with the class $\{M(p_1, p_2): p_1 \neq p_2\}$ and behavior policy $\mu$. Under $M(p_1, p_2)$, we have $r^t_1 = 0, \forall t \in [K]$ and $s^t_h = s^t_2, r^t_h = r^t_2, \forall t \in [K], \forall h \geq 2$. Thus, $\mathcal{D}' := \{(s^t_1, a^t_1, s^t_2, r^t_2)\}^{t \in [K]}$ is as informative as $\mathcal{D}$. \footnote{In the sense that knowing $\mathcal{D}'$ implies $\mathcal{D}$. } Thus, $KL(\mathbb{P}_{\mathcal{D} \sim M_1} \| \mathbb{P}_{\mathcal{D} \sim M_2}) = KL(\mathbb{P}_{\mathcal{D}' \sim M_1} \| \mathbb{P}_{\mathcal{D}' \sim M_2})$.

Let $n_l := \sum_{t=1}^K 1\{a^t_1 = b_l\}$ be the number of episodes in $\mathcal{D}'$ with the first-stage action $b_l$ and $Z_l = \{r^t_2: a^t_2 = b_l, t \in [K]\}$ be the set of the second-stage reward in  $\mathcal{D}'$ in such episodes. We have: 
\begin{align*}
    \mathbb{P}_{M_1}(\mathcal{D}') &= \mu_1^{n_1}(b_1 | x_0) (p^*)^{\sum_{z \in Z_1} z} (1 - p^*)^{n_1 - \sum_{z \in Z_1} z} \prod_{i \neq 1}  \mu_1^{n_i}(b_i | x_0) p^{\sum_{z \in Z_i} z} (1 - p)^{n_i - \sum_{z \in Z_i} z}, \\ 
    \mathbb{P}_{M_2}(\mathcal{D}') &= \mu_1^{n_1}(b_2 | x_0) (p^*)^{\sum_{z \in Z_2} z} (1 - p^*)^{n_2 - \sum_{z \in Z_2} z} \prod_{i \neq 2}  \mu_1^{n_i}(b_i | x_0) p^{\sum_{z \in Z_i} z} (1 - p)^{n_i - \sum_{z \in Z_i} z}.
\end{align*}
Thus, we have: 
\begin{align*}
    \log \frac{\mathbb{P}_{M_1}(\mathcal{D}')}{\mathbb{P}_{M_2}(\mathcal{D}')} = \delta_{1,2} \log \frac{p^* (1 - p)}{p (1 - p^*)} + (n_1 - n_2) \log \frac{1 - p^*}{1 - p},
\end{align*}
where $\delta_{1,2} := \sum_{z \in Z_1} z - \sum_{z \in Z_2} z$. Thus, we have: 
\begin{align*}
    KL(\mathbb{P}_{\mathcal{D}' \sim M_1} \| \mathbb{P}_{\mathcal{D}' \sim M_2}) &= \mathbb{E}_{\mathcal{D}' \sim M_1} \left[  \log \frac{\mathbb{P}_{M_1}(\mathcal{D}')}{\mathbb{P}_{M_2}(\mathcal{D}')} \right] \\ 
    &= \mathbb{E}_{\mathcal{D}' \sim M_1} [\delta_{1,2}] \log \frac{p^* (1 - p)}{p (1 - p^*)} + \mathbb{E}_{\mathcal{D}' \sim M_1} [n_1 - n_2] \log \frac{1 - p^*}{1 - p} \\ 
    &= K q (p^* - p) \log \frac{p^* (1 - p)}{p (1 - p^*)} \\ 
    &= K q (p^* - p) \log \left(1 + \frac{p^*-p}{p(1 - p^*)} \right). 
\end{align*} 
\paragraph{Construction of $(p^*, p)$.}
Now we choose $p, p^* \in [\frac{1}{4}, \frac{3}{4}]$ and $p^* - p \leq \frac{1}{16}$. For such $p^*, p$, we have $ \log \left(1 + \frac{p^*-p}{p(1 - p^*)} \right) \leq \log (1 + 16(p^* - p)) \leq 16 (p^* - p)$, where the last inequality is by that $\log(1 + x) \leq x, \forall x \in (0,1]$. Thus, we have $\kl(\mathbb{P}_{\mathcal{D}' \sim M_1} \| \mathbb{P}_{\mathcal{D}' \sim M_2}) \leq 16 q K (p^* - p)^2$. Hence, we have: 
\begin{align}
    2 \max_{l \in \{1,2\}} \mathbb{E}_{\mathcal{D} \sim M_l} \left[\subopt(\textrm{Algo}(\mathcal{D}); M_l) \right] &\geq (H -1) (p^* - p)(1 - \sqrt{\kl(\mathbb{P}_{\mathcal{D} \sim M_1} \| \mathbb{P}_{\mathcal{D} \sim M_2})/2}) \nonumber \\ 
    &\geq (H -1) (p^* - p) (1 - 2 (p^* - p) \sqrt{2 q K}). 
    \label{eq: lower bound of maximum of subopt}
\end{align}
Recall that we choose $q = 1/ \kappa_{\min}$. Now, we simply set: 
\begin{align*}
    p^* - p = \frac{1}{4 \sqrt{2}} \sqrt{\frac{\kappa_{\min}}{K}}. 
\end{align*}
Then, we have $1 - 2 (p^* - p) \sqrt{2 q K} = \frac{1}{2}$.

\paragraph{Minimax lower bound.} By Eq. (\ref{eq: lower bound of maximum of subopt}), with the parameter choice above, we have 
\begin{align*}
    2 \max_{l \in \{1,2\}} \mathbb{E}_{\mathcal{D} \sim M_l} \left[\subopt(\textrm{Algo}(\mathcal{D}); M_l) \right] = \Omega\left( H \sqrt{\frac{\kappa_{\min}}{K}} \right). 
\end{align*}

\paragraph{Gap-dependent lower bound.} To incorporate the gap information into the lower bound, note that $\Delta_{\min} = (p^* - p) (H - 1) = \frac{H - 1}{4 \sqrt{2}}  \sqrt{\frac{\kappa_{\min}}{K}}$. Thus 
\begin{align*}
    \Delta_{\min} = \frac{1}{32} \frac{(H - 1)^2 \kappa_{\min}}{K \Delta_{\min}}. 
\end{align*}
Therefore, by Eq. (\ref{eq: lower bound of maximum of subopt}), we have: 
\begin{align*}
    \max_{l \in \{1,2\}} \mathbb{E}_{\mathcal{D} \sim M_l} \left[\subopt(\textrm{Algo}(\mathcal{D}); M_l) \right] &\geq \frac{1}{2} (H -1) (p^* - p) (1 - 2 (p^* - p) \sqrt{2 q K}) \\
    &= 2^{-7} \frac{(H-1)^2 \kappa_{\min}}{ K \Delta_{\min}} = \Omega \left( \frac{ H^2 \kappa_{\min}}{K \Delta_{\min}} \right). 
\end{align*}
\end{proof}

\section{Auxiliary Lemmas}

\subsection*{MDPs}
\begin{lemma}[Extended Value Difference {\citep[Section~B.1]{cai2020provably}}]
Let $\pi = \{\pi_h\}_{h=1}^H$ and $\pi' = \{\pi'_h\}_{h=1}^H$ be two arbitrary policies and let $\{Q_h\}_{h=1}^H$ be arbitrary functions $\mathcal{S} \times \mathcal{A} \rightarrow \mathbb{R}$. Let $V_h := \langle Q_h, \pi_h \rangle$. Then $\forall s \in \mathcal{S}$, $\forall h \in [H]$,  
\begin{align*}
    V_h(s) - V_h^{\pi'}(s) &= \sum_{i=h}^H \mathbb{E}_{\pi'} \left[ \langle Q_i(s_i, \cdot), \pi_i(\cdot|s_i) - \pi'_i(\cdot|s_i) \rangle | s_h = s\right] \\ 
    &+ \sum_{i=h}^H \mathbb{E}_{\pi'} \left[ Q_i(s_i, a_i) - T_i V_{i+1} (s_i, a_i) | s_h = s\right], 
\end{align*}
where $T_iV := r_i + P_i V$ and $\mathbb{E}_{\pi'}$ is the expectation over the randomness of $(s_h, a_h, \ldots, s_H, a_H)$ induced by $\pi'$. 
\label{lemma:EVD}
\end{lemma}

\begin{proof}[Proof of Lemma \ref{lemma:EVD}]
Fix $h \in [H]$. Denote $\xi_i:= Q_i - T_i V_{i+1}$. For $\forall i \in [h, H-1]$, $\forall s_h \in \mathcal{S}$, we have
\begin{align*}
    \mathbb{E}_{\pi'} \left[V_i(s_i) - V_i^{\pi'}(s_i)|s_h \right] 
    &=  \mathbb{E}_{\pi'} \left[ \langle Q_i(s_i, \cdot), \pi_i(\cdot|s_i) \rangle - \langle Q_i^{\pi'}(s_i, \cdot), \pi'_i(\cdot|s_i) \rangle | s_h \right] \nonumber \\ 
    &= \mathbb{E}_{\pi'} \left[ \langle Q_i(s_i, \cdot), \pi_i(\cdot|s_i) - \pi'_i(\cdot|s_i) \rangle + \langle Q_i(s_i, \cdot) - Q_i^{\pi'}(s_i, \cdot), \pi'_i(\cdot|s_i) \rangle | s_h \right] \nonumber \\ 
    &= \mathbb{E}_{\pi'} \left[ \langle Q_i(s_i, \cdot), \pi_i(\cdot|s_i) - \pi'_i(\cdot|s_i) \rangle | s_h \right] \nonumber \\
    &+ \mathbb{E}_{\pi'} \left[ \langle \xi_i(s_i, \cdot) + T_i V_{i+1}(s_i, \cdot) - (r_i(s_i, \cdot) + P_i V^{\pi'}_{i+1}(s_i, \cdot)), \pi'_i(\cdot|s_i) \rangle | s_h \right] \nonumber \\ 
    &=\mathbb{E}_{\pi'} \left[ \langle Q_i(s_i, \cdot), \pi_i(\cdot|s_i) - \pi'_i(\cdot|s_i) \rangle |s_h \right] + \mathbb{E}_{\pi'} \left[ \xi_i(s_i, a_i) | s_h\right] \nonumber \\
    &+ \mathbb{E}_{\pi'} \left[  P_i (V_{i+1} - V^{\pi'}_{i+1})(s_i, a_i) |s_h \right] \nonumber \\
    &= \mathbb{E}_{\pi'} \left[ \langle Q_i(s_i, \cdot), \pi_i(\cdot|s_i) - \pi'_i(\cdot|s_i) \rangle |s_h \right] + \mathbb{E}_{\pi'} \left[ \xi_i(s_i, a_i) | s_h\right] \nonumber \\
    &+ \mathbb{E}_{\pi'} \left[V_{i+1}(s_{i+1}) - V_{i+1}^{\pi'}(s_{i+1})|s_h \right]. 
    % \label{eq:telescope_evd}
\end{align*}
Taking $\sum_{i=h}^H$ both sides of the last equation above completes the proof. 
% Note that 
% \begin{align*}
%     \mathbb{E}_{\pi'} \left[ \langle P_i (V_{i+1} - V^{\pi'}_{i+1})(s_i, \cdot), \pi'_i(\cdot|s_i) \rangle |s_h \right] &= \mathbb{E}_{\pi'} \left[ P_i (V_{i+1} - V^{\pi'}_{i+1})(s_i, a_i) | s_h \right] \\ 
%     &= \mathbb{E}_{\pi'} \left[ V_{i+1}(s_{i+1}) - V^{\pi'}_{i+1}(s_{i+1}) | s_h \right]. 
% \end{align*}
% Thus, taking \mathbb{E}_{\pi'} both sides of Eq. (\ref{eq:telescope_evd}) and recursively apply 

% recursively applying Eq. (\ref{eq:telescope_evd}) for all $i \in [h, H-1]$, taking $$and summing them up forms a telescoping sum, which reduces into Lemma \ref{lemma:EVD}. 
\end{proof}

\begin{lemma}
Let $\hat{\pi} = \{\hat{\pi}_h\}_{h=1}^H$ and $\hat{Q}_h(\cdot, \cdot)$ be arbitrary policy and $Q$-function. Let $\hat{V}_h(s) = \langle \hat{Q}_h(s,\cdot), \hat{\pi}_h(\cdot|s) \rangle$ and $\zeta_h(s,a) := (T_h \hat{V}_{h+1})(s,a) - \hat{Q}_h(s,a)$. For any policy $\pi$ and $h \in [H]$, we have 
\begin{align*}
    V_h^{\pi}(s) - V^{\hat{\pi}}_h(s) &= \sum_{i=h}^H \mathbb{E}_{\pi} \left[ \zeta_i(s_i, a_i) |s_h = s \right] - \sum_{i=h}^H \mathbb{E}_{\hat{\pi}} \left[ \zeta_i(s_i, a_i) |s_h = s \right] \\
    &+ \sum_{i=h}^H \mathbb{E}_{\pi} \left[ \langle \hat{Q}_h(s_h, \cdot), \pi_h(\cdot|s_h) - \hat{\pi}_h(\cdot|s_h) \rangle | s_h = s \right].
\end{align*}
\label{lemma:regret_decomposition}
\end{lemma}

\begin{proof}
We apply Lemma \ref{lemma:EVD} with $\pi = \hat{\pi}$, $\pi' = \hat{\pi}$, $Q_h = \hat{Q}_h$ and apply Lemma \ref{lemma:EVD} again with $\pi = \hat{\pi}$, $\pi' = \pi$, $Q_h = \hat{Q}_h$ and take the difference between two results to complete the proof. 
\end{proof}

\begin{lemma}
For any $0 \leq V(\cdot) \leq H$, there exists a $w_h \in \mathbb{R}^d$ such that $T_h V = \langle \phi_h, w_h \rangle$ and $\| w_h \|_2 \leq 2 H \sqrt{d}$. In addition, for any policy $\pi \in \Pi$, $\exists w_h^{\pi} \in \mathbb{R}^d$ s.t. $Q^{\pi}_h(s,a) = \phi_h(s,a)^T w^{\pi}_h$ with $\| w_h^{\pi} \|_2 \leq 2(H-h+1) \sqrt{d}$.  
\label{lemma:linear_bellman}
\end{lemma}

\begin{proof}
By definition, 
\begin{align*}
    T_h V = r_h + P_h V = \langle \phi, \theta_h \rangle + \langle \phi, \int_{\mathcal{S}} V(s) d \nu_h(s) \rangle = \langle \phi, w_h \rangle,
\end{align*}
where $w_h = \theta_h + \int_{\mathcal{S}} V(s) d \nu_h(s)$. By the assumption of linear MDP,
\begin{align*}
    \|w_h \|_2 &= \|\theta_h + \int_{\mathcal{S}} V(s) d \nu_h(s) \|_2 \leq \| \theta_h \|_2 + \| \int_{\mathcal{S}} V(s) d \nu_h(s) \|_2 \leq \sqrt{d} + H \sqrt{d} \leq 2H \sqrt{d}. 
\end{align*}
The second part is similar with $V^{\pi}_h \leq H - h + 1$. 
\end{proof}

\begin{lemma}[Bound on weights in algorithm]
For any $(k,h) \in [K] \times [H]$, the weight $\hat{w}^k_h$ in Algorithm \ref{alg:bpvi} satisfies: 
% \raman{Seems like a reference to the streamed PEVI algorithm which has since been removed. Fix!}
\begin{align*}
    \|\hat{w}^k_h \|_2 \leq (H - h + 1) \sqrt{dk/\lambda}. 
\end{align*}
\label{lemma:bound_on_weights_algo}
\end{lemma}
\begin{proof}
For any $v \in \mathbb{R}^d$, we have 
\begin{align*}
    | v^T \hat{w}^k_h | &= \left| v^T (\Sigma_h^{k})^{-1} \sum_{t=1}^k \phi_h(s_h^t, a_h^t)(r_h^t + \hat{V}^k_{h+1}(s^t_{h+1}))\right| \leq (H-h+1) \sum_{t=1}^k \left| v^T (\Sigma_h^{k})^{-1} \phi_h(s_h^t, a_h^t) \right| \\ 
    &\leq (H - h +1)  \| v \|_{(\Sigma_h^{k})^{-1}}  \sum_{t=1}^k  \| \phi_h(s_h^t, a_h^t) \|_{(\Sigma_h^{k})^{-1}} \\
    &\leq (H - h +1)  \| v \|_{(\Sigma_h^{k})^{-1}}  \sqrt{k\sum_{t=1}^k  \| \phi_h(s_h^t, a_h^t) \|^2_{(\Sigma_h^{k})^{-1}} } \leq (H - h +1) \| v \|_2  \sqrt{ \| (\Sigma_h^{k})^{-1} \| } \cdot \sqrt{kd} \\
    &\leq (H - h + 1) \sqrt{kd/\lambda} \cdot \| v \|_2,
\end{align*}
where the penultimate inequality is due to that
\begin{align*}
    \sum_{t=1}^k  \| \phi_h(s_h^t, a_h^t) \|^2_{(\Sigma_h^{k})^{-1}} 
    &= \sum_{t=1}^k \textrm{tr} \left( \phi_h(s_h^t, a_h^t) ^T (\Sigma_h^{k})^{-1} \phi_h(s_h^t, a_h^t) \right) \\
    &= \sum_{t=1}^k \textrm{tr} \left(  (\Sigma_h^{k})^{-1} \phi_h(s_h^t, a_h^t) \phi_h(s_h^t, a_h^t) ^T \right) = \sum_{t=1}^k \frac{\lambda_i}{\lambda_i + \lambda} \leq d 
\end{align*}
with $\{\lambda_i\}_{i=1}^d$ being the eigenvalues of  $\phi_h(s_h^t, a_h^t) \phi_h(s_h^t, a_h^t) ^T$. Finally, using $\| \hat{w}^k_h \|_2 = \max_{v: \|v\|_2 = 1} |v^T \hat{w}^k_h|$ completes the proof.
\end{proof}

\begin{lemma}[\cite{jin2020provably}]
Let $\mathcal{V}(L, B, \lambda) \subset \{\mathcal{S} \rightarrow \mathbb{R}\}$ be a class of functions with the following parametric form:
\begin{align*}
    V(\cdot) = \min\{ \max_{a \in \mathcal{A}} \phi_h(\cdot, a)^T w - \beta \| \phi(\cdot, a) \|_{\Sigma^{-1}}, H - h +1 \}^+, 
\end{align*}
where the parameters $(w, \beta, \Sigma)$ satisfy: $\|w\|_2 \leq L, \beta \in [0,B], \lambda_{\min}(\Sigma) \geq \lambda$. Assume $\|\phi_h(s,a)\|_2 \leq 1, \forall (s,a)$. Let $N_{\epsilon}$ be the $\epsilon$-covering number of $\mathcal{V}(L, B, \lambda)$ with respect to the maximal norm $\| \cdot\|_{\infty}$. We have: 
\begin{align*}
    \log N_{\epsilon} \leq d \log(1 + 4L/\epsilon) + d^2 \log(1 + 8 \sqrt{d} B^2 /(\lambda \epsilon^2)). 
\end{align*}
\label{lemma:covering}
\end{lemma}

\subsection*{Linear features}
\begin{lemma}
Let $\phi: \mathcal{S} \rightarrow \mathbb{R}^d$ be any feature and $p$ be any density with respect to the Lebesgue measure on $\mathcal{S}$. Let $A = \mathbb{E}_{s \sim p(s)} \left[ \phi(s) \phi(s)^T \right]$. We have 
\begin{align*}
    \textrm{col}(A) = \textrm{span}(\{\phi(s): \forall s \in \mathcal{S} \text{ s.t. } p(s) > 0\}),
\end{align*}
where $\textrm{col}(A)$ denotes the column space of $A$. 
\end{lemma}
\begin{proof}
Let $B := \textrm{span}(\{\phi(s): \forall s \in \mathcal{S} \text{ s.t. } p(s) > 0\})$. We need to prove that $A = B$. First, as the $i$-th column of $A$, $\textrm{col}_i(A) = \int_{\{s : p(s) > 0\}} \phi(s) \phi(s)_i dp(s) \in B$, where $\phi(s)_i$ denotes the $i$-component of $\phi(s) \in \mathbb{R}^d$. Thus, $\textrm{col}(A) \subseteq B$. Now we prove that $B \subseteq \textrm{col}(A)$. 

For any $x \in \mathbb{R}^d$, we have $x^T A x = \int_{\{s: p(s) > 0\}} (x^T \phi(s))^2 p(s) ds $. Thus, for any $x \in \textrm{null}(A)$ (i.e. $Ax = 0$), we have $x^T \phi(s) = 0, \forall s$ such that $p(s) > 0$. Hence, $\textrm{null}(A) \perp B$. But we have $\textrm{col}(A) = \textrm{null}(A)^{\perp}$, thus $B \subseteq \textrm{col}(A)$.  
\end{proof}

\begin{lemma}[\cite{DBLP:journals/corr/abs-2104-03781}]
For any symmetric p.s.d. matrix $A \in \mathbb{R}^{d \times d}$ with $\|A\| > 0$, the smallest positive eigenvalue of $A$ is:
\begin{align*}
    \lambda^+_{\min}(A) = \min_{x \in \textrm{col}(A): \| x\|_2 = 1} x^T A x. 
\end{align*}
\label{lemma:smallest_positive_eigenvalue_via_column_space}
\end{lemma}
\begin{proof}
Let $0 \leq \lambda_1 \leq \ldots \leq \lambda_d$ be the eigenvalues of $A$ with corresponding orthonormal eigenvectors $u_1, \ldots, u_d$. By the iterative representation of eigenvalues, we have: 
\begin{align*}
    \lambda_i = \min_{x \in  \{u_1, \ldots, u_{i-1}\} ^{\perp}: \|x\|_2 = 1} x^T A x. 
\end{align*}
Let $d' = \min \{i \in [d]: \lambda_i > 0\}$. As $\| A \| = \lambda_d > 0$, such $d'$ exists. Thus, $\textrm{span}(\{u_1, \ldots, u_{d'-1}\}) = \textrm{null}(A)$. Note that $\textrm{null}(A)^T = \textrm{col}(A)$ as $A$ is symmetric, we complete the proof.
\end{proof}

\begin{lemma}
Let $A \in \mathbb{R}^{d \times d}$ be a symmetric matrix with non-zero eigenvalues $\lambda_1 \leq \lambda_2 \ldots \leq \lambda_d$ and corresponding orthonormal eigenvectors $u_1, \ldots, u_d$. Assume $\| A \| > 0$ (i.e. $\lambda_d > 0$). Let $d' = \min_{i \in [d]: \lambda_i > 0}$. We have:
\begin{align*}
    \forall x \in \textrm{col}(A), \| x \|_{A^{-1}} \cdot \| x \|_{A} \leq \frac{\lambda_d}{\lambda_{d'}}. 
\end{align*}
\label{lemma:Kantovich}
\end{lemma}

\subsection*{Concentration inequalities}
\begin{lemma}[Hoeffding-Azuma inequality]
Suppose $\{X_k\}_{k=0}^{\infty}$ is a martingale, i.e. $\mathbb{E}[|X_k|] < \infty$ and $\mathbb{E}\left[ X_{k+1}| X_k, \ldots, 
X_0\right] = X_k, \forall k$, and suppose that $\forall k, |X_{k} - X_{k-1}| \leq c_k$ almost surely. Then for all positive $n$ and $t$, we have: 
\begin{align*}
    \mathbb{P}\left( |X_n - X_0| \geq t \right) \leq 2 \exp \left( \frac{-t^2}{\sum_{i=1}^n c_i^2} \right).
\end{align*}
\label{lemma:azuma0}
\end{lemma}

\begin{lemma}[A variant of Hoeffding-Azuma inequality]
Suppose $\{Z_k\}_{k = 0}^{\infty}$ is a real-valued stochastic process with corresponding filtration $\{\mathcal{F}_{k}\}_{k=0}^{\infty}$, i.e. $\forall k $, $Z_k$ is $\mathcal{F}_k$-measurable. Suppose that for any $k$, $\mathbb{E}[|Z_k|] < \infty$ and $|Z_k - \mathbb{E} \left[ Z_k | \mathcal{F}_{k-1} \right]| \leq c_k$ almost surely. Then for all positive $n$ and $t$, we have: 
\begin{align*}
    \mathbb{P}\left( \bigg|\sum_{k=1}^n Z_k  -  \sum_{k=1}^n  \mathbb{E} \left[ Z_k | \mathcal{F}_{k-1} \right] \bigg| \geq t \right) \leq 2 \exp \left( \frac{-t^2}{\sum_{i=1}^n c_i^2} \right).
\end{align*}
\label{lemma:azuma}
\end{lemma}
\begin{proof}
This lemma is a direct application of Lemma \ref{lemma:azuma0} with $X_k = \sum_{i=1}^k (Z_i - \mathbb{E}\left[ Z_i | F_{i-1} \right])$. 
\end{proof}

\begin{lemma}[Concentration of self-normalized processes \citep{NIPS2011_e1d5be1c} ]
Let $\{\eta_t\}_{t=1}^{\infty}$ be a real-valued stochastic process with corresponding filtration $\{\mathcal{F}_t\}_{t = 0}^{\infty}$ (i.e. $\eta_t$ is $\mathcal{F}_t$-measurable). Assume that $\eta_t | \mathcal{F}_{t-1}$ is zero-mean and $R$-subGaussian, i.e., $ \mathbb{E}\left[\eta_t | \mathcal{F}_{t-1} \right] = 0$, and
\begin{align*}
    \forall \lambda \in \mathbb{R}, \mathbb{E}\left[e^{\lambda\eta_t} | \mathcal{F}_{t-1} \right] \leq e^{\lambda^2 R^2/ 2}. 
\end{align*}
Let $\{x_t\}_{t=1}^{\infty}$ be an $\mathbb{R}^d$-valued stochastic process where $x_t$ is $\mathcal{F}_{t-1}$-measurable and $\|x_t \| \leq L$. Let $\Sigma_k = \lambda I_d + \sum_{t=1}^k x_t x_t^T$. Then for any $\delta > 0$, with probability at least $1 - \delta$, it holds for all $k > 0$ that 
\begin{align*}
    \left\| \sum_{t=1}^k x_t \eta_t \right \|_{\Sigma_k^{-1}}^2 \leq 2 R^2 \log \left[ \frac{\textrm{det}(\Sigma_k)^{1/2} \textrm{det}(\Sigma_0)^{-1/2}}{\delta} \right] \leq 2 R^2 \left[ \frac{d}{2} \log \frac{kL^2 + \lambda}{\lambda} + \log \frac{1}{\delta} \right].
    % 4 R^2 d \log\left(\frac{\lambda + kL}{\lambda \delta} \right). 
\end{align*}
\label{lemma:self_normalized}
\end{lemma}

\begin{lemma}[Uniform concentration of self-normalized processes \citep{jin2020provably}]
Let $\{s_t\}_{t=1}^{\infty}$ be a stochastic process on state space $\mathcal{S}$ with corresponding filtration $\{\mathcal{F}_t\}_{t = 0}^{\infty}$ (i.e. $s_t$ is $\mathcal{F}_t$-measurable). Let $\{\phi_t\}_{t=1}^{\infty}$ be an $\mathbb{R}^d$-valued stochastic process where $\phi_t$ is $\mathcal{F}_{t-1}$-measurable and $\| \phi_t\|_2 \leq 1$. Let $\Sigma_k = \lambda I_d + \sum_{t=1}^{k-1} \phi_t \phi_t^T$. Then for any $\delta > 0$, with probability at least $ 1- \delta$, for all $k \geq 0$ and any $V \in \mathcal{V} \subset \{\mathcal{S} \rightarrow [0,H]\}$, we have
\begin{align*}
    \left \| \sum_{t=1}^{k-1} \phi_t \left( V(s_t) - \mathbb{E}\left[ V(s_t) | \mathcal{F}_{t-1} \right] \right)\right \|^2_{\Sigma_k^{-1}} \leq 4H^2 \left[ \frac{d}{2} \log \left( \frac{k + \lambda}{\lambda} \right) + \log \frac{N_{\epsilon}}{\delta} \right] + \frac{8k^2 \epsilon^2}{ \lambda},
\end{align*}
where $N_{\epsilon}$ is the $\epsilon$-covering number of $\mathcal{V}$ with respect to the distance $\|\cdot\|_{\infty}$. 
\label{lemma:uniform_self_normalized}
\end{lemma}
\begin{proof}
For any $V \in \mathcal{V}$, there exists $\bar{V}$ in the $\epsilon$-covering such that $V = \bar{V} + \Delta_V$ where $\sup_s|\Delta_V(s)| \leq \epsilon$. We have the following decomposition:
\begin{align*}
    \left \| \sum_{t=1}^k \phi_t \left( V(s_t) - \mathbb{E}\left[ V(s_t) | \mathcal{F}_{t-1} \right] \right)\right \|^2_{\Sigma_k^{-1}} &\leq 2\left \| \sum_{t=1}^k \phi_t \left( \bar{V}(s_t) - \mathbb{E}\left[ \bar{V}(s_t) | \mathcal{F}_{t-1} \right] \right)\right \|^2_{\Sigma_k^{-1}} \\
    &+ 2 \left \| \sum_{t=1}^k \phi_t \left( \Delta_V(s_t) - \mathbb{E}\left[ \Delta_V(s_t) | \mathcal{F}_{t-1} \right] \right)\right \|^2_{\Sigma_k^{-1}}
\end{align*}
where the first term can be bounded by Lemma \ref{lemma:self_normalized} and the second term is bounded by $8 k^2 \epsilon^2 / \lambda$. Then using the union bound over the $\epsilon$-covering completes the proof.
\end{proof}

\begin{lemma}[Freedman's inequality \citep{tropp2011freedman}]
Let $\{X_k\}_{k=1}^n$ be a real-valued martingale difference sequence with the corresponding filtration $\{\mathcal{F}_k\}_{k=1}^n$, i.e. $X_k$ is $\mathcal{F}_{k}$-measurable and $\mathbb{E}[X_k | \mathcal{F}_{k-1}] = 0$. Suppose for any $k$, $|X_k| \leq M$ almost surely and define $V:= \sum_{k=1}^n \mathbb{E}\left[ X_k^2 | \mathcal{F}_{k-1} \right]$. For any $a,b > 0$, we have:
\begin{align*}
    \mathbb{P}\left( \sum_{k=1}^n X_k \geq a, V \leq b \right) \leq \exp \left( \frac{-a^2}{2b + 2 a M/3} \right). 
\end{align*}
In an alternative form, for any $t > 0$, we have: 
\begin{align*}
    \mathbb{P}\left( \sum_{k=1}^n X_k \geq  \frac{2Mt}{3} + \sqrt{2bt}, V \leq b \right) \leq e^{-t} . 
\end{align*}
\label{lemma:freedman}
\end{lemma}

\begin{lemma}[Matrix Freedman's inequality \citep{tropp2011freedman}] Let $\{X_k\}$ be a $d \times d$ stochastic matrices adapted to the filtration $\{F_k\}$, i.e. $X_k$ is $\mathcal{F}_k$-measurable. Suppose that $\forall k, \| X_k - \mathbb{E}\left[ X_k | \mathcal{F}_{k-1} \right] \| \leq M$ almost surely for some $M > 0$. Define the quadratic variation process
\begin{align*}
    V_k := \sum_{i=1}^k \textrm{Var} \left[ X_i | \mathcal{F}_{i-1} \right]. 
\end{align*}
For any $a,b \geq 0$, we have: 
\begin{align*}
    \mathbb{P} \left( \exists k \geq 0: \| \sum_{i=1}^k X_i - \mathbb{E}\left[ X_i | \mathcal{F}_{i-1} \right] \| \geq a, \| V_k \|_2 \leq b \right) \leq d \exp \left( \frac{-a^2}{2b + 2 a M/3} \right). 
\end{align*}
In an alternative form, for any $t > 0$, we have:
\begin{align*}
    \mathbb{P} \left( \exists k \geq 0: \left\| \sum_{i=1}^k X_i - \mathbb{E}\left[ X_i | \mathcal{F}_{i-1} \right] \right\| \geq \frac{2Mt}{3} + \sqrt{2bt}, \| V_k \|_2 \leq b \right) \leq d e^{-t}. 
\end{align*}
\label{lemma:matrix_freedman}
\end{lemma}
\section{Linear Mixture Models}
\label{model-based offline RL}
In this section, we consider the linear mixture MDP model \citep{ayoub2020model} that assumes that the unknown transition function is an unknown linear mixture of several basic known probabilities. 

\begin{defn}[Linear mixture MDP]
An MDP $\mathcal{M}(\mathcal{S}, \mathcal{A}, H, \{r_h\}_{h=1}^H \{\mathbb{P}_h\}_{h=1}^H)$ is said to be a linear mixture MDP if there is a known feature mapping $\phi(s'|s,a): \mathcal{S} \times \mathcal{A} \times \mathcal{S} \rightarrow \mathbb{R}^d$ and an unknown vector $w^*_h \in \mathbb{R}^d$ with $\| w^*_h \|_2 \leq C_w$ such that $\mathbb{P}_h(s'|s,a) = \langle \phi(s'|s,a), w^*_h \rangle$ for all $(s,a,s',h)$ and $r_h$ is deterministic and known (for simplicity). Moreover, for any bounded function $V: \mathcal{S} \rightarrow [0,1]$, we have $\| \phi_V(s,a) \|_2 \leq 1$ for any $(s,a)$, where $\phi_V(s,a) = \sum_{s' \in \mathcal{S}} \phi(s'|s,a) V(s') \in \mathbb{R}^d$. 
\label{definition: linear mixture mdp}
\end{defn}

We consider the bootstrapped, constrained and pessimistic variant of Value-Targeted Regression \citep{ayoub2020model} which is shown in Algorithm \label{algorithm:bcpvtr}. The algorithm is very similar to Algorithm \ref{alg:bpvi} except that we compute $\hat{w}^k_h$ by solving the following regularized least-square regression in Line \ref{bcpvtr:least_square}:  
\begin{align*}
    \hat{w}^k_h \leftarrow \argmin_{w \in \mathbb{R}^d} \lambda \|w \|_2^2 + \sum_{i=1}^{k-1} \left(\phi_{\hat{V}^k_{h+1}}(s_h^t, a_h^t) ^T w - \hat{V}^i_{h+1}(s^i_{h+1}) \right)^2. 
\end{align*}

\begin{algorithm}
\caption{Bootstrapped and Constrained Pessimistic Value-Targeted Regression (BCP-VTR) }
\begin{algorithmic}[1]
\State \textbf{Input:} Dataset $\mathcal{D} = \{(s^t_h, a^t_h, r^t_h)\}_{h \in [H]}^{t \in [K]}$, uncertainty parameters $\{\beta_k\}_{k \in [K]}$, regularization hyperparameter $\lambda$, \textcolor{red}{$\mu$-supported policy class $\{\Pi_h(\mu)\}_{h \in [H]}$}. 
% \State Set $\Sigma_h^0$
\For{ \textcolor{red}{$k = 1, \ldots, K + 1$}}
    \State $\hat{V}_{H+1}^k(\cdot) \leftarrow 0$.
    \For{step $h = H, H-1, ..., 1$}
        \State $\Sigma_h^k \leftarrow \sum_{t=1}^{k-1} \phi_{\hat{V}^k_{h+1}}(s_h^t, a_h^t) \cdot \phi_{\hat{V}^k_{h+1}}(s_h^t, a_h^t)^T  + \lambda \cdot I$. 
        \State $\hat{w}_h^k \leftarrow (\Sigma_h^k)^{-1} \sum_{t=1}^{k-1} \phi_{\hat{V}^k_{h+1}}(s_h^t, a_h^t) \cdot  \hat{V}_{h+1}^k(s^t_{h+1})$. 
        \label{bcpvtr:least_square}
        \State $b_h^k(\cdot,\cdot) \leftarrow \beta_k \cdot  \| \phi_{\hat{V}^k_{h+1}}(\cdot, \cdot) \|_{(\Sigma_h^k)^{-1}}$. 
        \label{bpvi:lcb}
        \State $\bar{Q}_h^k(\cdot, \cdot) \leftarrow \langle \phi_{\hat{V}^k_{h+1}}(\cdot, \cdot), \hat{w}_h^k \rangle - b_h^k(\cdot, \cdot)$. 
        \State $\hat{Q}_h^k(\cdot, \cdot) \leftarrow \min\{\bar{Q}_h^k(\cdot, \cdot), H - h +1\}^+$. 
        \State $\hat{\pi}_h^k \leftarrow \displaystyle \textcolor{red}{ \argmax_{\pi_h \in \Pi_h(\mu)} \langle \hat{Q}_h^k, \pi_h \rangle}$
        % $\Pi_h(\mu)$ 
        % defined in Eq. (\ref{eq:constrained_policy_class})
        \label{bpvi:greedy}
        \State $\hat{V}_h^k(\cdot) \leftarrow \langle \hat{Q}_h^k(\cdot, \cdot), \pi_h^k(\cdot|\cdot) \rangle$.
    \EndFor
\EndFor
% \Ensure $\theta^{(J)}$
\State \textbf{Output:} Ensemble $\{\hat{\pi}^k: k \in [K + 1]\}$. 
\label{bcpvtr:ensemble}
\end{algorithmic}
\label{alg:bcpvtr}
\end{algorithm}

The flow of the results is very similar to the case of BCPVI except some minor modifications to reflect the changes from model-free methods to model-based methods. Here we only present the results that are different from their counterpart in BCPVI. 

\begin{lemma}
In Algorithm \ref{alg:bcpvtr}, if we choose 
\begin{align*}
    \beta_k = H \sqrt{d \log \frac{H + k H^3/\lambda}{\delta}} + \sqrt{\lambda} C_{w}
\end{align*}
 then with probability at least $1 - \delta$: 
\begin{align*}
    \forall (k,h,s,a) \in [K] \times [H] \times \mathcal{S} \times \mathcal{A}, |(T_h \hat{V}_{h+1}^k)(s,a) - (\hat{T}^k_h \hat{V}_{h+1}^k)(s,a) | \leq \beta_k \cdot \| \phi_{\hat{V}^k_{h+1}}(s,a) \|_{(\Sigma_h^{k})^{-1}}. 
\end{align*}
\label{lemma:uncertainty_quantifier}
\end{lemma}

\begin{proof}
We have:
\begin{align*}
    (T_h \hat{V}^k_{h+1})(s,a) &= r_h(s,a) + \langle \phi^k_{\hat{V}^k_{h+1}}(s,a), w^*_h \rangle, \\ 
     (\hat{T}_h \hat{V}^k_{h+1})(s,a) &= r_h(s,a) + \langle \phi^k_{\hat{V}^k_{h+1}}(s,a), \hat{w}^k_h \rangle. 
\end{align*}
Moreover, by \citep[Theorem~2]{NIPS2011_e1d5be1c}, with probability at least $1 - \delta$, we have:
\begin{align*}
    \forall h \in [H], w^*_h \in \{ w \in \mathbb{R}^d: \|w - \hat{w}^k_h \|_{\Sigma^k_h} \leq \beta(k) \}.
\end{align*}
Finally, using the inequality $ \langle x, y \rangle  \leq \| x\|_A \cdot \| y\|_{A^{-1}}$ for any invertible matrix $A$ and vectors $x,y$ completes the proof. 
\end{proof}

\begin{theorem}
% Let $\kappa_h := \sup_{s_h, a_h} \frac{d^*_h(s_h,a_h)}{d^{\mu}_h(s_h, a_h)}, \forall h \in [H]$ and $\kappa := \sum_{h = 1}^H$. \footnote{Note that \citet{rashidinejad2021bridging,xie2021policy} consider this assumption with a single concentrability coefficient $\max_{h, s_h, a_h} \frac{d^*_h(s_h,a_h)}{ d^{\mu}_h(s_h, a_h)} < \infty$.  A single concentrability coefficient can be too conservative as the density ratio at different step $h$ can be much smaller than $\max_{h, s_h, a_h} \frac{d^*_h(s_h,a_h)}{ d^{\mu}_h(s_h, a_h)}$.} 
Under Assumption \ref{assumption:single_policy_concentration}-\ref{assumption:lower bound density}, w.p.a.l. $1 -  \Omega(\frac{1}{K})$ over the randomness of $\mathcal{D}$,  for the sub-optimality bound of BCP-VTR, we have: 
% \raman{Keep the terms in the bound in the same order as in the Theorem above.}
% \begin{align*}
%      &\subopt(\hat{\pi}^{mix}) \lor \subopt(\hat{\pi}^{last}) \\
%      &\leq  \frac{4 \beta(\delta) \kappa }{ \sqrt{K}} \sqrt{2 d \log\left(1 + \frac{K}{d} \right)} + \frac{4 \beta(\delta) \kappa}{ \sqrt{K}}  \sqrt{ \log\left(\frac{H}{\delta}\right) }  \\
%      &+ \frac{2}{K} + \frac{16 H}{3 K}\log\left( \frac{\log_2(KH)}{ \delta} \right),
% \end{align*}
\begin{align*}
    \mathbb{E} \left[ \subopt(\hat{\pi}^{mix}) \right] \lor \mathbb{E} \left[ \subopt(\hat{\pi}^{last}) \right] = \tilde{\mathcal{O}} \left( \frac{\kappa H d}{\sqrt{K}} \right).
\end{align*}
where $\kappa := \sum_{h = 1}^H \kappa_h$. 
\label{theorem:sublinear_subopt_bcp-vtr}
\end{theorem}
% \begin{remark}
% If we set the $\delta$ in Corollary \ref{corollary:sublinear_subopt} as $\delta = \mathcal{O}(1/ \sqrt{K})$, for the expected sub-optimality bound, we have:
% \begin{align*}
%     \mathbb{E} \left[ \subopt(\hat{\pi}^{mix}) \right] \lor \mathbb{E} \left[ \subopt(\hat{\pi}^{last}) \right] = \tilde{\mathcal{O}} \left( \frac{\kappa H d}{\sqrt{K}} \right).
% \end{align*}
% \end{remark}

\begin{theorem}[$\frac{\log K}{K}$-type sub-optimality bound   ]
Under Assumption \ref{assumption:single_policy_concentration}-\ref{assumption:lower bound density}-\ref{assumption:margin}, w.p.a.l. $1 - (1 + 3  \log_2(H/\Delta_{\min})) \delta$, for the  sub-optimality bound of BCP-VTR, we have:
\begin{align*}
    \subopt(\hat{\pi}^{mix})  &\lesssim  2 \frac{ d^2 H^2 \kappa^{3}}{\Delta_{\min} \cdot K} \log^3(dKH/\delta) + \frac{16 \kappa}{3 K} \log \log_2(K \kappa / \delta) + \frac{2}{K}.
\end{align*}
\label{theorem:logarithimic_regret_bcp-vtr} 
\end{theorem}
\begin{remark}
If we set the $\delta$ in Theorem \ref{theorem:logarithimic_regret} as $\delta = \Omega(1/K)$, then for the expected sub-optimality bound of BCP-VTR, we have: 
\begin{align*}
    \mathbb{E} \left[ \subopt(\hat{\pi}^{mix}) \right] = \tilde{\mathcal{O}} \left( \frac{ d^2 H^2 \kappa^{3}}{\Delta_{\min} \cdot K} \right). 
\end{align*}
\end{remark}

\begin{theorem}
Fix any $H \geq 2$. For any algorithm $\textrm{Algo}(\mathcal{D})$, and any concentrability coefficients $\{\kappa_h \}_{h \geq 1}$ such that $\kappa_h \geq 2$, there exist a linear mixture MDP $\mathcal{M} = (\mathcal{S}, \mathcal{A}, H, \mathbb{P}, r, d_0)$ and dataset $\mathcal{D} = \{(s_h^t, a_h^t, r_h^t)\}_{h \in [H]}^{t \in [K]} \sim \mathcal{P}(\cdot |\mathcal{M}, \mu)$ where $\sup_{h, s_h, a_h }\frac{d^{\mathcal{M},*}_h(s_h,a_h)}{d^{\mathcal{M}, \mu}_h(s_h, a_h)} \leq \kappa_h, \forall h \in [H]$ such that: 
\begin{align*}
    \mathbb{E}_{\mathcal{D} \sim \mathcal{M}} \left[\subopt(\textrm{Algo}(\mathcal{D}); \mathcal{M}) \right] = \Omega \left( \frac{H \sqrt{\kappa_{\min}} }{ \sqrt{K} } \right), 
\end{align*}
where  $\kappa_{\min} := \min\{\kappa_h: h \in [H]\}$. 
% where $\kappa_{\max} = \max\{\kappa_h: h \in [H]\}$. 
\label{theorem:lower_bound_minimax}
\end{theorem}

\end{document}